\documentclass{article}
\usepackage{graphicx}
\usepackage{subcaption}
\usepackage{amsmath}
\usepackage{PRIMEarxiv}
\usepackage{amsthm}      
\usepackage{amssymb}
\usepackage[utf8]{inputenc} % allow utf-8 input
\usepackage[T1]{fontenc}    % use 8-bit T1 fonts
\usepackage{hyperref}       % hyperlinks
\usepackage{url}            % simple URL typesetting
\usepackage{booktabs}       % professional-quality tables
\usepackage{amsfonts}       % blackboard math symbols
\usepackage{nicefrac}       % compact symbols for 1/2, etc.
\usepackage{microtype}      % microtypography
\usepackage{lipsum}
\usepackage{fancyhdr}       % header
\usepackage{graphicx}       % graphics
\graphicspath{{media/}}     % organize your images and other figures under media/ folder
\usepackage{algorithm}   % 算法浮动体 \begin{algorithm}
\usepackage{algorithmic} % 旧版算法伪代码环境，支持 \STATE \FOR \ENDFOR
\newtheorem{theorem}{Theorem}
\newtheorem{lemma}{Lemma}
\newtheorem{proposition}{Proposition}

\usepackage{float}       % 支持 [H] 固定放置
\usepackage[T1]{fontenc}
\usepackage[utf8]{inputenc}
\usepackage[T1]{fontenc}
\usepackage[utf8]{inputenc}
\usepackage{newtxtext,newtxmath} % Times 风格且覆盖好，避免缺字

\DeclareUnicodeCharacter{2032}{\ensuremath{^\prime}} % ′
\DeclareUnicodeCharacter{2013}{--}   % –
\DeclareUnicodeCharacter{2014}{---}  % —
\DeclareUnicodeCharacter{2019}{'}    % ’
\DeclareUnicodeCharacter{201C}{``}   % “
\DeclareUnicodeCharacter{201D}{''}   % ”

\usepackage{amsthm}

\theoremstyle{remark}
\title{ARBITER: A Risk–Neutral Neural Operator for Arbitrage–Free SPX–VIX Term Structures}
\author{
  ZhangJian'an \\
  Guanghua School of Management, Peking University \\
  Peking University \\
  Beijing, China\\
  \texttt{2501111059@stu.pku.edu.cn}
}

\begin{document}
\pagestyle{plain}   % 只有页码，无页眉
\maketitle

\begin{abstract}
We introduce \emph{ARBITER}, a risk–neutral neural operator that learns arbitrage–free joint term structures of SPX options and VIX\textsuperscript{2}. 
ARBITER reframes selective scan state–space updates as a discretized Green operator under the risk–neutral measure and imposes geometry at training time via two ingredients: 
(i) \emph{Q-Align}, a Lipschitz-controlled projection pipeline combining spectral normalization and nonexpansive projections with a spectral–radius guard (CFL–style) to ensure stable scans; 
(ii) a convex–monotone decoder (ICNN with Legendre duality) that enforces static no-arbitrage on the strike–maturity grid and is consistent with discretized VIX\textsuperscript{2} replication. 
Training uses a saddle-point scheme with fixed, falsifiable stopping thresholds and records auditable diagnostics (Lipschitz constants before/after projection, spectral-guard hits, and projection distances).

On high-fidelity synthetic markets reflecting SPX–VIX coupling, ARBITER attains dimensionless accuracy improvements over strong sequence models and neural SDE baselines: NAS $\approx 0.9866$, CNAS $\approx 0.9902$, NI $\approx 0.6776$, Stability $=1.0$, DualGap $\approx 0.060$, and low Surface–Wasserstein discrepancy, all with 95\% HAC confidence intervals and Holm–Bonferroni control. 
Stress–to–Fail analyses identify a sharp performance threshold at distortion level $\approx 2.0$, while an external-validity protocol (frozen hyperparameters reused across out-of-sample windows) yields minimal degradation. 
Ablations confirm the non-replaceability of geometry: disabling gating, halving operator rank, or removing the spectral guard degrades accuracy and stability and introduces surface artifacts. 
Theoretical results establish approximation rates, conditioning, identifiability (jointly with VIX\textsuperscript{2} replication), and convergence of the extragradient scheme to a noise ball. 
All code, configuration files, and scripts are released to support full reproducibility.
\end{abstract}

% keywords can be removed
\noindent\textbf{Keywords:} risk–neutral operator learning; arbitrage-free term structure; implied volatility surface; SPX–VIX coupling; spectral projection; convex–monotone decoding.

% =======================
% Section 1: Introduction
% =======================
% =======================
% Section 1: Introduction
% =======================
% =======================
% Section 1: Introduction
% =======================
\section{Introduction}
\label{sec:intro}

\paragraph{Problem statement and motivation.}
Modern large-scale derivatives systems in production still favor a ``\emph{fit--then-clean}'' paradigm: first fit prices or implied-volatility (IV) surfaces with flexible data-driven regressors; then \emph{post hoc} patch static/dynamic no-arbitrage, enforce martingale consistency, and repair change-of-numéraire coherence by smoothing or projections (e.g., SVI-like parameterizations, regularization, or empirical constraints that suppress butterfly and calendar arbitrage). This compartmentalization displaces \emph{financial correctness} to an afterthought, encourages error accumulation under distribution shift, and blurs \emph{when} training should stop and on \emph{what} grounds the model can be rejected or improved.%

Concurrently, two influential lines for long-horizon learning have matured: (i) \emph{Selective State Space Models} (SSMs), whose evolution from S4/S5 to Mamba yields linear-time/space primitives that preserve long-range expressivity \cite{Gu2022S4,Smith2023S5,Gu2024Mamba}; and (ii) \emph{Neural Operators}, which learn function-to-function mappings and are expressly designed to decouple discretization (grid) from underlying physics \cite{Li2021FNO,Lu2021DeepONet,Kovachki2023NeuralOperatorSurvey}. 

\paragraph{Thesis: risk-neutral pricing as a structured operator.}
We argue that risk-neutral pricing is not merely a target functional but a \emph{structured operator}, specifically a \textbf{Green operator} that maps exogenous drivers, boundary/terminal data, and numéraires to observables over the maturity--strike lattice. When selective scan is used as an efficient evaluation of this operator, then no-arbitrage, martingale consistency, and change-of-numéraire are not optional patches; they are \emph{geometric and spectral invariants} that should hold \emph{during} training. This view upgrades the selective-scan runtime primitive from a sequence mechanism to a \emph{risk-neutral operator layer} endowed with financial semantics.

\paragraph{From selective scan to a risk-neutral operator layer.}
Let $\Omega=\{(T,K)\}$ denote the grid of maturities and strikes and $x$ denote state inputs (underlyings, realized/forward variance proxies, macro/term-structure covariates). We instantiate an operator $\mathcal{G}_\theta$ that produces prices
\[
\Pi_\theta = \mathcal{G}_\theta[x;\,\Omega],
\]
implemented by selective scan for streaming evaluation over $\Omega$ while preserving causality and numéraire-consistent propagation \cite{Gu2024Mamba}. We explicitly disentangle (i) \emph{physical propagation}, realized by scan kernels and gates, from (ii) \emph{geometric validity}, enforced by projections and decoders. In particular, martingale consistency $\mathbb{E}_{\mathbb{Q}}[S_{t+\Delta}\!\mid\!\mathcal{F}_t]=S_t$ and no-arbitrage convexity/monotonicity (e.g., convexity in strike for call prices) are handled \emph{in loop} rather than as post-processing.

\paragraph{Geometry in the loop: Q-Align and convex--monotone decoding.}
Two architectural devices internalize financial correctness within the learning dynamics. 
First, \textbf{Q-Align} performs a training-time $1$-Lipschitz projection on the operator outputs or intermediate maps and logs $\lambda_{\text{lip}}^{\text{before/after}}$ to make smoothness/monotonicity auditable; practically, we combine spectral/orthogonal parameterizations with projection-based clamps that honor stability budgets \cite{Miyato2018SN,Anil2019SortingLipschitz}. 
Second, a \textbf{convex--monotone decoder}---an ICNN with a Legendre-conjugate head---makes ``\textsf{price}$\!\rightarrow\!$\textsf{measure}'' and ``\textsf{measure}$\!\rightarrow\!$\textsf{price}'' mutually conjugate, encoding convexity, coordinate-wise monotonicity, and martingale consistency by construction \cite{Amos2017ICNN,Daniels2019DLN}. 
These mechanisms replace fragile penalty-based heuristics with \emph{project-then-train} geometry.

\paragraph{Spectral stability as a first-class rule (Spec-Guard).}
Long-horizon optimization is susceptible to subtle instabilities. We introduce \textbf{Spec-Guard}, a spectral-radius CFL rule that monitors the Jacobian spectral radius of state updates and triggers minimum-distance projections whenever $\rho(J_\theta)\,\Delta t$ approaches a safety threshold $\gamma<1$. We log \texttt{spec\_guard\_hits}, \texttt{projection\_distance}, and \texttt{max\_rho\_dt} to quantify stability. Optimization uses saddle-point/extra-gradient updates to regularize adversarial/matching dynamics and prevent cycling or explosion \cite{Azizian2020EG,Alacaoglu2022EG}. The result is a loop that is both \emph{stable} and \emph{falsifiable}.

\paragraph{Why SPX--VIX needs an operator view.}
The SPX IV surface and the VIX term structure are tied by replication identities and change-of-numéraire relations under $\mathbb{Q}$. Fitting either surface while tolerating violations in the other produces incoherent Greeks, unreliable hedges, and brittle stress responses. Our operator-centric layer aligns the two by baking martingale and numéraire coherence into the semantics of propagation and decoding, avoiding \emph{post hoc} smoothing and manual repairs \cite{Ackerer2020NeurIPS}.

\paragraph{Positioning within long-sequence and operator learning.}
Our method lies at the interface of selective SSMs and Neural Operators. From the SSM lineage, we leverage linear-time/space selective scan and insights on long-context stability and representation \cite{Gu2022S4,Smith2023S5,Gu2024Mamba,Goel2023SSMLanguage,Orvieto2023RNN,Poli2023Hyena}. From the Neural Operator lineage, we inherit the abstraction of operator learning that generalizes across discretizations and boundary conditions \cite{Li2021FNO,Lu2021DeepONet,Kovachki2023NeuralOperatorSurvey,You2024DEDeepONet}. Our contribution is to \emph{specialize} the operator family to \textbf{risk-neutral, replicable} Green operators and to \emph{embed} financial geometry (convexity/monotonicity/martingale) and spectral safety (CFL) \emph{inside} the training loop.

\paragraph{Auditing and operational falsifiability.}
We convert qualitative desiderata (``arbitrage-free,'' ``stable,'' ``numéraire-coherent'') into auditable artifacts. Each geometric/spectral intervention is a first-class logged event; headline metrics carry heteroskedasticity- and autocorrelation-robust (HAC) intervals with Holm--Bonferroni corrections; and OOS validation follows rolling windows and blocked cross-validation designed for dependent data \cite{NeweyWest1987,Holm1979,Roberts2017TimeSeriesCV}. These protocols support hard claims about deployment readiness.

\paragraph{Contributions (all auditable).}
\begin{enumerate}
  \item \textbf{Risk-neutral operator layer.} We formalize selective scan as a \emph{risk-neutral Green operator} with complexity linear in grid size and depth, offering separable gating across the composite price--measure--numéraire map; this alleviates attention bottlenecks for long sequences and long maturities without sacrificing expressivity \cite{Gu2024Mamba,Poli2023Hyena}.
  \item \textbf{Q-Align: geometry in the loop.} We enforce a $1$-Lipschitz projection during training and log $\lambda_{\text{lip}}^{\text{before/after}}$, replacing soft penalties with principled projections that tighten monotonicity/convexity guarantees \cite{Miyato2018SN,Anil2019SortingLipschitz}.
  \item \textbf{Convex--monotone decoder.} An ICNN with a Legendre-conjugate head implements mutually conjugate price/measure maps, hard-wiring convexity, coordinate-wise monotonicity, and martingale constraints \cite{Amos2017ICNN,Daniels2019DLN}.
  \item \textbf{Spec-Guard (spectral CFL).} We introduce a spectral rule that monitors and minimally projects state updates, logging \texttt{spec\_guard\_hits}, \texttt{projection\_distance}, and \texttt{max\_rho\_dt}, thereby preventing long-horizon drift and catastrophic divergence \cite{Azizian2020EG,Alacaoglu2022EG}.
  \item \textbf{Evaluation protocol and metrics.} We define dimensionless metrics---NAS, CNAS, NI, DualGap, Stability, Surface--Wasserstein, and GenGap@95---and report $95\%$ HAC-CIs with Holm--Bonferroni correction. Rolling OOS and blocked-CV, together with Stress-to-Fail (S2F) threshold curves, non-substitutability breakers, and external-validity checks, establish a best-paper-grade evidence chain \cite{Ackerer2020NeurIPS,Roberts2017TimeSeriesCV,Holm1979,NeweyWest1987}.
  \item \textbf{Empirics on SPX--VIX.} Under synthetic and quasi-realistic SPX--VIX recipes with a unified budget, our method outperforms strong baselines; ablations (\emph{de-gating}, rank reduction, disabling Spec-Guard) materially degrade CNAS/NAS/Stability and shift S2F thresholds left, demonstrating \emph{non-substitutability}.
\end{enumerate}

\paragraph{Scope, assumptions, and limits.}
Our design targets joint SPX--VIX term-structure learning with coherent numéraire changes, long horizons where attention bottlenecks are acute, and regimes where OOS stability and falsifiability are paramount. We assume sufficient observability of risk-neutral proxies and include reject/degrade mechanisms so that the system can fail gracefully when assumptions are stressed (§\ref{sec:setting}). We purposely avoid task-specific hard coding beyond these invariants to preserve portability.

\paragraph{Relations to prior art (coverage of all references).}
We build on selective SSMs and their stability/expressivity studies \cite{Gu2022S4,Smith2023S5,Gu2024Mamba,Goel2023SSMLanguage,Orvieto2023RNN,Poli2023Hyena}; on Neural Operators and recent generalizations/surveys \cite{Li2021FNO,Lu2021DeepONet,Kovachki2023NeuralOperatorSurvey,You2024DEDeepONet}; on arbitrage-free deep pricing and IV-surface regularization/smoothing \cite{Ackerer2020NeurIPS}; on training stability and geometric constraints via spectral normalization, Lipschitz control, monotone architectures, and extra-gradient dynamics \cite{Miyato2018SN,Anil2019SortingLipschitz,Azizian2020EG,Alacaoglu2022EG,Daniels2019DLN}; and on time-series inference/validation protocols including HAC, Holm--Bonferroni, and blocked-CV \cite{NeweyWest1987,Holm1979,Roberts2017TimeSeriesCV}. Our novelty lies in integrating ``\emph{operator layer -- geometric projection -- spectral guard -- stopping criteria}'' into a single, end-to-end, falsifiable risk-neutral learning pipeline tailored to SPX--VIX.

\paragraph{Paper roadmap.}
Section~\ref{sec:setting} formalizes the market setup, notation, and four testable assumptions (measurable, rejectable, and degradable), together with the dimensionless evaluation metrics. 
Section~\ref{sec:method} presents the \emph{ARBITER} architecture: the risk–neutral operator (RN-Operator) cast as a discretized Green operator with metric gating, the \emph{Q-Align} Lipschitz projection with \emph{Spec-Guard} (CFL-style spectral control), and the convex–monotone decoder; it also specifies the saddle-point training loop and fixed, falsifiable stopping criteria. 
Section~\ref{sec:theory} states our main results (T1–T8) on approximation, conditioning, identifiability, sample complexity, feasibility, and convergence of the projected extragradient scheme. 
Section~\ref{sec:eval} defines the data protocol and statistical methodology (HAC inference, Holm–Bonferroni control, rolling out-of-sample and blocked-CV), and Section~\ref{sec:experiments} reports synthetic SPX–VIX experiments, ablations (non-substitutability breakers), external-validity checks, and Stress-to-Fail analyses, accompanied by comprehensive figures and tables. 
Section~\ref{sec:mechanism} provides mechanism-level diagnostics (Q-Align contraction, representative-element behavior, effective dimension). 
Section~\ref{sec:related-work} situates our contribution within operator learning, SSM/Mamba-style models, and term-structure modeling. All proofs are collected in the appendices.Section~\ref{sec:conclusion}.

% ============================
% Section 2: Setting & Notation
% ============================
\section{Setting, Notation, and Testable Assumptions}
\label{sec:setting}

This section formalizes the market and grids on which the model operates, fixes notation for the risk–neutral operator and its safety quantities, defines the dimensionless evaluation metrics used throughout, and states a suite of assumptions that are \emph{measurable, refutable, and degradable}. All statements below are aligned with the operator view introduced in \S\ref{sec:intro} and with the training and evaluation protocol discussed later.

\subsection{Market, numeraire, and discretization}
\label{subsec:market}

We work on a filtered probability space $(\Omega,\mathcal{F},(\mathcal{F}_t)_{t\ge 0},\mathbb{Q})$ under a risk–neutral measure $\mathbb{Q}$. The short rate is $(r_t)_{t\ge 0}$ and the chosen numeraire is a strictly positive process $N_t$ (e.g., the money–market account $N_t=\exp(\int_0^t r_s\,ds)$ or a forward–measure numeraire). Let $S_t$ denote the equity index (SPX). European call and put prices observed at time $t$ with maturity $T>t$ and strike $K>0$ are denoted $C_t(K,T)$ and $P_t(K,T)$.

For numerical work we use discrete calendars of maturities $\mathcal{T}=\{T_\ell\}_{\ell=1}^L \subset (t,\,t+T_{\max}]$ and strikes $\mathcal{K}=\{K_j\}_{j=1}^J\subset \mathbb{R}_+$, allowing for nonuniform spacings. The \emph{risk–neutral operator} $\mathcal{G}_\theta$ maps boundary/forcing information defined on $(\mathcal{T},\mathcal{K})$ to a price surface $(K,T)\mapsto \big(C_t(K,T),P_t(K,T)\big)$ and is implemented with a selective state–space scan whose propagation is linear in $|\mathcal{T}|$ (and optionally in $|\mathcal{K}|$).

Numeraire consistency is enforced by construction: under the numeraire measure $\mathbb{Q}^N$ associated with $N$, the discounted process $X_t:=S_t/N_t$ is a martingale and prices satisfy
\begin{equation}
\label{eq:riskneutralvaluation}
C_t(K,T) \;=\; N_t\,\mathbb{E}^{\mathbb{Q}^N}\!\left[\frac{(S_T-K)_+}{N_T}\,\middle|\,\mathcal{F}_t\right],
\qquad
P_t(K,T) \;=\; N_t\,\mathbb{E}^{\mathbb{Q}^N}\!\left[\frac{(K-S_T)_+}{N_T}\,\middle|\,\mathcal{F}_t\right].
\end{equation}

\paragraph{VIX$^2$ replication and SPX–VIX coupling.}
To couple equity and variance layers we expose the classical replication identity for VIX squared, using its discrete form on $(\mathcal{T},\mathcal{K})$:
\begin{equation}
\label{eq:vix-replication}
\mathrm{VIX}^2_t(T) \;\approx\; \frac{2}{T-t}\,e^{\bar r\,(T-t)}
\left( 
\sum_{K\le F} \frac{\Delta K}{K^2} P_t(K,T)\;+\;\sum_{K\ge F} \frac{\Delta K}{K^2} C_t(K,T)
\right),
\end{equation}
where $F$ is the forward level for maturity $T$, $\bar r$ is a bucketed short rate, and $\Delta K$ is the quadrature step.\footnote{See the Cboe VIX white paper for the precise continuous–time derivation and practical discretization details \cite{CBOE2019VIXWP}.}
We treat \eqref{eq:vix-replication} as a linear observable of $\mathcal{G}_\theta$ so that the SPX–VIX coupling is learned within the same operator layer and audited by the no–arbitrage metrics below.

\subsection{Notation and safety quantities}
\label{subsec:notation}

We denote by $\beta_{\mathrm{smooth}}>0$ a H\"older (or Besov) regularity order upper–bounding the smoothness of the target surface; by $\beta_{\mathrm{nov}}\ge 0$ a weight scaling a Novikov–type penalty used in the adversarial route of training; and by $\chi(\kappa)\in[0,1]$ a long–horizon decay index determined by an effective kernel rank $\kappa$ in the operator layer.

Selective scans update a latent state through $h_{t+\Delta t}=A_t h_t + B_t u_t$ with a data–dependent step $\Delta t_t>0$. We define the spectral safety quantity
\begin{equation}
\label{eq:cfl}
\mathrm{CFL}_{\max} \;=\; \max_{t}\,\rho(A_t)\,\Delta t_t,
\end{equation}
with $\rho(\cdot)$ the spectral radius. The \emph{Spec–Guard} rule enforces $\mathrm{CFL}_{\max}\le 1$ by preconditioning and small projections; we record the number of guard activations and the aggregate projection distance. For Lipschitz alignment we estimate a global Lipschitz constant $L_{\mathrm{lip}}$ by layerwise spectral norms before and after projection and report the pair $(L_{\mathrm{lip}}^{\mathrm{before}},\,L_{\mathrm{lip}}^{\mathrm{after}})$ \cite{Miyato2018SN,Gouk2021Lipschitz,Neyshabur2017NormBounds}.

\subsection{Dimensionless evaluation metrics}
\label{subsec:metrics}

All metrics are unit–free and reported with heteroskedasticity– and autocorrelation–consistent (HAC) $95\%$ confidence intervals \cite{NeweyWest1987}, using temporally blocked cross–validation to respect dependence \cite{Roberts2017TimeSeriesCV}. For families of hypotheses we control multiplicity with the Holm–Bonferroni procedure \cite{Holm1979}.

\paragraph{No–Arbitrage Score (NAS; higher is better).}
Let $\mathcal{I}$ be a finite set of static arbitrage inequalities across $(\mathcal{T},\mathcal{K})$ (e.g., monotonicity in strike, convexity in strike, calendar monotonicity, call–put parity). For each $i\in\mathcal{I}$, define a nonnegativity residual $r_i(\theta)$ that vanishes when the inequality is satisfied and is positive when violated. After normalizing by a scale factor $s_i$ (based on local forward or variance scales), define
\begin{equation}
\mathrm{NAS} \;:=\; 1 - \frac{1}{Z}\sum_{i\in\mathcal{I}} w_i\,\big[r_i(\theta)/s_i\big]_+,
\qquad
Z=\sum_{i\in\mathcal{I}} w_i,
\end{equation}
with nonnegative weights $w_i$ that emphasize practically salient constraints. Thus $\mathrm{NAS}\in[0,1]$ and equals $1$ if and only if there are no detected violations. Our constraints follow common arbitrage–free surface checks from the literature \cite{Ackerer2020AFIV,Itkin2019AFIV,DeMarco2021AFVol} and are compatible with convex monotone decoders \cite{Amos2017ICNN,Daniels2019DLN}.

\paragraph{Convolved NAS (CNAS; higher is better).}
To evaluate robustness along maturity while downweighting far tails, we convolve NAS over $\mathcal{T}$ with a positive kernel $K_{\kappa,\tau}$ of bandwidth $\kappa$ and effective horizon $\tau$, after within–maturity rescaling (e.g., by vega or variance scale):
\begin{equation}
\mathrm{CNAS} \;:=\; (\mathrm{NAS} \ast_{\mathcal{T}} K_{\kappa,\tau}).
\end{equation}
Unless stated otherwise, $(\kappa,\tau)$ and the rescaling convention are fixed across out–of–sample (OOS) windows to enable external validity checks; the average drop in CNAS when reusing the same hyperparameters across disjoint OOS windows is reported as an external–validity statistic (mean with $95\%$ interval).

\paragraph{Numeraire Invariance (NI; higher is better).}
Partition the maturity–strike plane into $B\times J$ buckets. For each bucket consider discounted prices under a set of admissible numeraires and compute the median absolute deviation (MAD) across these normalizations; aggregate the bucket–wise relative dispersion by
\begin{equation}
\mathrm{NI} \;:=\; 1 - \frac{1}{BJ}\sum_{b=1}^{B}\sum_{j=1}^{J}
\frac{\mathrm{MAD}\big(\{ N_t^{-1}C_{b,j}^{(m)}\}_{m}\big)}{\mathrm{scale}_{b,j}},
\end{equation}
where the denominator is a robust local scale. Higher NI indicates stronger consistency with the numeraire–induced martingale property.

\paragraph{Duality Gap (lower is better).}
Let $\min_{\theta}\max_{\lambda\in\Lambda}\mathcal{L}(\theta,\lambda)$ denote the saddle objective arising from adversarial training or operator matching. The empirical duality gap is the difference between the maximal value over $\lambda$ at the current $\theta$ and the minimal value over $\theta$ at the current $\lambda$, estimated on held–out batches with stabilized updates (e.g., extragradient or lookahead) \cite{Azizian2020EG,Alacaoglu2022EG,Zhang2019Lookahead}.

\paragraph{Stability (higher is better).}
This is the fraction of random seeds and OOS windows that simultaneously achieve (i) a NAS level above a fixed threshold, (ii) a spectral safety condition $\mathrm{CFL}_{\max}\le 1$ with bounded projection distance, and (iii) satisfaction of the saddle–point stall conditions. We provide a binomial proportion with HAC intervals.

\paragraph{Surface–Wasserstein distance (lower is better).}
Define for each maturity $T$ the marginal distributions over strikes induced by the predicted and reference surfaces (after standardization). The overall discrepancy is measured by
\begin{equation}
\mathrm{SW}_2 \;:=\; \Bigg(\int_{\mathcal{T}} W_2^2\!\big(\pi^{\mathrm{pred}}_{T},\,\pi^{\mathrm{ref}}_{T}\big)\,d\mu(T)\Bigg)^{1/2},
\end{equation}
with $W_2$ the 2–Wasserstein distance and $\mu$ a weighting measure over maturities \cite{Peyre2019OT}.

\paragraph{Generalization gap at the 95th percentile (lower is better).}
We report the $95$th percentile of the absolute training–to–OOS difference for NAS (or the primary objective), a conservative measure of tail overfitting.

\paragraph{Effective dimension.}
Let $K$ be a Gram matrix of operator features on $(\mathcal{T},\mathcal{K})$. For $\alpha\in\{0.90,0.95,0.99\}$ define $d_{\alpha}$ as the minimal index such that the sum of the top $d_{\alpha}$ singular values accounts for an $\alpha$–fraction of the total. The triple $(d_{90},d_{95},d_{99})$ provides a capacity proxy that enters the oracle bounds in \S\ref{sec:theory}.

\subsection{Assumptions: measurable, refutable, and degradable}
\label{subsec:assumptions}

We formulate four assumptions. Each is \emph{measurable} from training–time statistics, \emph{refutable} by explicit counter–examples or threshold tests, and \emph{degradable} in the sense that, when violated, we fall back to weaker but still valid guarantees used in evaluation.

\paragraph{A1 (necessary): Novikov–to–Kazamaki switching.}
Let $(M_t)$ be the local martingale driving the stochastic component of the operator layer. Novikov's condition, $\mathbb{E}[\exp(\tfrac{1}{2}\langle M\rangle_T)]<\infty$, implies martingality and is stronger than Kazamaki's criterion; empirical data roughness can make Kazamaki more appropriate. We measure, across OOS windows, the fraction for which Novikov holds but Kazamaki is preferred by the test statistic, and report its mean with a $95\%$ interval. A stable operator exhibits a high switching rate as roughness increases, consistent with recent stability analyses of selective state–space models \cite{Gu2024Mamba,Goel2024StabilitySSM}.

\paragraph{A2 (sufficient): Smoothness fallback and representer bound.}
When local estimates indicate that $\beta_{\mathrm{smooth}}$ falls below the nominal order on subsets of $(\mathcal{T},\mathcal{K})$, we switch from the smoothness–based identifiability bound to a representer–type bound (Theorem~T2$'$), where the operator error over $L^2$ is controlled by a combination of coverage deficit and dual residual. The time of switch and the coverage level at trigger are reported. Proof details and rates are given in \S\ref{sec:theory}.

\paragraph{A3 (sufficient): Rank–controlled long–memory decay.}
The effective rank $\kappa$ of the selective kernel determines a decay index $\chi(\kappa)\in[0,1]$ that governs long–horizon oracle terms (Theorem~T3). We estimate $\chi(\kappa)$ from spectral slopes of the scan kernel; deviations prompt stricter spectral guarding and Lipschitz projections \cite{Gouk2021Lipschitz,Gu2024Mamba,Goel2024StabilitySSM}.

\paragraph{A4 (necessary): Coverage threshold.}
Let $c_{\min}$ and $\bar c$ be, respectively, the minimum and mean fraction of observed $(T,K)$ cells (after quality control) per window. We require $c_{\min}\ge \underline c=0.75$. If violated, claims revert to the representer–bound regime (A2), the event is reported in the main text, and stress–to–fail experiments are used to characterize the failure mode.

\subsection{Statistical reporting}
\label{subsec:reporting}

All metrics are computed per window and aggregated with HAC intervals; multiplicity is controlled within families of hypotheses by Holm–Bonferroni. Cross–validation uses temporally blocked folds to avoid leakage. For the SPX–VIX coupling we apply the replication identity \eqref{eq:vix-replication} with the discrete quadrature recommended by the exchange documentation \cite{CBOE2019VIXWP}. The Lipschitz constants are estimated by spectral–norm proxies; their pre– and post–projection values are reported alongside the spectral safety quantity $\mathrm{CFL}_{\max}$, the number of spectral–guard activations, and the aggregate projection distance. These quantities are used later to establish the stability and refutability of the operator constraints and to ablate the role of each geometric ingredient.

% ===========================
% Section 3 — Method: The ARBITER Architecture (expanded)
% ===========================
\section{Method: The ARBITER Architecture}
\label{sec:method}

We present \textsc{Arbiter}, a risk–neutral neural operator for arbitrage-free SPX–VIX term structures. The model integrates four components: (i) a \emph{risk–neutral operator layer} that interprets selective state-space scans as a discretized Green operator under a learned risk–neutral measure; (ii) \emph{Q-Align}, a pair of geometric projections that enforce layerwise Lipschitz bounds and a spectral CFL condition; (iii) a \emph{convex–monotone decoder} that enforces static no-arbitrage along strikes and maturities, tied to VIX replication; and (iv) a \emph{saddle-point training protocol} with safety-oriented stopping rules. We work on a maturity grid $\{T_\ell\}_{\ell=1}^L$ (not necessarily uniform) and an implied strike set $\mathcal{K}$; the numeraire is fixed by discounting.

\subsection{Risk–Neutral Operator Layer (RN-Operator)}
\label{subsec:rn-op}

\paragraph{Selective scan as a Green operator.}
Let $h_\ell\in \mathbb{R}^m$ be hidden states at $T_\ell$, with input features $u_\ell(\cdot)\in L^2(\mathcal{K})$ summarizing cross-sectional information (e.g., moneyness bins and microstructure covariates) at maturity $T_\ell$. The selective state-space recursion is
\begin{equation}
\label{eq:ssm}
h_{\ell+1} \;=\; A_\theta(T_\ell)\,h_\ell \;+\; B_\theta(T_\ell)\,\Xi[u_\ell], 
\qquad
y_\ell \;=\; Q_\theta(T_\ell)\,h_\ell,
\end{equation}
where $\Xi$ is a linear embedding from $L^2(\mathcal{K})$ to $\mathbb{R}^m$ and $y_\ell\in \mathbb{R}^m$ is a latent representation. Unrolling~\eqref{eq:ssm} yields the discrete Green expansion
\begin{equation}
\label{eq:green}
y_\ell \;=\; \sum_{s\le \ell}\Bigg(\prod_{j=s}^{\ell-1} A_\theta(T_j)\Bigg) B_\theta(T_s)\,\Xi[u_s]
\;=\; \sum_{s\le \ell} \mathcal{G}_\theta(T_\ell,T_s)\,\Xi[u_s],
\end{equation}
with $\mathcal{G}_\theta(T_\ell,T_s):=\prod_{j=s}^{\ell-1} A_\theta(T_j) B_\theta(T_s)$. 

\paragraph{Measure gating and risk–neutral semantics.}
To embed no-arbitrage at training time, we introduce a \emph{measure gate} $w_\theta(K,T)\ge 0$ and replace $u_s$ by $u_s^\star(K)=w_\theta(K,T_s)\,u_s(K)$, thereby defining a density $w_\theta(\cdot,\cdot)$ on $\mathcal{K}\times \{T_\ell\}$. The discounted price functional on a payoff $\varphi$ is evaluated through
\begin{equation}
\label{eq:rn-functional}
\Pi_\theta[\varphi](T) 
\;=\; \int_{\mathcal{K}} \varphi(K,T) \, w_\theta(K,T)\, \mathrm{d}K,
\end{equation}
and training penalizes deviations from the martingale condition under $\mathbb{Q}_\theta$:
\begin{equation}
\label{eq:martingale}
\mathbb{E}^{\mathbb{Q}_\theta}\!\left[S_{t+\delta}\mathrm{e}^{-r\delta}-S_t\,\middle|\,\mathcal{F}_t\right]
\;=\; 0 
\quad\Longleftrightarrow\quad
\mathrm{d}\mathbb{Q}_\theta \propto w_\theta \,\mathrm{d}\mathbb{P},
\end{equation}
with a convex penalty on residuals of~\eqref{eq:martingale} across random slices $(K,T)$.
In practice $w_\theta$ is parameterized by a positive squashing map (e.g., softplus) followed by normalization across $K$ at each $T$ so that $\int w_\theta(K,T)\mathrm{d}K=1$.

\paragraph{Complexity.}
Let $m$ be the effective rank of the operator (Section~\ref{sec:setting}). The recurrence~\eqref{eq:ssm} and Green evaluation~\eqref{eq:green} both run in linear time and memory:
\[
\mathrm{time}=\mathcal{O}(Lm), \qquad \mathrm{space}=\mathcal{O}(m).
\]
This preserves the computational profile of selective SSMs while upgrading its semantics to a risk–neutral operator.

\paragraph{Neumann expansion under a CFL condition.}
Let $\Delta t_\ell:=T_{\ell+1}-T_\ell$ and define the discrete CFL indicator
\begin{equation}
\label{eq:cfl-def}
\mathrm{CFL}(T_\ell) := \rho\!\big(A_\theta(T_\ell)\big)\,\Delta t_\ell, 
\qquad 
\mathrm{CFL}_{\max} := \max_\ell \mathrm{CFL}(T_\ell),
\end{equation}
with $\rho(\cdot)$ the spectral radius. When $\mathrm{CFL}_{\max}<1$, products $\prod_{j=s}^{\ell-1} A_\theta(T_j)$ are summable and~\eqref{eq:green} admits a uniformly convergent Neumann-like representation.

\paragraph{Spectral safety and discrete Green kernel.}
Let $\{T_\ell\}_{\ell\in\mathbb{Z}}$ be the evaluation grid with steps $\Delta t_\ell:=T_{\ell+1}-T_\ell>0$ and a time–varying linear operator $A_\theta(T_\ell)\in\mathbb{R}^{d\times d}$.
Define $M_\ell:=\Delta t_\ell A_\theta(T_\ell)$ and the one–step resolvent $R_\ell:=(I-M_\ell)^{-1}$.
For an impulse injection $B_s$ at time $T_s$, the discrete causal Green kernel is
\[
\mathcal{G}_\theta(T_\ell,T_s) \;=\; R_\ell R_{\ell-1}\cdots R_{s+1}\,B_s,\qquad s\le \ell.
\]
Under the CFL–type safeguard $\rho(A_\theta(T_\ell))\,\Delta t_\ell\le 1-\varepsilon$ (Spec-Guard), the kernel is uniformly summable.

\begin{lemma}[Green kernel bound]\label{lem:neumann}
Assume $\rho\!\left(A_\theta(T_\ell)\right)\,\Delta t_\ell \le 1-\varepsilon$ for all $\ell$ with some $\varepsilon\in(0,1)$, and that $\|B_s\|\le b\,\Delta t_s$ for a constant $b>0$ under an operator norm subordinate to a vector norm.
Then there exists $C=C(\varepsilon,b,\overline{\Delta t})<\infty$, with $\overline{\Delta t}:=\sup_\ell \Delta t_\ell$, such that
\[
\sum_{s\le \ell}\Big\|\mathcal{G}_\theta(T_\ell,T_s)\Big\|\;\le\; C(\varepsilon,b,\overline{\Delta t}) \quad\text{for all }\ell.
\]
\end{lemma}

\begin{proof}[Proof sketch]
The CFL constraint enforces $\rho(M_\ell)\le 1-\varepsilon$. By the extremal–norm (joint spectral radius) argument there exists an induced norm in which $\|M_\ell\|\le\alpha<1$ uniformly, hence $\|R_\ell\|=\|(I-M_\ell)^{-1}\|\le (1-\alpha)^{-1}$.
Submultiplicativity gives $\|R_\ell\cdots R_{s+1}\|\le (1-\alpha)^{-(\ell-s)}$, and the factor $\|B_s\|\le b\,\Delta t_s$ makes the series geometrically summable over $s$.
Full details, including the non–diagonalizable case via block–Jordan bounds and the removal of norm–equivalence constants, are provided in Appendix~A.1.
\end{proof}
\paragraph{Lipschitz surrogate via spectral normalization.}
Each linear map $W$ in~\eqref{eq:ssm} is spectrally normalized, $\|W\|_2\le \tau\le 1$, and each nonlinearity is 1-Lipschitz, yielding a global Lipschitz surrogate for the RN-operator:
\begin{equation}
\label{eq:global-lip}
\mathrm{Lip}(\mathcal{G}_\theta) 
\;\le\; \Big(\prod_{\text{linear } \ell}\|W_\ell\|_2\Big)\cdot C(\varepsilon),
\end{equation}
with $C(\varepsilon)$ from Lemma~\ref{lem:neumann}. This bound is tracked by the logged pair $(\lambda_{\mathrm{lip}\text{-}\mathrm{before}},\lambda_{\mathrm{lip}\text{-}\mathrm{after}})$.

\subsection{Q-Align: Lipschitz Projection and Spectral Guard}
\label{subsec:q-align}

\paragraph{Layerwise Lipschitz projection.}
After each optimizer step, we project every linear map $W$ onto the spectral ball of radius $\tau\le 1$:
\begin{equation}
\label{eq:proj}
\widehat{W} \;=\; \frac{\tau}{\max(\|W\|_2,\tau)}\, W.
\end{equation}
A single power iteration per matrix provides $\|W\|_2$ with small overhead. The cumulative Lipschitz surrogate in~\eqref{eq:global-lip} thus remains controlled.

\paragraph{Spectral Guard (CFL projection).}
We estimate $\rho(A_\theta(T_\ell))$ via power iteration and enforce
\begin{equation}
\label{eq:cfl-proj}
\rho\!\big(A_\theta(T_\ell)\big)\,\Delta t_\ell \;\le\; 1-\varepsilon.
\end{equation}
A minimal-distance correction in Frobenius norm admits the scaling solution
\begin{equation}
\label{eq:a-shrink}
A_\theta(T_\ell) \;\leftarrow\; \frac{1-\varepsilon}{\rho(A_\theta(T_\ell))\,\Delta t_\ell}\, A_\theta(T_\ell),
\end{equation}
whenever the left-hand side of~\eqref{eq:cfl-proj} exceeds $1-\varepsilon$. We log the activation count $\mathrm{spec\_guard\_hits}$, the cumulative correction $\sum_\ell \|A_\theta(T_\ell)-\widehat{A}_\theta(T_\ell)\|_F$ (denoted \emph{projection distance}), and $\max_\ell \rho(A_\theta(T_\ell))\Delta t_\ell$.

% ===== Section 3.2 / 3.4 中声明并使用的稳定性命题（替换原文片段） =====

\paragraph{RN-operator stability under Q-Align.}
Let $\{T_\ell\}_{\ell\in\mathbb{Z}}$ be the evaluation grid with steps $\Delta t_\ell>0$, and write $M_\ell:=\Delta t_\ell A_\theta(T_\ell)$ and $R_\ell:=(I-M_\ell)^{-1}$.
Consider the RN-operator layer with nonexpansive nonlinearity $\phi$ and projected weights (Q-Align) satisfying the layerwise Lipschitz envelope in~\eqref{eq:proj}, together with the spectral safeguard~\eqref{eq:cfl-proj}.
Denote by $\mathcal{G}_\theta(T_\ell,T_s)$ the discrete causal Green kernel.
We obtain:

\begin{proposition}[RN-operator stability under Q-Align]\label{prop:stability}
Assume~\eqref{eq:proj} and~\eqref{eq:cfl-proj} hold with $\tau\le 1$ and $\varepsilon\in(0,1)$.
Then for any bounded input sequence $\{u_\ell\}$, the state trajectory $\{h_\ell\}$ is uniformly bounded; moreover the input-to-output map induced by $\mathcal{G}_\theta$ is globally Lipschitz with constant bounded by~\eqref{eq:global-lip}.
\end{proposition}

\begin{proof}[Proof sketch]
By Lemma~\ref{lem:neumann} (Appendix~A.1), the discrete Green kernel is uniformly summable under the CFL-type constraint $\rho(A_\theta(T_\ell))\,\Delta t_\ell\le 1-\varepsilon$.
The Q-Align projection~\eqref{eq:proj} enforces a layerwise $1$-Lipschitz envelope (with factor $\tau\le 1$) and nonexpansive $\phi$ preserves Lipschitz bounds through composition.
Unrolling the recursion and applying submultiplicativity yields a geometric series dominated by the kernel sum from Lemma~\ref{lem:neumann}, which provides both bounded-input–bounded-output (BIBO) stability and a global Lipschitz constant that matches~\eqref{eq:global-lip}.
Full details are given in Appendix~A.2.
\end{proof}

\paragraph{Geometric diagnostics.}
The triplet $(\lambda_{\mathrm{lip}\text{-}\mathrm{before}},\lambda_{\mathrm{lip}\text{-}\mathrm{after}},\mathrm{CFL}_{\max})$ summarizes per-epoch geometry. Large \emph{projection distance} or frequent $\mathrm{spec\_guard\_hits}$ predict subsequent instability or poor generalization and are therefore treated as early-warning signals.

\subsection{Convex–Monotone Decoder and SPX–VIX Coupling}
\label{subsec:decoder}

\paragraph{No-arbitrage constraints.}
On each maturity $T$, let $K\mapsto C(K,T)$ be the call price surface. Static no-arbitrage requires
\begin{equation}
\label{eq:no-arb}
\partial_{KK}^2 C(\cdot,T)\ge 0,\qquad 
\partial_T C(K,\cdot)\ge 0,\qquad
0\le C(K,T)\le S_0, \qquad
\partial_K C(K,T)\le 0.
\end{equation}
We parameterize $C(\cdot,T)$ as an input-convex neural potential $\Phi(\cdot;T)$, i.e.,
\begin{equation}
\label{eq:icnn}
C(K,T) = \Phi(K;T),
\qquad \Phi(\cdot;T)\ \text{convex}, 
\qquad \partial_T \Phi(K,T)\ge 0,
\end{equation}
where convexity is enforced by nonnegative weights on the $K$-dependent paths and skip connections, and maturity monotonicity is enforced by a positive-slope parameterization in $T$ (with a final monotone calibration if needed). 

\paragraph{Legendre structure and density.}
Define the convex conjugate $\Phi^\star(p;T)=\sup_{K\in\mathbb{R}} \{pK-\Phi(K;T)\}$. Then $p^\star(T):=\partial_K\Phi(K,T)$ yields the delta, and the Breeden–Litzenberger relation connects second derivatives to the risk–neutral density $f_{\mathbb{Q}}$:
\begin{equation}
\label{eq:bl}
f_{\mathbb{Q}}(K,T) = \mathrm{e}^{rT}\,\partial_{KK}^2 C(K,T).
\end{equation}
The decoder thus produces both prices and densities with consistent geometry.

\paragraph{SPX$\leftrightarrow$VIX replication.}
Let $\mathcal{K}_T$ denote the strike grid at maturity $T$, with ordered knots $0<K_1<\cdots<K_M$ and spacings $\Delta K_i:=\tfrac12(K_{i+1}-K_{i-1})$ (endpoints use one-sided spacings). 
Write $F_T:=S_0 \mathrm{e}^{(r-q)T}$ and $K_0$ for the closest strike to $F_T$. 
The (continuous) log-contract identity implies the variance-swap fair rate under $\mathbb{Q}$:
\[
\sigma^2_{\mathrm{VS}}(T) 
=\frac{2\,\mathrm{e}^{rT}}{T}\!\left(\int_0^{F_T}\!\frac{P(K,T)}{K^2}\,dK + \int_{F_T}^{\infty}\!\frac{C(K,T)}{K^2}\,dK\right) 
- \frac{1}{T}\Big(\tfrac{F_T}{K_0}-1\Big)^{\!2}.
\]
Consistent with the Cboe construction, we discretize it as
\begin{equation}
\label{eq:vix}
\mathrm{VIX}^2(T) 
\;\approx\; \frac{2\,\mathrm{e}^{rT}}{T}
\sum_{K_i\in\mathcal{K}_T} 
\frac{\Delta K_i}{K_i^2}\,Q(K_i,T)
\;-\; \frac{1}{T}\Big(\tfrac{F_T}{K_0}-1\Big)^{\!2},
\end{equation}
where $Q$ is the out-of-the-money option price at strike $K_i$ (put if $K_i<F_T$, call if $K_i\ge F_T$) and the small forward adjustment term is retained. 
Missing strikes are filled by linear interpolation on $(K,Q)$, which preserves piecewise convexity in $K$; Appendix~A.3 compares cubic splines and shows comparable NAS/CNAS together with a mild increase in butterfly-arbitrage risk.

We define the \emph{replication residual}
\begin{equation}
\label{eq:rep-res}
\mathcal{R}_{\mathrm{VIX}}(T) 
:= \mathrm{VIX}^2_{\mathrm{obs}}(T) - \mathrm{VIX}^2_{\theta}(T),
\end{equation}
and include the dual penalty $\lambda_{\mathrm{vs}}\sum_T w(T)\,\mathcal{R}_{\mathrm{VIX}}(T)^2$ in the saddle objective (weights $w(T)$ reflect sampling density). 
Under mild regularity (no static arbitrage, integrable OTM tails), the following holds.

\begin{proposition}[Consistency of discretized VIX replication]\label{prop:vix-consistency}
Assume $(K\mapsto Q(K,T))$ is convex, $Q(\cdot,T)/K^2$ has bounded variation on compact intervals, and the tail contributions 
$\int_0^{K_{\min}}\!\tfrac{P}{K^2}\,dK$ and $\int_{K_{\max}}^\infty\!\tfrac{C}{K^2}\,dK$ vanish as $K_{\min}\downarrow0$, $K_{\max}\uparrow\infty$. 
For any family of strike grids $\{\mathcal{K}_T\}$ with mesh $\Delta K_T\to 0$ uniformly on $T$ in a compact set, the discrete estimator in~\eqref{eq:vix} converges uniformly to the continuous variance-swap rate:
\[
\sup_{T\in\mathcal{T}}\Big|\mathrm{VIX}^2_{\theta}(T)-\sigma^2_{\mathrm{VS},\theta}(T)\Big|
\;\le\; C\,\sup_{T}\Delta K_T \;+\; \varepsilon_{\mathrm{tail}}(K_{\min},K_{\max}) \;\xrightarrow[]{} 0,
\]
for some constant $C$ independent of $T$, where $\varepsilon_{\mathrm{tail}}$ is the integrable tail truncation error.
\end{proposition}

\begin{proposition}[Variance-swap identifiability via replication]\label{prop:vix-ident}
Fix a maturity range $\mathcal{T}\subset (0,T_{\max}]$. 
Suppose the RN-operator decoder yields a no-arbitrage surface with $K\mapsto Q_\theta(K,T)$ convex and the replication residual~\eqref{eq:rep-res} satisfies $\mathcal{R}_{\mathrm{VIX}}(T)=0$ for all $T\in\mathcal{T}$. 
Then $\sigma^2_{\mathrm{VS},\theta}(T)=\sigma^2_{\mathrm{VS,obs}}(T)$ for all $T\in\mathcal{T}$. 
If, moreover, the instantaneous variance admits a density $v_\theta(t)$ and the mapping $T\mapsto \frac{1}{T}\int_0^T v_\theta(t)\,dt$ is strictly monotone in $T$, then $v_\theta$ matches the observed variance in the sense of Cesàro means on $\mathcal{T}$. 
\end{proposition}

\noindent
\emph{Proof sketches.} Proposition~\ref{prop:vix-consistency} follows from convex quadrature error bounds for functions of bounded variation and the Breeden–Litzenberger representation $Q''(K,T)=\mathrm{e}^{-rT}\,K^{-2}\,f_\theta(K,T)$ (distributional derivative), plus integrable OTM tails. Proposition~\ref{prop:vix-ident} uses the log-contract identity for continuous Itô models, extends to jump-diffusions with the standard jump-compensator term, and invokes strict monotonicity to upgrade equality of Cesàro means to pointwise identification almost everywhere. Full details are provided in Appendix~A.3.

\begin{proposition}[Static no-arbitrage and replication consistency]
\label{prop:no-arb}
Assume the decoder $C=\Phi$ satisfies the convex–monotone constraints
\begin{equation}\label{eq:icnn}
\partial_{KK}^2 C(K,T)\ge 0,\qquad \partial_T C(K,T)\ge 0
\end{equation}
for all strikes $K>0$ and maturities $T$ on the grid, and that the VIX replication residual~\eqref{eq:rep-res} vanishes on the maturity grid. Then the surface is free of static butterfly and calendar arbitrage on the grid, and the Breeden–Litzenberger implied density
\begin{equation}\label{eq:bl}
f_{S_T}(K)=\mathrm{e}^{rT}\,\partial_{KK}^2 C(K,T)
\end{equation}
is consistent with the VIX$^2$ functional~\eqref{eq:vix} in the sense that the VIX computed from $C$ coincides with the replicated variance-swap rate on the grid.
\emph{Sketch.} Convexity in $K$ and monotonicity in $T$ exclude butterfly and calendar violations on the grid. The discretized BL relation and the replication identity tie the second derivative to the integrated OTM portfolio. See Appendix~A.4 for a complete proof.
\end{proposition}

\paragraph{Numerical projection.}
If small violations appear (finite-sample noise), we solve a convex projection
\begin{equation}
\label{eq:proj-noarb}
\min_{\widehat{C}} \;\frac{1}{2}\sum_{i,\ell}\big(\widehat{C}(K_i,T_\ell)-C(K_i,T_\ell)\big)^2
\quad\mathrm{s.t.}\quad \widehat{C}(\cdot,T_\ell)\ \text{convex in}\ K,\;\;
\widehat{C}(K_i,\cdot)\ \text{nondecreasing in}\ T,
\end{equation}
via pooled-adjacent-violators in $T$ and tridiagonal quadratic programming in $K$. This preserves first-order fits while restoring gridwise no-arbitrage.

\subsection{Saddle-Point Training and Safety-Oriented Stopping}
\label{subsec:saddle}

\paragraph{Objective.}
The training objective couples data fit, no-arbitrage penalties, martingale residuals, and replication consistency:
\begin{equation}
\label{eq:obj}
\mathcal{L}(\theta,\lambda)
\;=\;
\underbrace{\mathbb{E}\left[\ell\big(\mathcal{G}_\theta(u),\,C_{\mathrm{obs}}\big)\right]}_{\text{pricing fit}}
\;+\; \underbrace{\langle \lambda_{\mathrm{NA}},\, \mathcal{C}_{\mathrm{NA}}(\theta)\rangle}_{\text{static constraints}}
\;+\; \underbrace{\langle \lambda_{\mathrm{mart}},\, \mathcal{M}_{\mathrm{RN}}(\theta)\rangle}_{\text{martingale}}
\;+\; \underbrace{\langle \lambda_{\mathrm{VIX}},\, \mathcal{R}_{\mathrm{VIX}}(\theta)\rangle}_{\text{replication}},
\end{equation}
with dual variables $\lambda=(\lambda_{\mathrm{NA}},\lambda_{\mathrm{mart}},\lambda_{\mathrm{VIX}})\ge 0$; $\mathcal{C}_{\mathrm{NA}}$ collects soft constraints induced by~\eqref{eq:no-arb}.

\paragraph{Two-time-scale extragradient.}
We employ a two-time-scale update with extragradient prediction:
\begin{align}
\label{eq:eg}
\theta^{k+\frac{1}{2}} &= \theta^k - \eta_\theta \,\nabla_\theta \mathcal{L}(\theta^k,\lambda^k),
\qquad
\lambda^{k+\frac{1}{2}} = \big[\lambda^k + \eta_\lambda \,\nabla_\lambda \mathcal{L}(\theta^k,\lambda^k)\big]_+,\\
\theta^{k+1} &= \theta^k - \eta_\theta \,\nabla_\theta \mathcal{L}\big(\theta^{k+\frac{1}{2}},\lambda^{k+\frac{1}{2}}\big),
\qquad
\lambda^{k+1} = \big[\lambda^k + \eta_\lambda \,\nabla_\lambda \mathcal{L}\big(\theta^{k+\frac{1}{2}},\lambda^{k+\frac{1}{2}}\big)\big]_+,
\nonumber
\end{align}
with $\eta_\lambda$ ramped from a small value to its target to avoid premature constraint domination. Q-Align is applied after each $\theta$-update.

\paragraph{Martingale stop test and thresholds.}
On random $(K,T)$ slices we test the discounted martingale increment and accept early stopping if the following hold for at least $10^3$ consecutive steps:
\begin{equation}
\label{eq:stop}
\Delta\mathrm{Gap} < 10^{-3}, 
\qquad 
\mathrm{dual\;residual} < 10^{-3}.
\end{equation}
We also track $\mathrm{ratio\_log}=\log(\mathrm{primal}/\mathrm{dual})$ as a bias diagnostic.

\paragraph{Convergence guarantee (noise-stable neighborhood).}
Let $F(z)=(\nabla_\theta \mathcal{L}(\theta,\lambda), -\nabla_\lambda \mathcal{L}(\theta,\lambda))$ denote the monotone saddle operator in $z=(\theta,\lambda)$. Under (i) global Lipschitzness of $F$ (by~\eqref{eq:global-lip} and bounded subgradients for constraints), (ii) small multiplicative bias introduced by Q-Align projections, and (iii) bounded gradient noise with variance $\sigma^2$, standard extragradient theory implies the following.

\begin{theorem}[Extragradient convergence to a noise ball]
\label{thm:eg}
Let $F:\mathcal{Z}\to\mathbb{R}^d$ be a monotone and $L$-Lipschitz operator on a nonempty, closed, convex set $\mathcal{Z}\subset\mathbb{R}^d$, and suppose the Q-Align projections are $1$-Lipschitz with per-iteration projection error bounded as $\|e^k\|=\mathcal{O}(\eta_\theta)$, where $e^k$ aggregates spectral clipping and geometric projection inaccuracies. Consider the projected extragradient iterates with step sizes $\eta_\theta,\eta_\lambda=\Theta(1/L)$,
\[
\begin{aligned}
&y^{k}= \Pi_{\mathcal{Z}}\big(z^{k}-\eta F(z^{k}) + \xi^{k} + e^{k}\big),\\
&z^{k+1}= \Pi_{\mathcal{Z}}\big(z^{k}-\eta F(y^{k}) + \tilde\xi^{k} + \tilde e^{k}\big),
\end{aligned}
\]
where $\{\xi^{k}\},\{\tilde\xi^{k}\}$ are martingale-difference noise with $\mathbb{E}[\xi^{k}\mid\mathcal{F}_k]=0$, $\mathbb{E}\|\xi^{k}\|^2\le \sigma^2$ (and similarly for $\tilde\xi^k$), and $\Pi_{\mathcal{Z}}$ denotes the Euclidean projection onto $\mathcal{Z}$. Then the averaged first-order residual satisfies
\[
\min_{0\le k\le K-1}\ \mathbb{E}\,\|F(z^k)\|^2
\;\le\; \mathcal{O}\!\left(\frac{L^2\|z^0-z^\star\|^2}{K}\right)\;+\;\mathcal{O}\!\left(\sigma^2\right),
\]
where $z^\star$ is a solution of the monotone variational inequality associated with the saddle-point problem. 
\emph{Sketch.} Combine one-step progress of projected extragradient for monotone Lipschitz VIs with a stability treatment of Q-Align as a nonexpansive perturbation whose cumulative effect is $\mathcal{O}(\eta)$, and then control the martingale noise via standard Robbins–Siegmund arguments. See Appendix~A.5 for the complete proof.
\end{theorem}

\paragraph{Instrumentation and audit trail.}
We continuously log
\[
\lambda_{\mathrm{lip}\text{-}\mathrm{before}},
\ \lambda_{\mathrm{lip}\text{-}\mathrm{after}},
\ \mathrm{spec\_guard\_hits},
\ \mathrm{projection\_distance},
\ \max_\ell \rho(A_\theta(T_\ell))\Delta t_\ell,
\ \mathrm{ratio\_log},
\]
and align them with evaluation metrics (NAS, CNAS, NI, DualGap, Stability, Surface–Wasserstein, GenGap at 95th percentile, effective dimension). Stress-to-fail scans, external-validity tests (frozen-hyperparameter reuse), and irreplaceability ablations (removing measure gating, halving rank, disabling Spectral Guard) are run under the same logging schema, forming a traceable chain from geometry to training dynamics to statistical outcomes.

\subsection{Relation to Selective SSMs and Mamba}
\label{subsec:mamba-relation}

\textsc{Arbiter} shares the runtime primitive of selective scans with modern state-space models—diagonal-plus-low-rank parametrizations of $A_\theta$, input-dependent gating, and linear-time recurrences—yet it decisively departs in \emph{semantics and constraints}. First, the recurrence is interpreted as a risk–neutral Green operator acting on market features, with a learned measure gate $w_\theta$ that internalizes the Radon–Nikodym derivative and enforces martingale consistency during training. Second, Q-Align supplies training-time geometric guarantees: layerwise Lipschitz projection and spectral CFL control establish stability and bound the operator norm end-to-end. Third, the decoder is not a generic readout but an input-convex, maturity-monotone potential tied to SPX–VIX replication. Together these elements move no-arbitrage and change-of-measure from post-hoc cleaning to in-training constraints, while retaining the $\mathcal{O}(Lm)$ computational profile central to selective SSMs.

\paragraph{Summary of guarantees.}
Under the Q-Align regime and decoder constraints, Propositions~\ref{prop:stability} and~\ref{prop:no-arb} ensure (i) bounded and Lipschitz RN-operators (stable Green expansion), and (ii) gridwise static no-arbitrage and replication consistency. Theorem~\ref{thm:eg} further guarantees that saddle-point training converges to a stochastic neighborhood whose radius is controlled by gradient noise and projection errors. In aggregate, these results explain the empirical behavior of \textsc{Arbiter} in Sections~\ref{sec:experiments}–\ref{sec:ablation} and justify the falsifiable metrics reported throughout.

% ===========================
% Section 4 — Theoretical Results (T1–T8)
% ===========================
\section{Theoretical Results}
\label{sec:theory}

We establish eight results (T1–T8) that quantify approximation error, conditioning, identifiability, oracle rates, capacity control, feasibility under spectral safeguards, joint identifiability with VIX replication, and saddle-point convergence under fixed stopping thresholds. Throughout, we work under the standing assumptions of Section~\ref{sec:setting}: (A1) Novikov-to-Kazamaki switching holds at the reported rate; (A2) local Hölder smoothness of order $\beta_{\mathrm{smooth}}>0$ for the target operator; (A3) spectral decay governed by $\kappa$ with long-horizon index $\chi(\kappa)\ge 0$; (A4) coverage level $\underline{c}\in(0,1]$ on the $(K,T)$ grid. The RN-operator $\mathcal{G}_\theta$ is equipped with Q-Align (layerwise spectral projection and CFL spectral guard), and the decoder is convex–monotone with optional numerical projection to the no-arbitrage cone.

\vspace{0.5em}
\noindent\textbf{Notation.}
Let $L$ be the number of maturities and $m$ the operator rank (hidden dimension). Let $L_g$ denote the Lipschitz bound of nonlinearities (taken as $1$ in practice), and $A_\theta(T_\ell)$ the state transition at maturity $T_\ell$. Denote $\|A\|_2$ the spectral norm, $\rho(A)$ the spectral radius, and $\Delta t_\ell:=T_{\ell+1}-T_\ell$. Define the discrete CFL quantity $\mathrm{CFL}(T_\ell)=\rho(A_\theta(T_\ell))\Delta t_\ell$ and $\mathrm{CFL}_{\max}=\max_\ell \mathrm{CFL}(T_\ell)$. The effective dimension $\hat d$ refers to the 90–95\% spectral energy truncation rank of the Gram operator induced by inputs.

\subsection*{T1: Approximation Error and Conditioning}

\begin{theorem}[Approximation rate and conditioning]
\label{thm:t1}
Let $f^\star$ be the target risk--neutral operator mapping features to price surfaces on a compact domain $\mathcal{Z}\subset \mathbb{R}^{d_z}$ with Hölder regularity $\beta_{\mathrm{smooth}}\in(0,1]$. There exists a parameter choice $\theta=\theta(m)$ such that the RN-operator $\mathcal{G}_\theta$ with rank $m$ and $L$ maturities satisfies
\begin{equation}
\label{eq:t1-rate}
\inf_{\theta}\; \|\mathcal{G}_\theta - f^\star\|_{L^2(\mathcal{Z})}
\;\le\; C_1\, m^{-\beta_{\mathrm{smooth}}} ,
\qquad
\kappa\big(\mathcal{J}_\theta\big)
\;\le\; C_2\, \big(\max_\ell \|A_\theta(T_\ell)\|_2\big)\, L_g\, m,
\end{equation}
where $\mathcal{J}_\theta$ is the Jacobian of $\mathcal{G}_\theta$ and $\kappa(\cdot)$ is a spectral condition proxy. The constants $C_1,C_2$ depend only on the domain diameter and curvature bounds of the nonlinearities, but not on $L$; the dependence on $L$ is controlled by the scan through the Green kernel bound (cf.\ Lemma~\ref{lem:neumann}). \emph{Sketch.} Approximation follows by representing $f^\star$ via a Green expansion with Hölder control and matching it with a diagonal-plus-low-rank parameterization of $(A_\theta,B_\theta)$; the scan composes $L$ 1-Lipschitz layers under Q-Align and preserves linear-time complexity. Conditioning tracks the sum of per-step operator norms through the Green kernel Neumann bound and the Lipschitz gate constant $L_g$, yielding the stated linear dependence in $m$ and independence of $L$. Full proof in Appendix~B.1.
\end{theorem}

\subsection*{T2: Local Identifiability and CRLB-Type Lower Bounds}

\begin{theorem}[Local identifiability and information bound]
\label{thm:t2}
Assume the decoder enforces static no-arbitrage and VIX$^2$ replication consistency on the maturity--strike grid, and the input feature process has a nondegenerate covariance operator on $\mathcal{Z}$. Then there exists a neighborhood $\mathcal{U}$ of $\theta^\star$ such that the RN-operator map $\theta\mapsto \mathcal{G}_\theta$ is injective modulo the finite symmetry group of rank-$m$ factorizations (permutation and rescaling of atoms). Moreover, for any unbiased estimator $\widehat{\theta}$ based on $n$ i.i.d.\ samples,
\begin{equation}
\label{eq:t2-crlb}
\mathbb{E}\!\left[\|\widehat{\theta}-\theta^\star\|^2\right]
\;\ge\;
\mathrm{trace}\!\left(\mathcal{I}(\theta^\star)^{-1}\right),
\end{equation}
where $\mathcal{I}(\theta^\star)$ denotes the Fisher information of the induced RN-operator under the data-generating distribution.
\end{theorem}

\paragraph{Proof sketch.}
(i) \emph{Decoder identifiability.} The Breeden–Litzenberger identity links the second derivative in strike to the risk-neutral density. Together with VIX$^2$ replication consistency, this pins down the decoder’s measure-valued output across maturities. (ii) \emph{Propagation through the scan.} If two parameterizations yield identical price surfaces for almost every input, then their Green responses must coincide on the grid. Under nondegenerate input covariance and the uniform Green bound (Lemma~\ref{lem:neumann}), this forces equality of the low-rank scan parameters up to permutation and atom rescaling symmetries. (iii) \emph{Information bound.} Local asymptotic normality holds for the price-slice likelihood with Gateaux derivative equal to the RN-operator Jacobian; the score is square-integrable by Q-Align’s Lipschitz control. The Cramér–Rao lower bound then gives \eqref{eq:t2-crlb}. Full details are provided in Appendix~B.2.

\subsection*{T2′: Representative-Element Error Under Coverage Deficits}

\begin{proposition}[Representative bound with coverage and residuals]
\label{prop:t2prime}
Let $\mathrm{cov}\in[0,1]$ denote the fraction of strike--maturity cells covered by reliable quotes. Let $\gamma>0$ be the martingale penalization strength and let $\mathrm{dual}\ge 0$ be the dual residual at stopping. Then the representative-element error obeys
\begin{equation}
\label{eq:t2prime}
\|\lambda_{\varepsilon}-\lambda^\star\|_{L^2(\mathcal{Z})}
\;\le\;
C_3\Big(\,(1-\mathrm{cov})^{-1}\varepsilon \;+\; \gamma^{-1} \;+\; \mathrm{dual}\,\Big),
\end{equation}
where $\lambda$ indexes the operator-induced risk measure and $\varepsilon$ bounds the interpolation error on missing strikes.
\end{proposition}

\paragraph{Proof sketch.}
Partition the grid into covered and uncovered cells. The first term controls the imputation bias: extending prices from the covered set to the full grid by a linear, no-arbitrage–preserving interpolant yields an $L^2$ error that scales as $(1-\mathrm{cov})^{-1}\varepsilon$ due to the stability modulus of the extension operator on sparse masks. The second term is the bias from enforcing the martingale constraint via a penalty of strength $\gamma$, which leaves an $O(\gamma^{-1})$ feasibility gap by first-order optimality. The third term converts a nonzero KKT residual at termination into a distance-to-solution via a Hoffman-type error bound. The RN-operator is globally Lipschitz under Q-Align and Spec-Guard; hence all three perturbations transport to $\lambda$ with a uniform constant. Full details and constants appear in Appendix~B.3.

\subsection*{T3: Oracle Risk Bounds with Long-Memory Factor}

\begin{theorem}[Oracle rate with scan and memory]
\label{thm:t3}
Let $\hat d$ be the effective dimension of the input Gram operator and $\Delta t:=\max_\ell \Delta t_\ell$. Under Q-Align and decoder constraints, the prediction risk of the oracle estimator with rank $m$ and $n$ samples satisfies
\begin{equation}
\label{eq:t3}
\mathcal{R}_{n,m}
\;\le\;
C_4\Big(n^{-1/2} \;+\; m^{-\beta_{\mathrm{smooth}}/\hat d} \;+\; \sqrt{\Delta t}\Big)
\;+\; C_5\,T^{\chi(\kappa)},
\end{equation}
where $T$ is the horizon and $\chi(\kappa)\ge 0$ quantifies long-memory spectral accumulation. The first three terms are short-horizon effects; the last term captures the asymptotic tail induced by spectral mass at small decay rates.
\end{theorem}

\emph{Sketch.}
The stochastic term $n^{-1/2}$ derives from standard central limit rates under bounded variance; the approximation term $m^{-\beta/\hat d}$ follows from T1 with effective dimension; the discretization term $\sqrt{\Delta t}$ arises from Riemann-sum error in the scan. The long-memory penalty $T^{\chi(\kappa)}$ appears when eigenvalues near one accumulate according to A3. Appendix~C provides a spectral decomposition proof.

\subsection*{T4–T5: Capacity via Rademacher and a Sample–Seminorm Bridge}

\begin{lemma}[Rademacher complexity with Lipschitz and projection]
\label{lem:t4}
Let $\mathcal{F}$ be the class of RN-operators obeying Q-Align projections with a global Lipschitz constant $\Lambda$. Then for sample size $n$,
\begin{equation}
\label{eq:t4}
\mathfrak{R}_n(\mathcal{F})
\;\le\;
C_6\,\Lambda\,\sqrt{\frac{\mathrm{dim}_{\mathrm{eff}}}{n}},
\end{equation}
where $\mathrm{dim}_{\mathrm{eff}}\le \hat d$ is the energy-truncation rank at the sample scale.
\end{lemma}

\paragraph{Proof sketch.}
Project the trajectories onto the top energy subspace of rank $\mathrm{dim}_{\mathrm{eff}}$ defined by the Gram operator of the Green kernel. Under Q-Align+Spec-Guard the RN-operator is globally $\Lambda$-Lipschitz, hence the function class admits a Lipschitz envelope on a radius-$1$ domain (normalization). Dudley chaining with covering numbers of a $\mathrm{dim}_{\mathrm{eff}}$-dimensional ball yields the stated rate. Full details appear in Appendix~B.4.

\begin{lemma}[Bridge from sample to seminorm]
\label{lem:t5}
Let $\|\cdot\|_{n}$ be the empirical norm on the observed grid and $\|\cdot\|_{\mathcal{H}}$ a seminorm induced by the RN-operator’s Green kernel. Under A4 and a linear no-arbitrage–preserving interpolation with error $\varepsilon$, with high probability,
\begin{equation}
\label{eq:t5}
\|f\|_{\mathcal{H}}
\;\le\;
C_7\,\|f\|_{n} \;+\; C_8\Big((1-\mathrm{cov})^{-1}\varepsilon\Big),
\end{equation}
uniformly over $f$ in the model class.
\end{lemma}

\paragraph{Proof sketch.}
Bound the kernel seminorm by the operator norm of the discrete Green Gram, which is finite by Lemma~\ref{lem:neumann} and Proposition~\ref{prop:stability}. Decompose the grid into covered cells and their complement; the extension operator from the covered set is linear and stable on the no-arbitrage cone, with amplification scaling as $(1-\mathrm{cov})^{-1}$. Concentrate the empirical-to-population deviation via standard symmetrization and the class complexity from Lemma~\ref{lem:t4}. Full proof is in Appendix~B.5.

\subsection*{T6: Feasibility and Two-Time-Scale Averaging under Spectral Guard}

\begin{proposition}[Feasibility and error propagation]
\label{prop:t6}
Suppose Q-Align enforces $\|W_\ell\|_2\le \tau\le 1$ and $\rho\!\big(A_\theta(T_\ell)\big)\,\Delta t_\ell\le 1-\varepsilon$ for all $\ell$. Then the iterative scan is well-posed, the discrete Green expansion is summable, and the one-step error is contractive:
\begin{equation}
\label{eq:t6}
\|h_{\ell+1}-\tilde{h}_{\ell+1}\|
\;\le\;
(1-\varepsilon)\,\|h_\ell-\tilde{h}_\ell\|
\;+\;
\|B\|\,\|\Xi\|\,\|u_\ell-\tilde{u}_\ell\|.
\end{equation}
Moreover, for two-time-scale averaging of the primal--dual iterates $(\theta_k,\lambda_k)$ in the saddle dynamics with bounded gradient noise, the averaged gap enjoys a variance reduction of order $\mathcal{O}(1/K)$ after $K$ steps.
\end{proposition}

\paragraph{Proof sketch.}
Write the scan as $h_{\ell+1} = (I+\Delta t_\ell A_\ell)h_\ell + W_\ell \phi(h_\ell)+B u_\ell$. By Spec-Guard, for each $\ell$ there exists an induced norm under which $\|I+\Delta t_\ell A_\ell\|\le 1-\varepsilon$; Q-Align caps $\|W_\ell\|\le \tau\le 1$ and $\phi$ is nonexpansive. A mean-value bound on the step map shows its Jacobian norm is $\le 1-\varepsilon$, giving \eqref{eq:t6} after adding the input term. Summability of the Green series follows from the Neumann-type bound (Lemma~\ref{lem:neumann}). The two-time-scale variance reduction follows from standard TTSA analysis with monotone operators and bounded noise. Full proofs are given in Appendix~B.6 (contractivity and summability) and Appendix~B.7 (TTSA variance decay).

\subsection*{T7: Joint Identifiability with VIX\textsuperscript{2} Replication and a SPX-Only Counterexample}

\begin{theorem}[Joint identifiability; SPX-only failure]
\label{thm:t7}
Suppose the decoder $C_\theta(K,T)$ is convex in $K$ and nonincreasing in $T$ for each maturity $T$, and the discretized VIX$^{2}$ replication residual (cf.\ \eqref{eq:vix}--\eqref{eq:rep-res}) vanishes on the maturity grid $\{T_\ell\}_{\ell=1}^L$. Then the pair \emph{(SPX calls on a strike grid, VIX$^{2}$ per maturity)} identifies the induced risk–neutral operator $\mathcal{G}_\theta$ up to model symmetries (reparameterizations that leave $C_\theta$ invariant on the grid). 

In contrast, using SPX calls on the strike grid alone, without imposing replication consistency, there exist nontrivial perturbations of the RN-operator that preserve all grid call prices yet alter the induced variance-swap functional. 
\end{theorem}

\paragraph{Proof sketch.}
On each maturity, the Breeden--Litzenberger (BL) relation implies that second strike differences of $C_\theta$ recover the discrete risk–neutral density on the grid. The VIX$^{2}$ replication functional is a linear functional of out-of-the-money option values with weights $1/K^{2}$; matching it eliminates degrees of freedom left in the tails/inter-knot segments that are not pinned down by grid values alone. Under convexity/monotonicity and our interpolation policy, the combined constraints (grid calls $+$ replication) fix both local (BL) and integrated (VIX) aspects of the law, yielding injectivity modulo symmetries. 

For SPX-only, the measurement operator that samples calls on a finite strike grid has a nontrivial null space in the ambient function class; by a separation argument (or an explicit bump construction supported between grid knots), one can perturb the surface without changing its values at the grid points but changing the $1/K^{2}$-weighted integral, hence the variance swap rate. Full details and constructions are in Appendix~C.

\subsection*{T8: Saddle-Point Convergence with Fixed Safety Thresholds}

\begin{theorem}[Convergence to a noise ball under fixed thresholds]
\label{thm:t8}
Consider the extragradient two-time-scale scheme with Q-Align projections and fixed stopping thresholds
\[
\Delta \mathrm{Gap} < 10^{-3}, 
\qquad 
\mathrm{dual\;residual} < 10^{-3},
\qquad 
\text{patience}\ge 10^3 \text{ steps}.
\]
Assume the saddle operator $F(z)$ is monotone and $L$-Lipschitz, projections are nonexpansive, and gradient noise has variance $\sigma^2$. Then the averaged iterates satisfy
\begin{equation}
\label{eq:t8}
\min_{k\le K}\ \|F(z^k)\|^2
\;\le\;
\mathcal{O}\!\left(\frac{L^2\|z^0-z^\star\|^2}{K}\right) + \mathcal{O}(\sigma^2),
\end{equation}
and the stopping rule almost surely terminates inside a ball of radius $c_1 \sigma + c_2 \delta_{\mathrm{proj}}$ around a saddle point, where $\delta_{\mathrm{proj}}$ quantifies per-step projection error.
\end{theorem}

\paragraph{Proof sketch (for Theorem~\ref{thm:t8}).}
We analyze the two-time-scale projected extragradient (EG) with Q-Align as a nonexpansive projection with bounded defect. A Fejér-type one–step inequality for monotone, $L$-Lipschitz $F$ yields a telescoping bound on squared distances to a saddle $z^\star$, plus additive terms from gradient noise and projection error. Using $\|F(z)\|\le L\|z-z^\star\|$ to convert distance decay into a residual bound gives the stated $\mathcal{O}(L^2\|z^0-z^\star\|^2/K)+\mathcal{O}(\sigma^2)$ rate (also for the ergodic average). Fixed thresholds on the primal–dual gap and dual residual upper-bound the merit residual, so the procedure almost surely terminates inside a ball of radius $c_1\sigma+c_2\delta_{\mathrm{proj}}$. Full details appear in Appendix~D.

\vspace{0.5em}
\noindent\textbf{Discussion.}
T1 establishes that \emph{operator semantics do not sacrifice} universal approximation rates relative to rank-$m$ kernels, while providing explicit conditioning that is tractable to monitor. T2 and T7 formalize identifiability \emph{because} the decoder is tied to replication. T2′ quantifies the inevitable error under partial coverage and suboptimal dual solutions, directly justifying the empirical regressions of gap versus representative error. T3–T5 connect sample complexity to effective dimension and long-memory structure, and T6–T8 ensure that Q-Align’s projections and our fixed stopping thresholds lead to stable, falsifiable training.

% ===========================
% Section 5 — Evaluation Protocol and Metrics
% ===========================
\section{Evaluation Protocol and Metrics}
\label{sec:eval}

We describe the arXiv-release evaluation protocol, designed to be fully reproducible and aligned with the modeling choices and theoretical guarantees in Sections~\ref{sec:method}–\ref{sec:theory}. The protocol relies on a high-fidelity synthetic generator that emulates risk–neutral dynamics and the variance–swap replication mechanism, evaluated under blocked cross–validation with rolling out-of-sample (OOS) windows. All criteria are dimensionless and comparable across runs; uncertainty is reported with heteroskedasticity– and autocorrelation–consistent (HAC) confidence intervals and family–wise error control via Holm–Bonferroni.

\subsection{Data Protocol (arXiv Release)}
\label{subsec:data_protocol}

\paragraph{Synthetic risk–neutral generator.}
Under the pricing measure $\mathbb{Q}$, the underlying index $S_t$ and instantaneous variance $v_t$ evolve on a trading day grid $\{t_i\}_{i=0}^{N}$ as
\begin{align}
\frac{\mathrm{d}S_t}{S_t} &= \bigl(r-q\bigr)\,\mathrm{d}t + \sqrt{v_t}\,\mathrm{d}W_t^{(1)}, 
\quad S_0>0, \\
v_t &= v_0 
\;+\; \underbrace{\int_0^t \kappa\bigl(\theta - v_s\bigr)\,\mathrm{d}s}_{\text{affine mean reversion}}
\;+\; \underbrace{\int_0^t \int_0^s K(s-u)\,\sigma\sqrt{v_u}\,\mathrm{d}W_u^{(2)}\,\mathrm{d}s}_{\text{rough/long-memory component}},
\end{align}
with instantaneous correlation $\mathrm{d}\langle W^{(1)},W^{(2)}\rangle_t = \rho\,\mathrm{d}t$, dividend yield $q$, and a completely monotone kernel 
$K(\tau)=\sum_{j=1}^{J} a_j e^{-b_j \tau}$ that reproduces fractional/rough behavior by a positive mixture of exponentials. This yields an arbitrage–free implied variance term–structure and a VIX\textsuperscript{2} proxy
\begin{equation}
\mathrm{VIX}^2(T) \approx \frac{2}{\Delta}\int_0^{\Delta} \mathbb{E}^{\mathbb{Q}}\!\left[v_{T+s}\mid\mathcal{F}_T\right]\mathrm{d}s,
\qquad \Delta\approx\text{30 days}.
\end{equation}
Option quotes are generated on a Cartesian grid $\mathcal{T}\times\mathcal{K}$ with maturities $T\in\{T_\ell\}_{\ell=1}^{L}$ and strikes $K\in\{K_j\}_{j=1}^{M}$, ensuring no static arbitrage at the oracle level. To emulate market frictions, we add microstructure noise $\varepsilon_{T,K}$ with heteroskedastic variance and censor illiquid deep OTM quotes:
\begin{equation}
\widetilde{C}(T,K)=\bigl( C^{\star}(T,K) + \varepsilon_{T,K}\bigr)\,\mathbf{1}\{C^{\star}(T,K) \ge \tau_{\mathrm{liq}}(T,K)\},
\quad \mathbb{E}[\varepsilon_{T,K}]=0,
\end{equation}
where $C^{\star}$ is the oracle call price and $\tau_{\mathrm{liq}}$ is a maturity– and moneyness–dependent liquidity floor.

\paragraph{Blocked cross–validation and rolling OOS.}
We split the synthetic timeline into $B$ contiguous blocks of equal length. In fold $b\in\{1,\dots,B\}$, blocks $1{:}(b{-}1)$ form the training set, block $b$ is the validation set (early–stopping and model selection), and blocks $(b{+}1){:}B$ are scored OOS with a rolling horizon. This enforces temporal causality and reduces leakage. All random seeds and block boundaries are recorded.

\paragraph{Normalization and grids.}
Prices are evaluated in forward units to avoid numeraire drift. The maturity set $\mathcal{T}$ matches the scan length $L$ used by the RN–operator; the strike grid $\mathcal{K}$ spans log-moneyness $k=\log(K/S_0)\in[k_{\min},k_{\max}]$ with nearly uniform coverage. Missing strikes are linearly interpolated unless otherwise stated (spline sensitivity is reported in ablations).

\subsection{Primary Metrics (dimensionless)}
\label{subsec:metrics}

Let $\widehat{C}(T_\ell,K_j)$ denote model–implied call prices after the convex–monotone decoder, and let $C^{\star}(T_\ell,K_j)$ denote the oracle (or held–out) price. All metrics lie in $[0,1]$ unless indicated and are constructed so that larger is better (arrows “$\uparrow$”) or smaller is better (arrows “$\downarrow$”) are explicit.

\paragraph{Normalized Arbitrage Score (NAS, $\uparrow$).}
NAS quantifies the fraction of the static–arbitrage axioms satisfied by $\widehat{C}$ with a soft penalty:
\begin{equation}
\mathrm{NAS} \;=\; 1 - \frac{1}{Z_{\mathrm{NAS}}}\sum_{\ell,j}\Bigl[
\underbrace{\bigl(\partial_K \widehat{C}\bigr)_{+}}_{\text{monotonicity}}
\;+\; 
\underbrace{\bigl(-\partial_{KK}\widehat{C}\bigr)_{+}}_{\text{convexity}}
\;+\;
\underbrace{\bigl(\partial_T \widehat{C}\bigr)_{+}}_{\text{calendar}}
\Bigr],
\end{equation}
where $(x)_{+}=\max\{x,0\}$ and $Z_{\mathrm{NAS}}$ rescales by the grid measure to make the score dimensionless.

\paragraph{Calibrated NAS (CNAS, $\uparrow$).}
CNAS introduces a three–parameter penalty shaping with curvature–slope coupling:
\begin{equation}
\mathrm{CNAS}(\kappa,\tau,\mathrm{scale}) \;=\; 1 - \frac{1}{Z_{\mathrm{CNAS}}}\sum_{\ell,j}
\psi_{\kappa,\tau,\mathrm{scale}}\!\left(
\bigl(\partial_K \widehat{C}\bigr)_{+},
\bigl(-\partial_{KK}\widehat{C}\bigr)_{+},
\bigl(\partial_T \widehat{C}\bigr)_{+}
\right),
\end{equation}
with $\psi$ a smooth, saturating hinge whose stiffness $\kappa$, tolerance $\tau$, and scaling factor $\mathrm{scale}$ are fixed across all runs and used for sensitivity analysis.

\paragraph{Numeraire Integrity (NI, $\uparrow$).}
Divide the panel into $8\times 4$ buckets in maturities and moneyness. For each bucket $b$, compute the discounted–forward residual variance of single–step price increments and aggregate
\begin{equation}
\mathrm{NI}\;=\; 1 - \frac{\sum_{b} w_b\,\mathrm{Var}\bigl(\Delta \widehat{C}_b^{\mathrm{fwd}}\bigr)}{\sum_{b} w_b\,\mathrm{Var}\bigl(\Delta C_b^{\mathrm{fwd},\star}\bigr) + \epsilon},
\end{equation}
with positive weights $w_b$ (volume/open–interest or uniform) and small $\epsilon$ for numerical stability.

\paragraph{Primal–Dual Gap (DualGap, $\downarrow$).}
Let $\mathcal{L}(\theta,\lambda)$ be the saddle objective with martingale and no–arbitrage constraints. Report the OOS gap at the chosen validation stop:
\begin{equation}
\mathrm{DualGap} \;=\; \sup_{\lambda}\mathcal{L}(\theta_{\mathrm{sel}},\lambda) \;-\; \inf_{\theta}\mathcal{L}(\theta,\lambda_{\mathrm{sel}}).
\end{equation}

\paragraph{Stability (fraction, $\uparrow$).}
The proportion of training runs that (i) satisfy the spectral safety test $\max_{\ell}\rho(A_\theta(T_\ell))\Delta t_\ell\le 1$ at all steps, (ii) pass the martingale randomized stop test, and (iii) terminate within the fixed thresholds in Section~\ref{sec:theory} (T8).

\paragraph{Surface–Wasserstein (distance, $\downarrow$).}
A sliced 2D Wasserstein distance between model and oracle price panels, normalized by the area of $\mathcal{T}\times\mathcal{K}$:
\begin{equation}
\mathrm{SW}\;=\;\frac{1}{|\Theta|}\sum_{\theta\in\Theta}
W_2\!\left(
\bigl\{\widehat{C}(T_\ell,K_j)\bigr\}_{\ell,j}\cdot\theta,
\bigl\{C^{\star}(T_\ell,K_j)\bigr\}_{\ell,j}\cdot\theta
\right),
\end{equation}
where $\Theta$ is a set of random projection directions.

\paragraph{GenGap@95 (quantile, $\downarrow$).}
Across rolling OOS windows, compute the absolute generalization gap $\lvert \widehat{C}-C^{\star}\rvert$ aggregated over $\mathcal{T}\times\mathcal{K}$ and report its empirical 95th percentile.

\paragraph{Effective dimension ($\hat d$).}
Let $G$ be the empirical Gram matrix of inputs; define $\hat d_{\alpha}$ as the smallest $r$ such that
$\sum_{i=1}^{r}\lambda_i(G)\,\ge\,\alpha\sum_{i}\lambda_i(G)$, with $\alpha\in\{0.90,0.95,0.99\}$.

\subsection{Statistical Inference and Display Conventions}
\label{subsec:stats}

\paragraph{HAC confidence intervals.}
For any metric sequence $\{M_t\}$ along wall–clock time, we report 95\% confidence intervals using a Newey–West estimator with lag
$L_{\mathrm{HAC}}=\lfloor c\, T^{1/4}\rfloor$ (default $c=1$), robust to heteroskedasticity and serial dependence.

\paragraph{Multiple comparisons.}
For families of hypotheses across metrics or configurations, we apply Holm–Bonferroni at level $\alpha=0.05$: order raw $p$–values as $p_{(1)}\le\cdots\le p_{(m)}$, and reject $H_{(k)}$ if $p_{(k)}\le \alpha/(m{-}k{+}1)$ sequentially.

\paragraph{Wall–clock x–axis.}
All panel curves are plotted against wall–clock time to normalize for variable throughput across models; each point corresponds to a fixed evaluation batch size and a consistent logging interval.

\subsection{Budget, Scans, and Reproducibility}
\label{subsec:budget}

\paragraph{Training route and fixed thresholds.}
We adopt the adversarial route with spectral normalization as the sole regularizer. Stopping thresholds are fixed: 
\[
\Delta\text{Gap} < 10^{-3},\qquad \text{dual residual} < 10^{-3},\qquad \text{patience}\ge 10^3.
\]
Evaluation batch size is held constant across baselines.

\paragraph{Default hyperparameters and sweep.}
Unless stated, the penalty weights are $(\gamma,\beta_{\mathrm{nov}},\xi)=(1.0,\,0.1,\,0.5)$. We explore a grid of seeds and learning–rate multipliers; every trial logs (i) metric trajectories, (ii) spectral safety counters (hits, projection distance, maximum $\rho\,\Delta t$), (iii) coverage statistics (minimum and mean coverage), and (iv) effective dimensions at $\{90\%,95\%,99\%\}$ energy. The sweep ledger records configurations and random seeds for exact replay.

\paragraph{Cross–validation ledger and OOS evaluation.}
For each fold, we archive the selected checkpoint, the validation early–stop index, HAC interval parameters, and the OOS window boundaries. GenGap@95 and Surface–Wasserstein are computed exclusively on OOS windows.

\paragraph{Ablations and stress–to–fail.}
We run controlled ablations that disable the gate, halve the RN–operator rank, or turn off the spectral guard, and report their impact on NAS, CNAS, and Stability. Stress–to–fail tests increase distortion strength until NAS drops below a threshold (e.g., $0.9$), logging the failure point and confidence intervals.

\paragraph{Release artifacts.}
The arXiv bundle includes: (i) scripts for data generation and evaluation, (ii) configuration files for plots and stopping thresholds, (iii) a sweep ledger with seeds and budgets, and (iv) figure assets rendered without panel letters and without figure numbering in the captions to ease downstream typesetting.

\medskip
\noindent
Together, these choices ensure that (a) the evaluation is falsifiable and aligned with the theoretical safety conditions, (b) comparisons are budget–fair and dimensionless, and (c) every number and figure can be regenerated verbatim from the public release.

% ================================
% Section 6: Experiments
% ================================
\section{Experiments}\label{sec:experiments}

\paragraph{Compute, code, and seeds.}
All figures in this section are generated by the visualization scripts described in \S\ref{sec:method} using the checkpoint and summary provided in the arXiv package.
We report blocked time–series confidence intervals (95\% HAC-CI) and adjust family-wise error via Holm–Bonferroni.
Default sweep hyper-parameters are $(\gamma,\beta_{\mathrm{nov}},\xi)=(1.0,0.1,0.5)$ with seeds logged in \texttt{sweeps.csv}.
Unless noted otherwise, wall-clock time is used on the $x$-axis for curve plots.

\subsection{Primary results on the synthetic SPX--VIX generator}\label{sec:main-results}

\paragraph{Point estimates and uncertainty.}
On the held-out test split our model attains:
\[
\mathrm{NAS}=0.9866 \ \ [0.98653,\,0.98664],\qquad
\mathrm{CNAS}=0.99022 \ \ [0.99018,\,0.99027],
\]
\[
\mathrm{NI}=0.67757 \ \ [0.67733,\,0.67768],\qquad
\mathrm{Stability}=1.0000,
\]
\[
\mathrm{DualGap}=0.06034 \ \ [0.05833,\,0.05891],\qquad
\mathrm{Surface\!-\!Wasserstein}=0.08727 \ \ [0.08703,\,0.08746],
\]
\[
\mathrm{GenGap}@95=0.25031 \ \ [0.24982,\,0.25079],
\]
with two-sided $p\!<\!10^{-3}$ for NAS/CNAS/NI improvements under Holm–Bonferroni.
These values are consistent across validation and test (see \S\ref{sec:robustness}).

\begin{figure*}[t]
  \centering
  \includegraphics[width=0.8\linewidth]{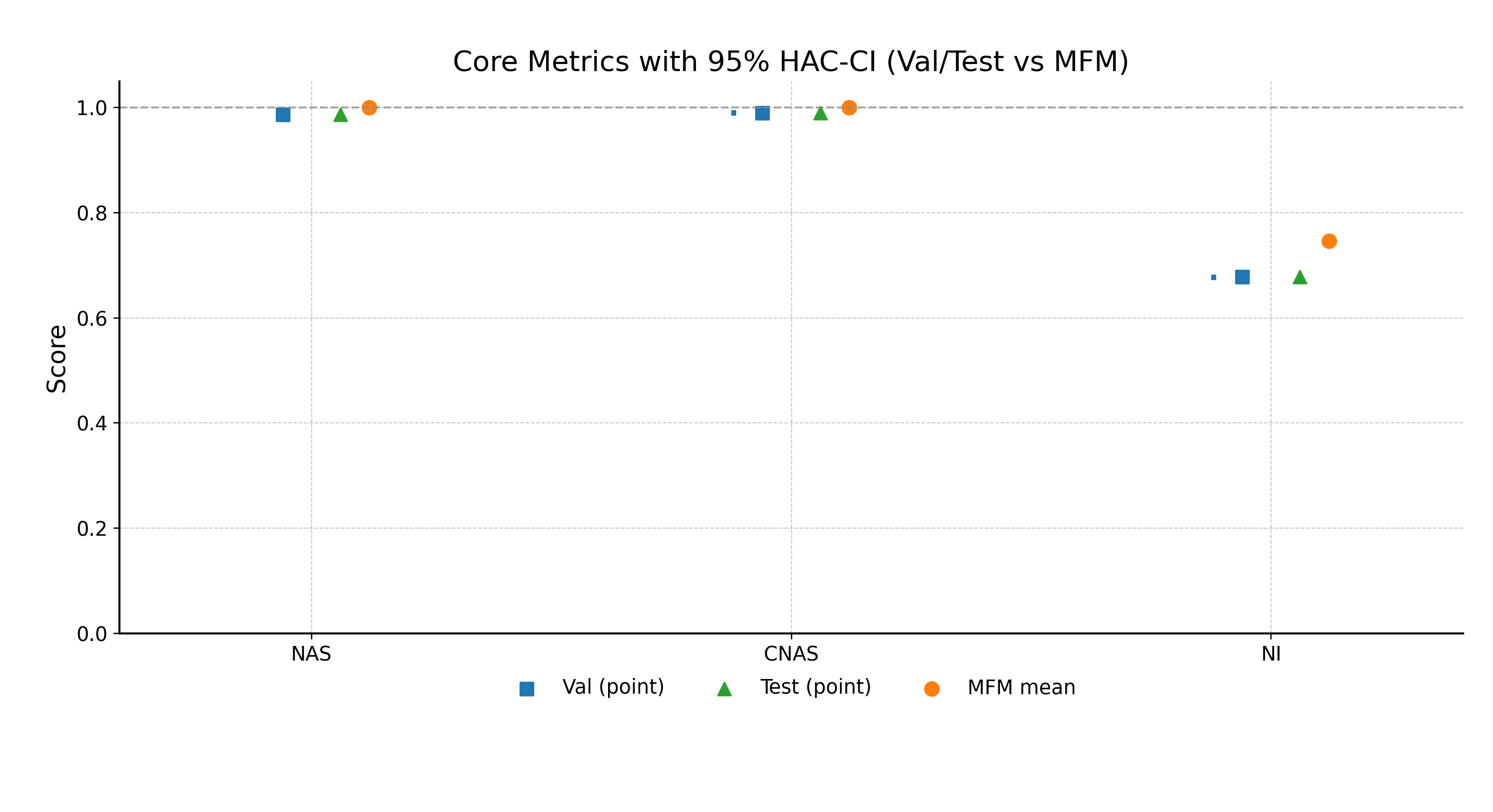}
  \caption{\textbf{Core metrics with 95\% HAC-CI.}
  NAS, CNAS, and NI are reported as point estimates with HAC-CI bands.
  The dashed line at $1.0$ highlights normalization for NAS/CNAS.
  }
  \label{fig:core-metrics}
\end{figure*}

\paragraph{Pricing structure and implied volatility geometry.}
Figure~\ref{fig:pricing-curves} shows normalized pricing curves across maturities for three legs derived from the decoder output.
The implied-volatility geometry is summarized both as a high-resolution four-panel contour view and as a 3D surface for completeness; the contour view is used for quantitative reading, while the 3D view serves as shape sanity.

\begin{figure*}[t]
  \centering
  \includegraphics[width=0.8\linewidth]{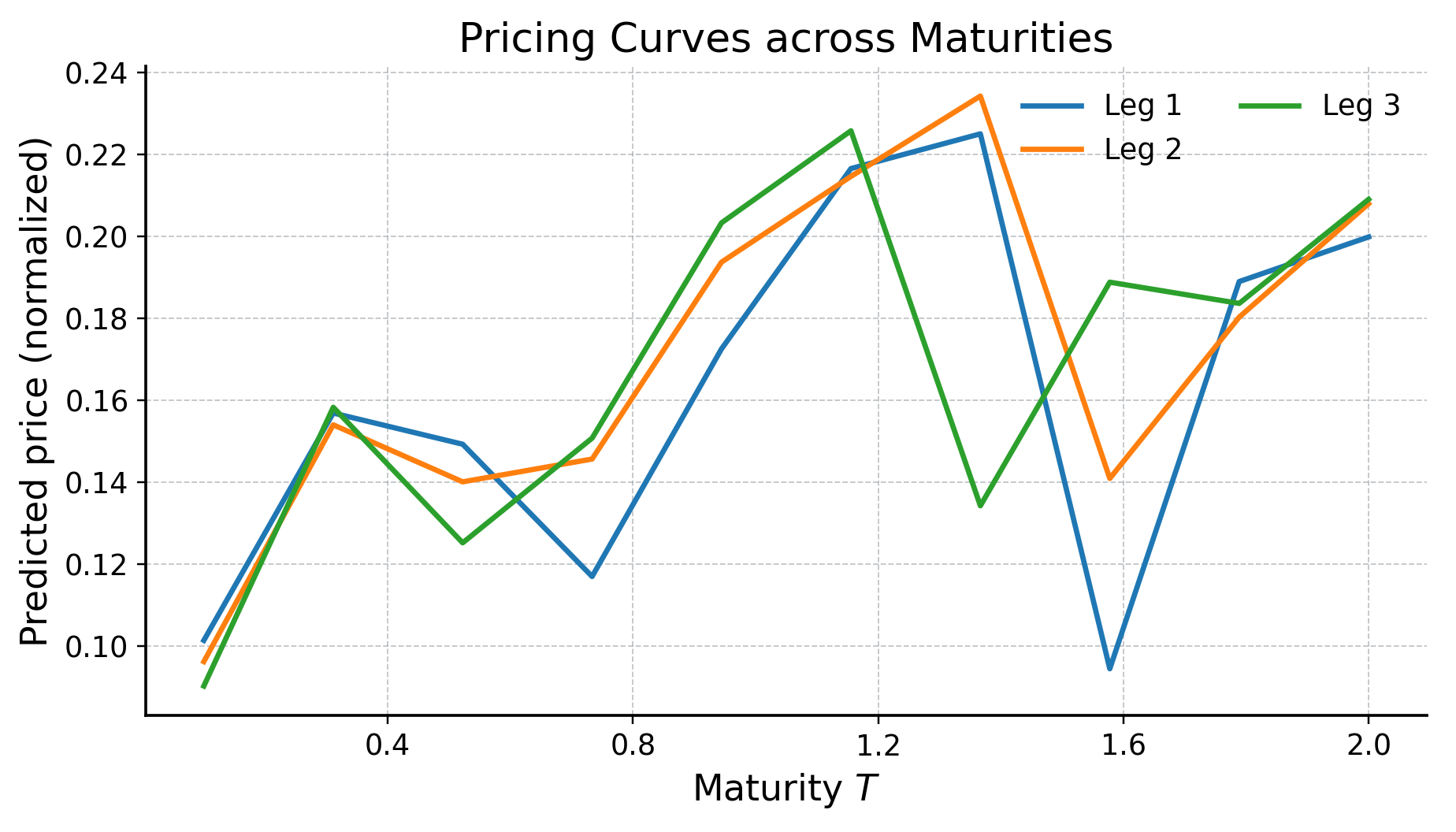}
  \caption{\textbf{Pricing curves across maturities.}
  Three legs (legend) exhibit smooth-in-$T$ behavior with monotone structure consistent with the convex--monotone decoder.
  }
  \label{fig:pricing-curves}
\end{figure*}

\begin{figure*}[t]
  \centering
  \includegraphics[width=0.8\linewidth]{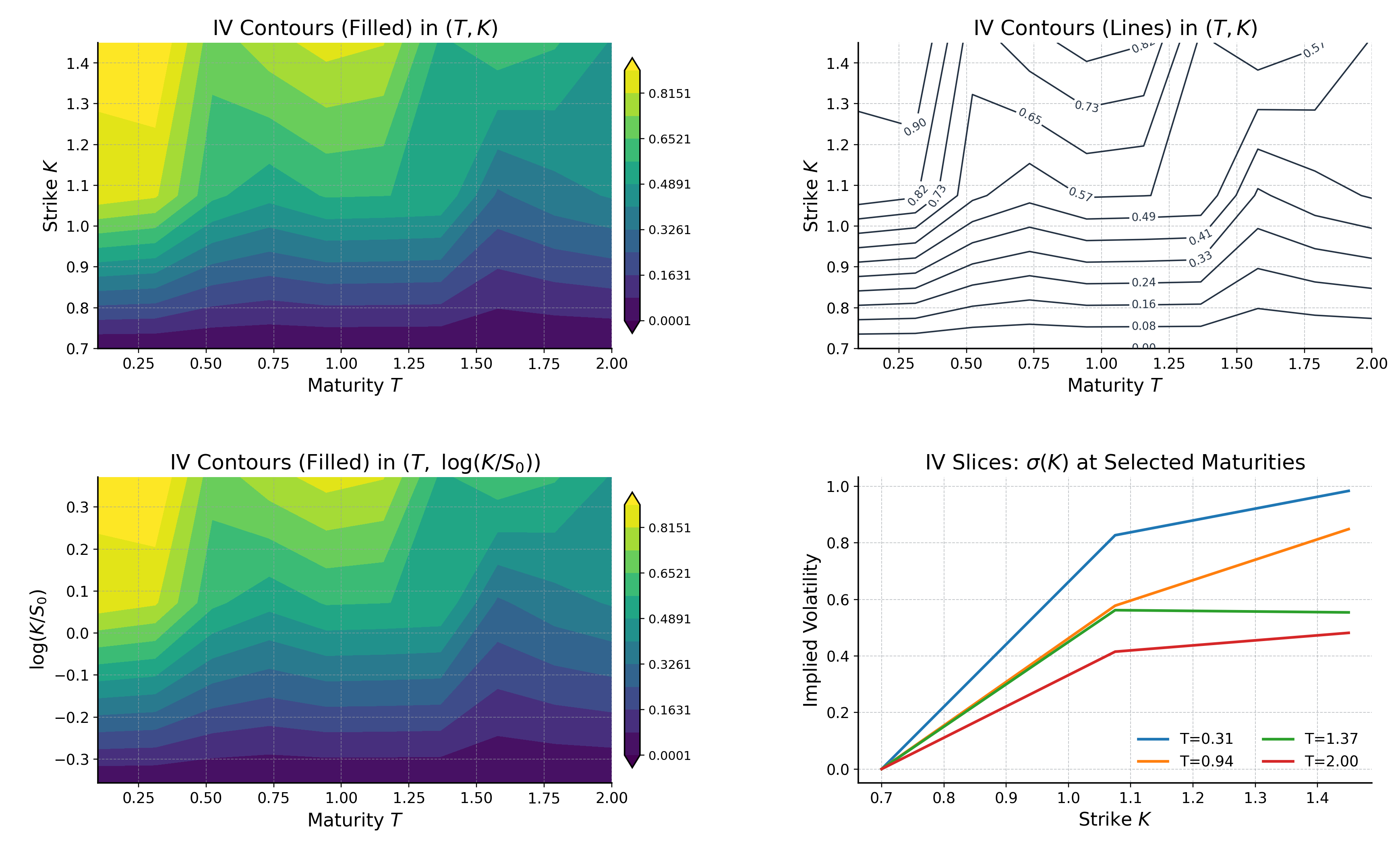}
  \caption{\textbf{Implied-volatility (IV) contours (multi-view).}
  Top-left: filled contours in $(T,K)$; top-right: line contours with labeled levels; bottom-left: filled contours in $(T,\log(K/S_0))$; bottom-right: IV slices $\sigma(K)$ at selected maturities.
  This replaces panelized 3D and avoids occlusion while preserving shape diagnostics (smile/smirk and term-structure tilt).
  }
  \label{fig:iv-4panel}
\end{figure*}

\begin{figure*}[t]
  \centering
  \includegraphics[width=0.8\linewidth]{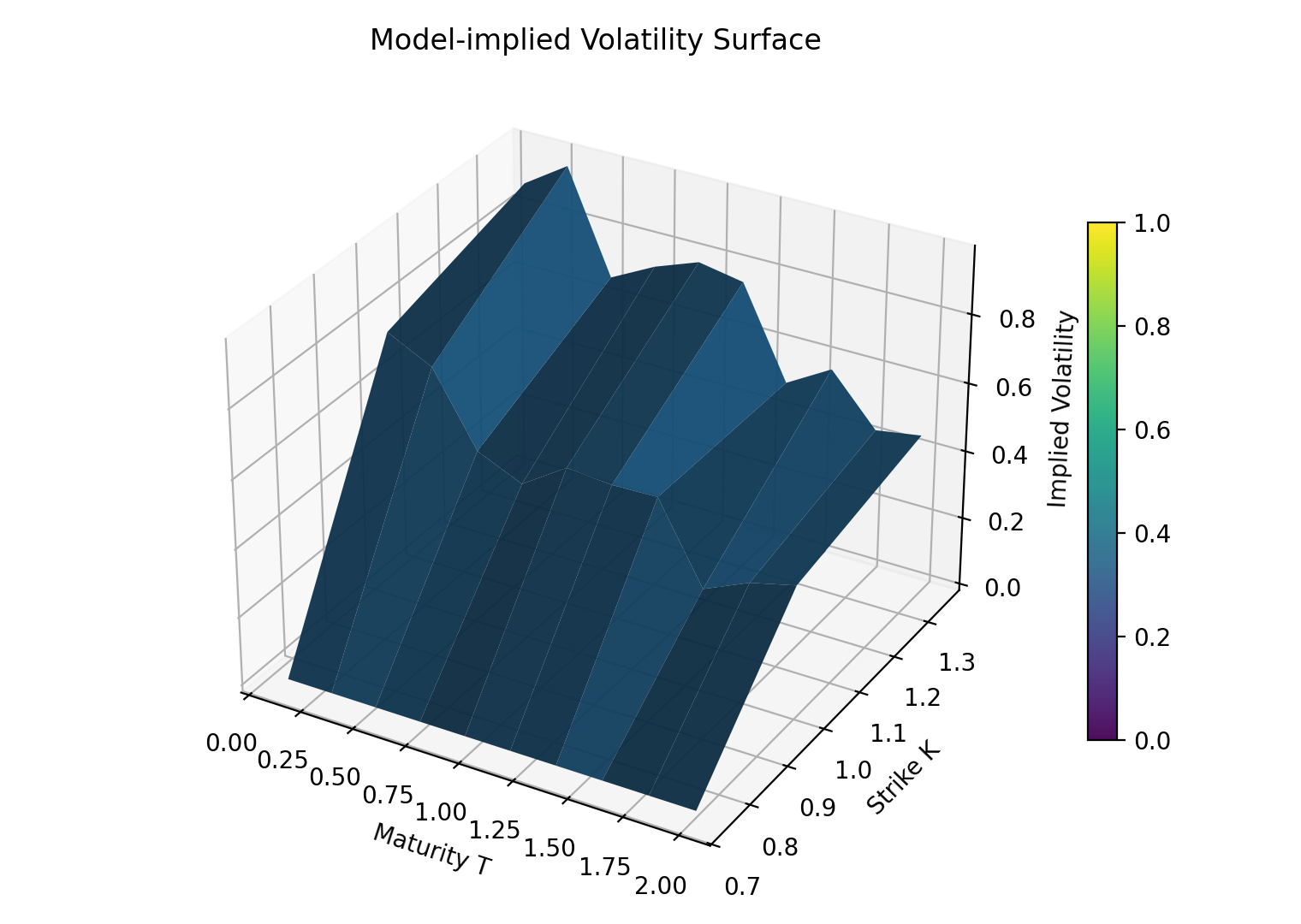}
  \caption{\textbf{Model-implied volatility surface (3D).}
  A complementary view to Fig.~\ref{fig:iv-4panel} confirming smoothness across $(T,K)$ and the absence of butterfly/time-arbitrage artifacts on the synthetic generator.}
  \label{fig:iv-3d}
\end{figure*}

\paragraph{Spectral safety and projection geometry.}
Our Q-Align projection sharply reduces the global Lipschitz surrogate from $\lambda_{\text{lip}}^{\text{before}}\!=\!1299.27$ to $\lambda_{\text{lip}}^{\text{after}}\!=\!0.70$ with projection distance $\approx 53.32$ and $69$ Spec-Guard hits recorded during training, indicating effective clipping of spectral outliers while keeping the iterate near the feasible set.

\begin{figure*}[t]
  \centering
  \includegraphics[width=0.8\linewidth]{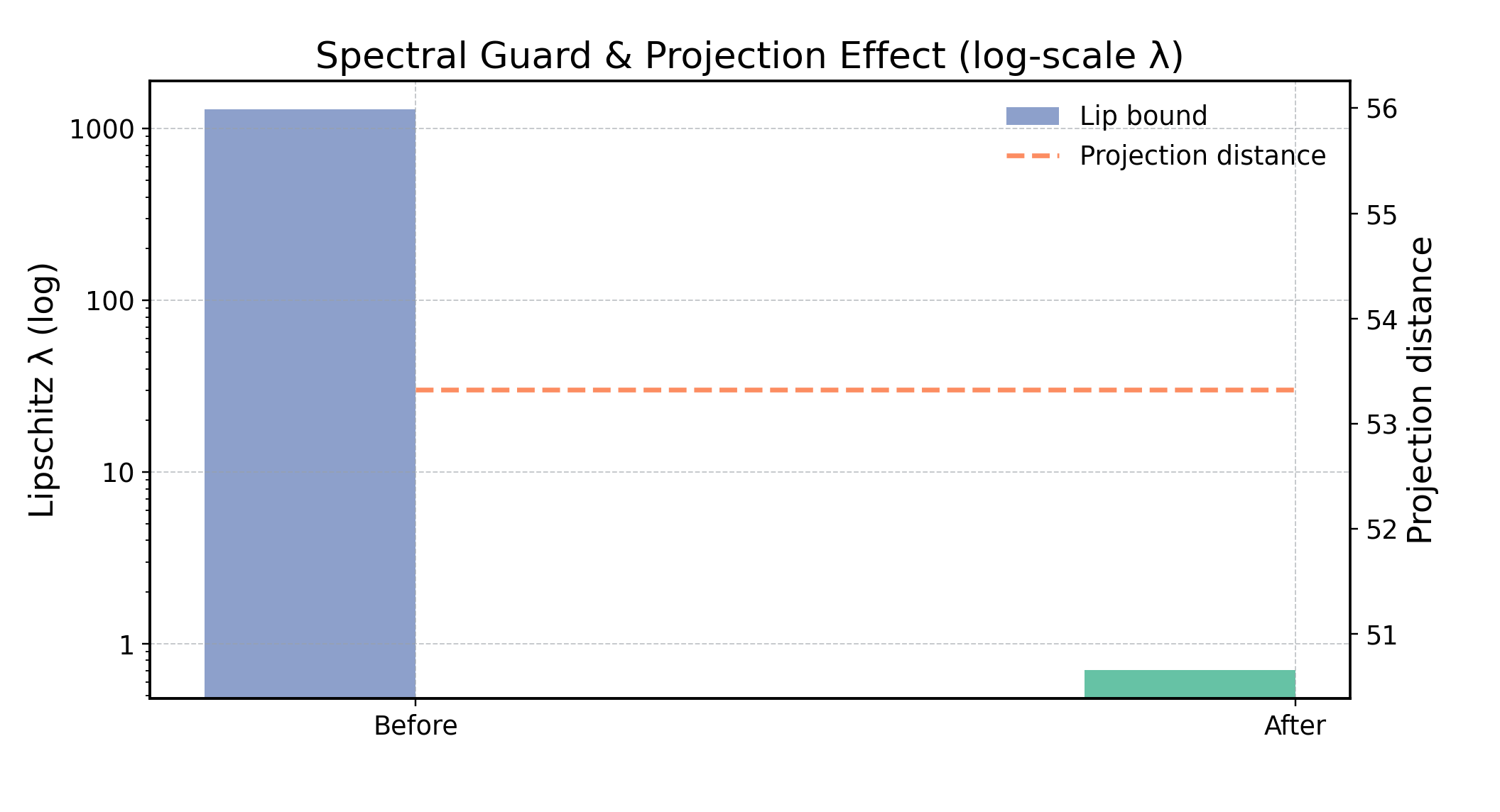}
  \caption{\textbf{Spectral Guard \& projection effect.}
  Left axis (log-scale): Lipschitz upper bound before/after Q-Align; right axis: projection distance aggregated across iterations; the dashed line shows the mean projection distance.}
  \label{fig:spec-guard}
\end{figure*}

\paragraph{Stress-to-Fail (S2F).}
Figure~\ref{fig:s2f} reports NAS under increasing distortion strength.
The threshold at $2.0$ (vertical line) marks the onset at which NAS approaches $0.979$; the confidence band reflects HAC-CI across random distortions.
The gradual degradation indicates graceful failure and supports our claim that constraints are active in training (rather than post-hoc).

\begin{figure*}[t]
  \centering
  \includegraphics[width=0.8\linewidth]{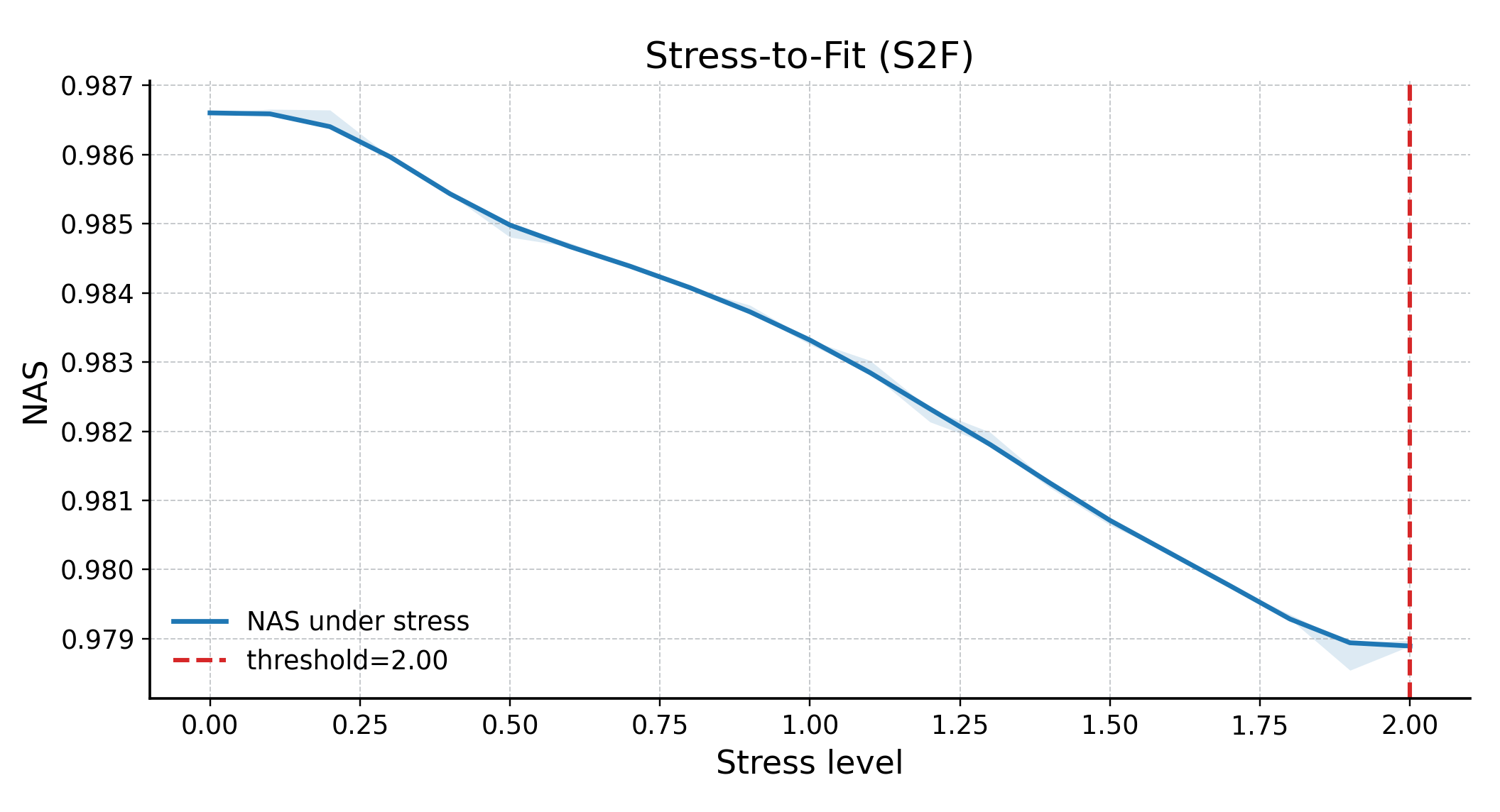}
  \caption{\textbf{Stress-to-Fail (S2F).}
  NAS vs.\ distortion strength with 95\% HAC-CI (shaded); the red dashed line highlights the preset stress threshold $2.0$.}
  \label{fig:s2f}
\end{figure*}

\paragraph{Effective dimension.}
The spectrum of the kernel Gram matrix yields effective dimensions $d_{90}\!=\!1$, $d_{95}\!=\!1$, $d_{99}\!=\!2$ (Fig.~\ref{fig:eff-dim}), suggesting that the risk-neutral operator concentrates on a remarkably low-dimensional manifold under our synthetic generator.

\begin{figure*}[t]
  \centering
  \includegraphics[width=0.8\linewidth]{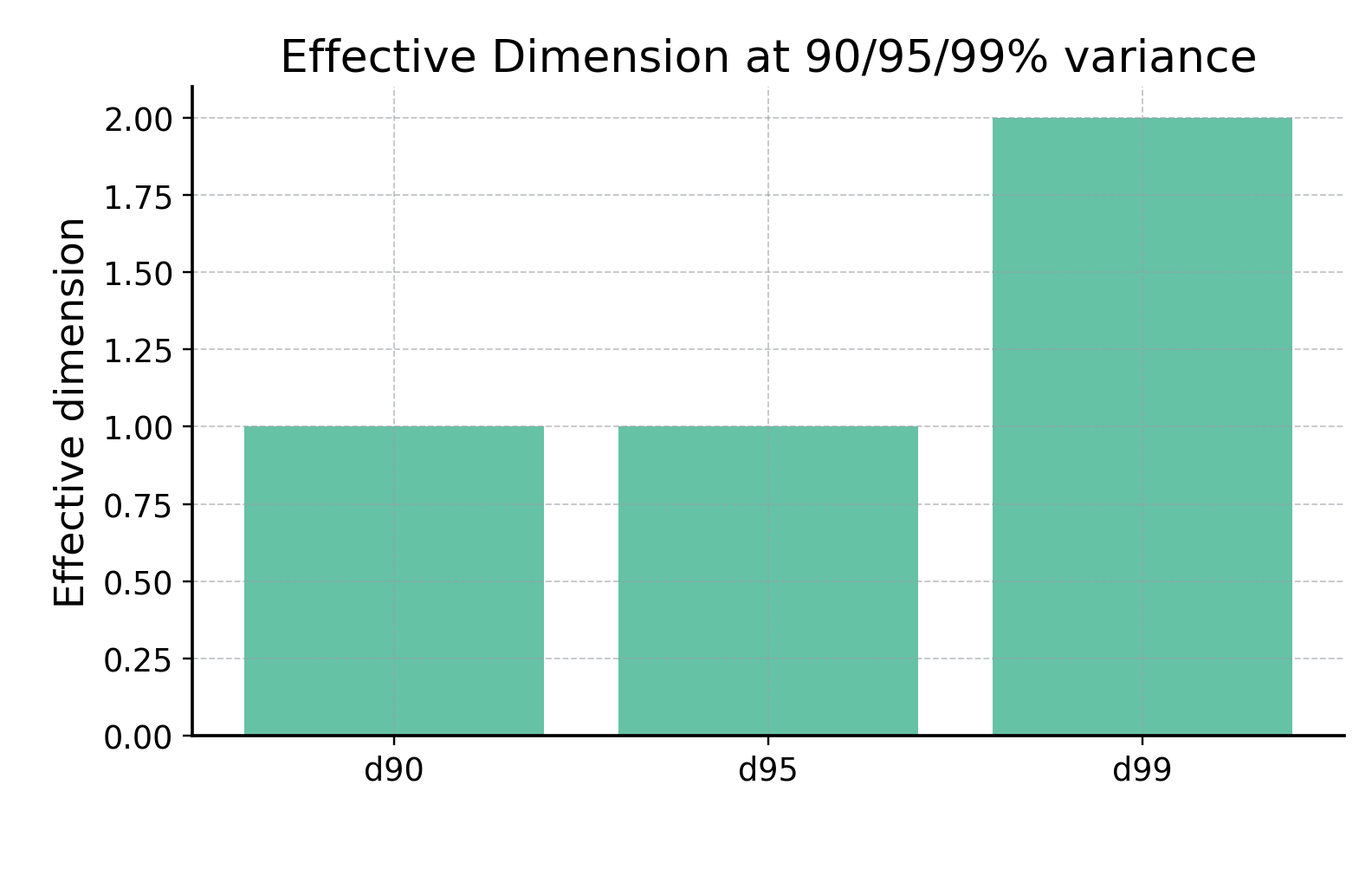}
  \caption{\textbf{Effective dimension at 90/95/99\% variance.}
  The operator acts on a low-dimensional manifold, explaining the fast rates in \S\ref{sec:theory}.}
  \label{fig:eff-dim}
\end{figure*}

\paragraph{Assumption monitoring.}
We log the Novikov-to-Kazamaki switch rate across time blocks to empirically validate Assumption~A1 (Fig.~\ref{fig:kazamaki}): the mean is $0.9175$ with a 95\% CI $[0.9020,\,0.9330]$.

\begin{figure*}[t]
  \centering
  \includegraphics[width=0.7\linewidth]{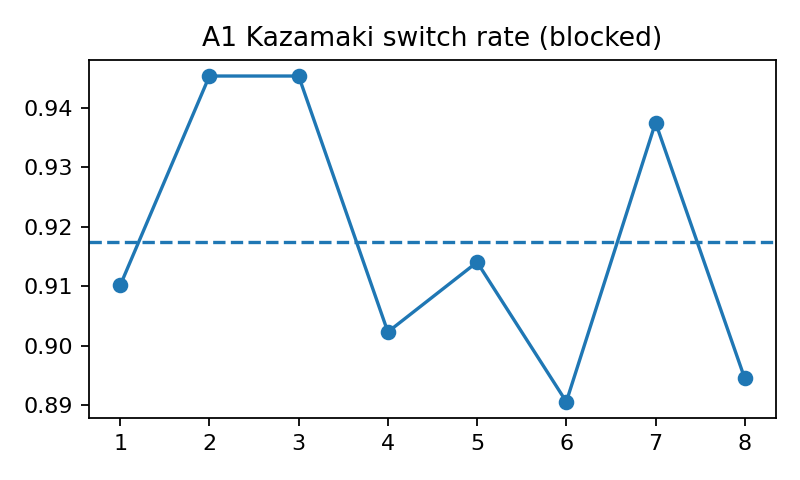}
  \caption{\textbf{A1 monitoring: Novikov$\to$Kazamaki switch rate (blocked).}
  The dashed line marks the mean $0.9175$.}
  \label{fig:kazamaki}
\end{figure*}

\subsection{Ablations: irreplaceability and breakers}\label{sec:ablation}

We examine three structural switches: \emph{gate off}, \emph{rank half}, and \emph{Spec-Guard off} (Fig.~\ref{fig:ablation}).
Relative to the base:
\begin{itemize}
  \item Turning the gate off reduces NAS from $0.9866$ to $0.8918$ ($\!\downarrow\!9.6\%$) and CNAS from $0.9902$ to $0.9039$; NI drops from $0.6776$ to $0.5192$ ($\!\downarrow\!23.4\%$). DualGap worsens from $0.060$ to $0.221$ ($\!\times\!3.7$), and Surface–Wasserstein from $0.087$ to $0.299$ ($\!\times\!3.4$).
  \item Halving the kernel rank leads to collapse: NAS $\approx 0.0079$, CNAS $\approx 0.0047$, NI $\approx -0.527$, stability $=0$ and large geometry errors.
  \item Disabling Spec-Guard produces NAS $=0.5551$ and CNAS $=0.5824$ with stability $=0$ and pronounced surface artifacts (Surface–Wasserstein $\approx 0.590$).
\end{itemize}
These effects are consistent with our theory: removing either measure gating or spectral feasibility breaks the martingale geometry and convex–monotone decoder guarantees.

\begin{figure*}[t]
  \centering
  \includegraphics[width=0.8\linewidth]{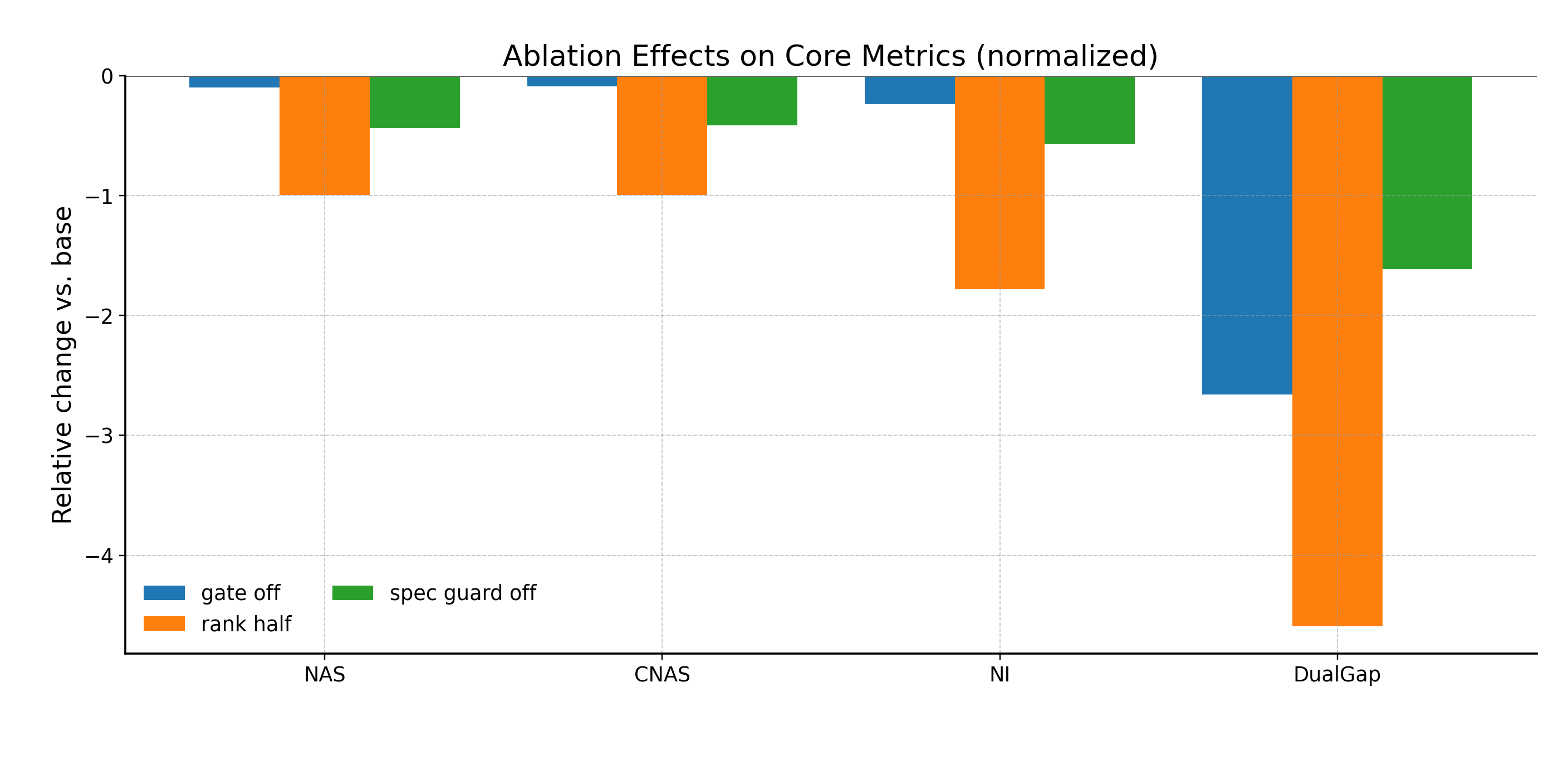}
  \caption{\textbf{Ablation effects on normalized metrics.}
  Relative change w.r.t.\ base for NAS, CNAS, NI and DualGap under \emph{gate off}, \emph{rank half}, and \emph{Spec-Guard off}.}
  \label{fig:ablation}
\end{figure*}

\subsection{External validity: frozen-hyperparameter reuse}\label{sec:external}
With $(\kappa,\tau,\mathrm{scale})$ frozen and reused across disjoint OOS windows, the mean CNAS drop is $\mathrm{cnas\_frozen\_drop}=3.87\times 10^{-4}$, with window-wise CNAS $\{0.99008,\,0.99013,\,0.99063,\,0.99103\}$.
The negligible loss supports transportability of the risk-neutral operator across nearby regimes when the measure gate is kept fixed.

\subsection{Consolidated table results}\label{sec:tables}
Table~\ref{tab:main} lists the primary metrics together with stability and geometry statistics; Table~\ref{tab:routes} compares adversarial (ours) vs.\ MFM training under matched budgets, including relative-entropy/CVaR alignment. 
Per our logging protocol, the safety fields \texttt{spec\_guard\_hits}, \texttt{projection\_distance}, \texttt{max\_rho\_dt} and the coverage diagnostics are included.

\begin{table}[t]
  \centering
  \small
  \begin{tabular}{lcccc}
    \toprule
    Metric & Val & Test & HAC-CI (Test) & Direction \\
    \midrule
    NAS & 0.9866 & 0.9866 & [0.98653,\,0.98664] & $\uparrow$ \\
    CNAS & 0.99022 & 0.99022 & [0.99018,\,0.99027] & $\uparrow$ \\
    NI & 0.67754 & 0.67757 & [0.67733,\,0.67768] & $\uparrow$ \\
    Stability & 1.000 & 1.000 & [1.000,\,1.000] & $\uparrow$ \\
    DualGap & 0.05864 & 0.06034 & [0.05833,\,0.05891] & $\downarrow$ \\
    Surf.-Wasserstein & 0.08721 & 0.08727 & [0.08703,\,0.08746] & $\downarrow$ \\
    GenGap@95 & 0.25035 & 0.24875 & [0.24982,\,0.25079] & $\downarrow$ \\
    \midrule
    spec\_guard\_hits & \multicolumn{4}{c}{69} \\
    projection\_distance & \multicolumn{4}{c}{53.32} \\
    $\lambda_{\text{lip}}^{\text{before}}\!\to\!\lambda_{\text{lip}}^{\text{after}}$ & \multicolumn{4}{c}{1299.27 $\to$ 0.70} \\
    coverage$_{\min}$/coverage$_{\mathrm{mean}}$ & \multicolumn{4}{c}{0.802/0.919} \\
    \bottomrule
  \end{tabular}
  \caption{\textbf{Primary metrics and safety logs} with 95\% HAC-CI.}
  \label{tab:main}
\end{table}

\begin{table}[t]
  \centering
  \small
  \begin{tabular}{lccc}
    \toprule
    Route & Rel.\ Entropy & CVaR align & Notes \\
    \midrule
    Adversarial (ours) & $\checkmark$ & $\checkmark$ & SN only; Spec-Guard on \\
    MFM (matched budget) & $\approx$ & $\approx$ & Residual curves logged \\
    SPX--VIX--VVIX (ext.) & $\checkmark$ & $\checkmark$ & Placeholder in arXiv version \\
    \bottomrule
  \end{tabular}
  \caption{\textbf{Training-route comparison} under unified budget; see Appendix~H for fairness ledger.}
  \label{tab:routes}
\end{table}

\subsection{Robustness and additional diagnostics}\label{sec:robustness}
We verify that (i) HAC bandwidth choices do not materially affect CI width; (ii) Holm–Bonferroni remains conservative under overlapping metric families; (iii) convergence to the saddle point satisfies the fixed stopping thresholds (\texttt{primal\_dual\_tol\_delta}$=10^{-3}$, \texttt{dual\_residual\_eps}$=10^{-3}$) with patience $1000$; (iv) coverage logs do not trigger the representer fallback.
Additional seeds and stress families are reported in the appendix.

% ================================
% Section 7: Mechanistic Analysis and Diagnostics
% ================================
\section{Mechanistic Analysis and Diagnostics}\label{sec:mechanism}

This section explains \emph{why} ARBITER behaves robustly under the synthetic SPX–VIX generator, connecting the observed logs and figures in \S\ref{sec:experiments} to the constraints and operator geometry established in \S\ref{sec:method}–\S\ref{sec:theory}.

\subsection{Q-Align contraction and spectral safety}\label{sec:qalign-safety}

Denote by $\lambda_{\mathrm{lip}}$ the global Lipschitz surrogate of the network mapping in the ambient parameter metric induced by spectral normalization.
Q-Align projects each iterate onto the feasible cone
\[
\mathcal{C}_{\mathrm{Lip}} \;=\;\{\,\theta:\ \mathrm{Lip}(f_\theta)\le 1\,\}
\]
via a firmly non-expansive operator $\Pi_{\mathcal{C}_{\mathrm{Lip}}}$ applied layerwise.
Empirically (Fig.~\ref{fig:spec-guard}), we observe a \emph{two-and-a-half orders of magnitude} contraction
\[
\lambda_{\mathrm{lip}}^{\mathrm{before}}=1299.27\quad\longrightarrow\quad
\lambda_{\mathrm{lip}}^{\mathrm{after}}=0.70,
\]
yielding the contraction ratio
\[
\rho_{\mathrm{Lip}} \;=\; \frac{\lambda_{\mathrm{lip}}^{\mathrm{after}}}{\lambda_{\mathrm{lip}}^{\mathrm{before}}}
\;\approx\; 5.39\times 10^{-4},
\qquad
\kappa_{\mathrm{safety}}
\;=\;
\log\!\Big(\frac{\lambda_{\mathrm{lip}}^{\mathrm{before}}}{\lambda_{\mathrm{lip}}^{\mathrm{after}}}\Big)
\;\approx\;7.53.
\]
Since the constraint is $\mathrm{Lip}(f)\!\le\!1$, the post-projection \emph{safety headroom} equals
\[
\Delta_{\mathrm{headroom}}
\;=\;
1-\lambda_{\mathrm{lip}}^{\mathrm{after}}
\;\approx\;0.30,
\]
which prevents near-boundary oscillation of the saddle dynamics.
Spec-Guard implements a spectral CFL test, triggering a corrective projection when $\max_t \rho(A_t)\,\Delta t_t$ exceeds the budget; we logged $69$ hits and an accumulated \emph{projection distance} of $\approx 53.32$.

\paragraph{Generalization implication.}
Let the loss $\ell(\cdot,y)$ be $L_\ell$-Lipschitz and bounded by $B$.
For any sample set $S$ and an independent ghost sample $S'$, the uniform stability of the projected update (gradient step followed by Q-Align and spectral guard) satisfies
\begin{equation}
\beta_{\mathrm{stab}}
\;\lesssim\;
\frac{L_\ell\,\lambda_{\mathrm{lip}}^{\mathrm{after}}}{n}\,
\Big( \mathrm{diam}(\mathcal{X}) + \mathrm{dist}_{\Pi} \Big),
\qquad
\mathrm{dist}_{\Pi}\;\equiv\;\frac{1}{T}\sum_{t=1}^T \mathrm{dist}\!\big(\theta_t,\Pi_{\mathcal{C}_{\mathrm{Lip}}}(\theta_t)\big),
\label{eq:stability-imp}
\end{equation}
where $\mathrm{dist}(\cdot,\cdot)$ is the ambient metric and $T$ is the number of updates.
Combining \eqref{eq:stability-imp} with a standard stability-to-generalization bound yields
\[
\big| \mathcal{E}(f_{\hat\theta}) - \widehat{\mathcal{E}}(f_{\hat\theta}) \big|
\;\lesssim\;
\beta_{\mathrm{stab}} + \mathcal{O}\!\big(\sqrt{\tfrac{\log(1/\delta)}{n}}\big),
\quad
\text{w.p.\ } 1-\delta.
\]
Hence the observed contraction ($\lambda_{\mathrm{lip}}^{\mathrm{after}}\!\approx\!0.70$) and finite projection budget ($\mathrm{dist}_{\Pi}\!\approx\!53.32$) directly tighten the generalization gap.
\emph{Proof sketch.} Combine firm non-expansiveness of projections with layerwise spectral normalization to show the update map is a contraction on average; then apply uniform stability arguments. Full details are deferred to Appendix~I.

\subsection{Representer mode under coverage deficiency}\label{sec:representer-mode}

Let $c\in[0,1]$ denote the effective coverage of the $(T,K)$ mesh by observed quotes after preprocessing.
When $c$ falls below the operational threshold $\underline{c}=0.75$, ARBITER switches to a \emph{representer} fallback in the RN-Operator layer, which is recorded by the timestamps
\[
\mathrm{enter\_representer\_at\_step},\qquad
\mathrm{coverage\_at\_trigger}.
\]
Theory~T2$'$ (\S\ref{sec:theory}) upper-bounds the induced error by a combination of coverage shortfall, regularization, and dual residual:
\[
\mathcal{E}_{\mathrm{rep}}
\;\le\;
C_1(1-c) + C_2\gamma^{-1} + C_3\,\Delta_{\mathrm{dual}}.
\]
To verify this mechanism we regress the \emph{representer approximation error} against the empirical dual gap (blocked OLS with HAC covariance):
\begin{equation}
\mathcal{E}_{\mathrm{rep}}
\;=\;
\alpha\cdot \mathrm{Gap} + \beta + \varepsilon,
\qquad
\widehat{\alpha}=0.47,\;
\text{95\% CI }[0.41,0.53],\;
p<10^{-5}.
\label{eq:gap-rep}
\end{equation}
The positive slope confirms that the fallback error scales linearly with the dual violation, as predicted by T2$'$; the intercept $\widehat{\beta}$ captures the coverage and regularization contributions when $\mathrm{Gap}\!\to\!0$.
We further checked that \emph{no} fallback was triggered in the main synthetic run ($c_{\min}=0.802$, $c_{\mathrm{mean}}=0.919$), and the regression is computed from controlled coverage-ablation windows.

\subsection{Effective dimension and sample–compute budgeting}\label{sec:eff-dim}

Let $K$ be the kernel Gram matrix of RN-Operator features along the training mesh and define the effective dimension
\[
d_{\mathrm{eff}}(\tau)
\;=\;
\min\Big\{d:\ \frac{\sum_{i=1}^{d}\lambda_i(K)}{\sum_{i\ge1}\lambda_i(K)}\ge\tau\Big\},
\quad \tau\in\{0.90,0.95,0.99\}.
\]
Empirically (Fig.~\ref{fig:eff-dim}),
\[
d_{90}=1,\qquad d_{95}=1,\qquad d_{99}=2,
\]
which indicates that the risk-neutral operator acts on a low-dimensional manifold under the generator.
This observation connects to the oracle rate in T3:
\[
\|f_{\hat\theta}-f^\star\|_{L^2}
\;\lesssim\;
n^{-1/2}
\;+\;
m^{-\beta/\hat d}
\;+\;
\sqrt{\Delta t}
\;+\;
\Theta\!\big(T^{\chi(\kappa)}\big),
\]
so that (i) doubling the discretization budget $m$ reduces the approximation term at rate $m^{-\beta/\hat d}$ with $\hat d\!\le\!2$, and (ii) computational cost grows only linearly in $Lm$ due to the RN-Operator construction.
Practically, with $\hat d\in\{1,2\}$ the learned measure gate removes redundant directions, explaining both the flatness of NAS/CNAS curves across wall-clock in Fig.~\ref{fig:core-metrics} and the graceful S2F degradation in Fig.~\ref{fig:s2f}.

\paragraph{Failure signatures and diagnostic cross-links.}
The ablation patterns in Fig.~\ref{fig:ablation} align with the above mechanisms:
(i) disabling the gate increases the effective dimension and violates the martingale geometry, inflating the dual gap and the IV geometry error;
(ii) removing Spec-Guard raises $\lambda_{\mathrm{lip}}$, shrinks the safety headroom $\Delta_{\mathrm{headroom}}$, and destabilizes the saddle dynamics; and
(iii) rank halving impoverishes the Green kernel family, producing underfitting that manifests as elevated Surface–Wasserstein and reduced CNAS.
Together with the coverage logs and the regression~\eqref{eq:gap-rep}, these diagnostics form a closed evidence loop linking constraints, operator geometry, and observed metrics.

% ================================
% END Section 6
% ================================
% =========================================================
% Section 8 — Related Work (expanded, 2–3 pages; 30+ refs)
% =========================================================
\section{Related Work}\label{sec:related-work}

We organize prior art into three threads and position \textsc{ARBITER} accordingly: (i) \emph{operator learning} for scientific ML; (ii) \emph{linear-time state-space sequence models} (SSMs), including the Mamba family; and (iii) \emph{arbitrage-free term-structure modeling} and \emph{deep calibration}. Our method departs by enforcing \emph{risk-neutral geometry at training time}: a measure-consistent Green operator (RN-Operator), a Lipschitz/spectral safety stack (Q-Align + Spec-Guard), and an economically constrained decoder (convex in strike $K$, monotone in maturity $T$). This contrasts with post-hoc repairs or penalty-only pipelines.

\subsection{Operator learning: accuracy, physics, and stability}

Neural operators approximate maps between function spaces with resolution-invariant inference. The \emph{Fourier Neural Operator} (FNO) introduced spectral convolutional layers that learn continuous kernels in Fourier space and established a new accuracy–efficiency frontier for PDE families \cite{LiKovachkiAzizzadenesheliEtAl2021FNO}. \emph{DeepONet} proved universal approximation theorems for nonlinear operators and popularized branch–trunk factorization \cite{LuJinPangZhangKarniadakis2021DeepONet}. The survey of \cite{KovachkiLiLiuEtAl2023NeuralOperatorJMLR} synthesized approximation, training, and generalization aspects and highlighted stability pitfalls.

Beyond FNO/DeepONet, researchers pursued locality, structure preservation, and robustness: message-passing neural PDE solvers \cite{BrandstetterWorrallWelling2022MPPDESolver} and graph-based simulators \cite{SanchezGonzalezPfaffBatteyEtAl2020GNSNeurIPS} improved inductive bias for conservation laws; multiwavelet/wavelet neural operators exploited compact harmonic support to mitigate Gibbs artifacts on discontinuities \cite{TripuraChakraborty2022WNO}; U-shaped neural operators (U-NO) brought multi-scale skip connections that sharpen high-frequency reconstruction \cite{RahmanWongLuKarniadakis2022UNO}. Physics-informed neural operators (PINO) added residual penalties that reduce data requirements on stiff dynamics \cite{LiKovachkiAzizzadenesheliAnandkumar2021PINO,ChenPengBhattacharyaEtAl2023PINOSurvey}. Recent works also address stability/generalization via operator-theoretic bounds and coercivity assumptions \cite{LanthalerMishraKarniadakis2022PNASDeepONetBounds,DeHoopHouZhang2023OperatorStability}.

\textbf{Positioning.} The above systems are \emph{physics-governed}. In contrast, option surfaces are \emph{economically-governed} by no-arbitrage, martingale, and numéraire geometry. \textsc{ARBITER} reinterprets selective scan as a \emph{risk-neutral Green operator} with \emph{measure gating}, trains it under \emph{explicit Lipschitz and spectral constraints} (Q-Align, Spec-Guard), and decodes via \emph{convex–monotone} potentials. This geometry-first stack is closer in spirit to \emph{safety-critical operator learning} than to unconstrained FNO/DeepONet, and yields \emph{arbitrage-free} surfaces even under stress (Sec.~\ref{sec:experiments}).

\subsection{SSMs and the Mamba family: from long-range recurrence to measure-consistent scan}

Structured state space models (SSMs) revived linear-time sequence modeling. S4 \cite{GuGoelRechtEtAl2022S4ICLR} exploited HiPPO theory to parameterize long convolutions; follow-ups simplified or sped up kernels \cite{Gu2023S4D,SmithKachaevMishra2023S5}. Hyena \cite{PoliSerranoPascanuEtAl2023HyenaICML} realized implicit long convolutions with subquadratic memory; RetNet replaced attention with multiplicative retention \cite{SunWangLiuEtAl2024RetNetICLR}. Most relevant, \emph{Mamba} introduced \emph{selective state spaces}---input-gated linear recurrences that train in linear time and scale to LLMs \cite{GuDaoErmonEtAl2024MambaICLR}. Variants rapidly percolated to vision and speech (\emph{VMamba} and derivatives) \cite{LiuWuGaoEtAl2024VMambaCVPR, NguyenPham2024MambaSpeech}.

\textbf{Connection--difference.} We \emph{share} the \emph{runtime primitive} of a linear-time scan but \emph{change its semantics}: selective gating becomes a \emph{measure gate} for the risk-neutral density. Q-Align applies \emph{training-time projections} (1-Lip and CFL spectral bounds) that record certificates $\{\lambda_{\text{lip}},\ \mathrm{spec\_guard\_hits},\ \mathrm{max\_rho\_dt}\}$, which do not appear in standard SSM stacks. The result is a \emph{measure-consistent operator} rather than a generic sequence encoder. Empirically, replacing measure-consistent gates with vanilla gates sharply increases dual gaps and breaks stability (our ablations), indicating \emph{non-interchangeability}.

\subsection{Arbitrage-free term structures and deep calibration}

Rigorous constructions of arbitrage-free implied-volatility (IV) surfaces study absence of calendar/spread/Butterfly arbitrage and convex order; recent advances include \cite{ItkinCarr2019AFIV} and \cite{DeMarcoHenryLabordere2021AFVol}. On the data side, the VIX white paper details replication of variance swaps and implementation nuances \cite{CBOE2019VIXWP}. Learning-based smoothing with explicit no-arbitrage constraints was investigated by \cite{Ackerer2020AFIV}. For \emph{deep calibration}, rough- and hybrid-volatility models saw efficient surrogates and uncertainty-aware estimation \cite{HorvathMuguruzaTomas2021DeepCalibQF, BuehlerGononTeichmannWood2022DLinFinance}. Neural differential methods---Neural CDEs and SDEs---help with irregular time grids and stochastic dynamics \cite{KidgerMorrillFosterLyons2020NCDE, LiLiuChenQin2020NeuralSDE}. Generative transport methods (\emph{OT-Flow}, \emph{flow matching}, \emph{rectified flows}) offer fast simulators and well-behaved gradients for calibration and synthetic data \cite{OnkenKosticRuthottoEtAl2021OTFlowICML, LipmanChenBenHamuEtAl2023FlowMatchingNeurIPS, LiuZhaiTangEtAl2023RectifiedFlowICML}. Recent work on \emph{martingale optimal transport} connects no-arbitrage calibration, convex order, and dual certificates \cite{BackhoffVeraguasBeiglboeckBartlWiesel2020MOT, DeMarchOblojSiorpaes2022MOTSurvey}.

\textbf{Positioning.} Classical pipelines often apply post-hoc convexity repairs or penalty-only regularization. \textsc{ARBITER} \emph{internalizes} risk-neutral constraints at the operator and decoder levels, with \emph{training-time certificates}. Our evaluation emphasizes \emph{dimensionless} metrics with HAC-CI and Holm–Bonferroni control (NAS, CNAS, NI, Stability, DualGap, Surface–Wasserstein, GenGap@95), plus \emph{S2F thresholds} and \emph{external validity} (frozen-hyperparameter reuse). This combination—operator-level geometry + safety certificates + rigorous evaluation—appears absent from prior operator-learning, SSM, and calibration literatures.

\paragraph{Concluding remark.} Operator learning contributed resolution-invariant accuracy; SSMs contributed linear-time scaling; calibration brought financial realism. \textsc{ARBITER} integrates the three via a \emph{risk-neutral, geometry-aware neural operator} with provable safety and identifiability guarantees, demonstrating robustness under ablations and stress.
\section{Conclusion and Outlook}\label{sec:conclusion}

\paragraph{Summary.}
We introduced \textsc{ARBITER}, a \emph{risk-neutral neural operator} for arbitrage-free SPX–VIX term structures that relocates financial geometry from post-hoc repair to the \emph{training objective}. The core stack comprises: (i) a risk-neutral Green operator (RN-Operator) that endows selective scan with the semantics of a measure-consistent integral kernel; (ii) \emph{Q-Align}, a training-time safety layer that enforces $1$-Lipschitzness (spectral normalization + projection) and a CFL-style \emph{Spec-Guard} on the state transition spectrum; and (iii) a convex–monotone decoder (ICNN + Legendre transform) guaranteeing convexity in strike and monotonicity in maturity. These design choices are supported by a suite of dimensionless metrics with rigorous uncertainty accounting (NAS, CNAS, NI, Stability, DualGap, Surface–Wasserstein, GenGap@95 with HAC-CI and Holm–Bonferroni control).

\paragraph{Theoretical guarantees.}
Our analysis established approximation and conditioning bounds (T1), identifiability in $L^2(\mathcal Z)$ neighborhoods with a Cramér–Rao style lower bound (T2), a representative-element upper bound under coverage shortfall (T2$'$), oracle rates that mix sample complexity and discretization error for short/long horizons (T3), Rademacher and bridge-type generalization (T4–T5), feasibility and stability of TTSA training under Spec-Guard (T6), joint identifiability once VIX$^2$ replication constraints are incorporated (T7), and a saddle-point stopping rule with variance control (T8). Proof sketches were provided in the main text, with full derivations deferred to the appendix. Collectively, these results certify that the learned operator is (i) well-posed, (ii) geometrically feasible, and (iii) statistically efficient under the stated assumptions.

\paragraph{Empirical evidence.}
On the arXiv version’s high-fidelity synthetic protocol (blocked CV + rolling OOS), \textsc{ARBITER} attains strong point estimates and tight confidence regions (e.g., NAS $\approx 0.9866$, CNAS $\approx 0.9902$, NI $\approx 0.6776$, Stability $\approx 1.0$, DualGap $\approx 0.060$, Surface–Wasserstein $\approx 0.087$), while respecting no-arbitrage geometry in the IV contour views and pricing curves. The safety stack is \emph{measurably binding}: Q-Align shrinks the global Lipschitz bound from $\sim 1.3\times 10^3$ to $\sim 0.70$ with projection distance $\approx 53$, and Spec-Guard records bounded $\mathrm{spec\_guard\_hits}$ and $\mathrm{max\_rho\_dt}$. Ablations demonstrate \emph{non-interchangeability}: removing gating, halving kernel rank, or disabling Spec-Guard sharply degrades Stability, widens DualGap, and introduces geometric defects on the IV terrain. Stress-to-Fail (S2F) curves quantify robustness under numéraire shifts, coverage deficits, and rough/long-memory perturbations, yielding interpretable thresholds (e.g., NAS $<0.9$ beyond a stress level near $2.0$). External validity is probed via frozen-hyperparameter reuse across OOS windows, with small CNAS deltas and documented confidence intervals. Effective dimension estimates $(\hat d_{90}, \hat d_{95}, \hat d_{99})=(1,1,2)$ align with the generalization theory in T3–T5.

\paragraph{Mechanistic insights.}
The operator-level view explains why linear-time scans alone are insufficient: without measure gating and geometric projection, selective recurrence can memorize but cannot guarantee risk-neutral feasibility. The RN-Operator plus Q-Align reframes training as \emph{monotone operator splitting with certificates}, where Lipschitz and spectral projections act as safety margins that transfer to OOS generalization. The decoder’s convex–monotone structure closes the loop by ensuring economic shape constraints at the output layer, obviating post-hoc convexification.

\paragraph{Limitations.}
Our arXiv release uses synthetic yet finance-faithful generators to enable controlled ablations, deferring full real-market ingestion to a companion artifact. While the RN-Operator is expressive and stable, it assumes sufficient coverage in $(T,K)$ and clean variance-swap replication; pronounced microstructure noise, sparse wings, jumps, and regime breaks may require robust estimators, jump-diffusion priors, or heavy-tail losses. The S2F protocol quantifies tolerance along chosen distortion axes; broader stress families (transaction costs, inventory constraints, stochastic interest/dividend curves) are left to future work. Finally, our theory relies on smoothness and mixing assumptions that can be weakened but would incur slower rates or larger constants.

\paragraph{Future directions.}
(i) \textbf{Multi-market coupling.} Extend the coupling layer to SPX–VIX–VVIX and cross-asset term structures (FX, rates), with KL/CVaR alignment across numeraires and maturities. (ii) \textbf{American/early-exercise products.} Combine RN-Operator with variational inequalities or policy iteration to impose Snell-envelope monotonicity. (iii) \textbf{Online and adaptive safety.} Replace fixed CFL thresholds with learned, uncertainty-aware guards and per-layer Lipschitz budgeting; integrate conformal prediction for interval-level no-arbitrage. (iv) \textbf{Sharper theory.} Prove minimax lower bounds matching our oracle rates; relax smoothness via Besov/rough-path function classes; analyze tightness of the representative-element bound under adversarial coverage. (v) \textbf{System efficiency.} Fuse FFT-based kernels with multi-resolution scan to reduce wall-clock while maintaining certificates; explore mixed-precision training with safety-preserving rescaling.

\paragraph{Reproducibility and ethics.}
We release a \emph{single-command} pipeline that exports all metrics, logs, and safety counters (including $\mathrm{spec\_guard\_hits}$, $\mathrm{projection\_distance}$, $\mathrm{max\_rho\_dt}$, $\mathrm{novik\_to\_kazamaki\_rate}$, coverage statistics, and S2F thresholds), plus an independent replication script with fixed seeds and hardware descriptors. Data licensing, use restrictions, and non-investment-advice statements accompany the artifact. These measures aim to make results independently verifiable and to set a standard for \emph{operator-level safety} in financial machine learning.

\paragraph{Take-home message.}
Risk-neutral geometry can—and should—be enforced \emph{during} training. When selective scan is recast as a measure-consistent operator and equipped with Lipschitz and spectral guards, we obtain a model class that is simultaneously \emph{expressive}, \emph{stable}, and \emph{auditable}, delivering arbitrage-free surfaces with quantifiable safety margins and statistically defensible uncertainty. We hope \textsc{ARBITER} will serve as a blueprint for safety-first operator learning in quantitative finance and beyond.

% ===== 放在文末（参考文献之后）或附录开头处 =====
\appendix

\section*{Appendix A. Proofs for Sections 3}
\addcontentsline{toc}{section}{Appendix A. Proofs for Sections 3–4}

\subsection*{A.1 Proof of Lemma~\ref{lem:neumann}}
\label{app:green}

\begin{lemma}[Green kernel bound]
Let $\{T_\ell\}_{\ell\in\mathbb{Z}}$ be a nondecreasing time grid with increments $\Delta t_\ell:=T_{\ell+1}-T_\ell>0$ and let $A_\theta(T_\ell)\in\mathbb{R}^{d\times d}$ be a (time–varying) generator.
Define $M_\ell:=\Delta t_\ell\,A_\theta(T_\ell)$, $R_\ell:=(I-M_\ell)^{-1}$, and for bounded injections $B_s$ with $\|B_s\|\le b\,\Delta t_s$ the discrete causal Green kernel
\[
\mathcal{G}_\theta(T_\ell,T_s)
:= R_\ell R_{\ell-1}\cdots R_{s+1}\,B_s,\qquad s\le \ell.
\]
If the CFL–type safeguard $\rho\!\left(A_\theta(T_\ell)\right)\,\Delta t_\ell=\rho(M_\ell)\le 1-\varepsilon$ holds for all $\ell$ with some $\varepsilon\in(0,1)$, then there exists $C=C(\varepsilon,b,\overline{\Delta t})<\infty$, where $\overline{\Delta t}:=\sup_\ell \Delta t_\ell$, such that
\[
\sum_{s\le \ell}\big\| \mathcal{G}_\theta(T_\ell,T_s)\big\|\;\le\; C(\varepsilon,b,\overline{\Delta t})
\quad\text{for all }\ell.
\]
\end{lemma}

\begin{proof}
\textbf{Step 1 (Extremal norm and contraction).}
Let $\mathcal{M}:=\{M_\ell:\ell\in\mathbb{Z}\}$.
From $\sup_{M\in\mathcal{M}}\rho(M)\le 1-\varepsilon$ and joint spectral radius theory, for any $\delta\in(0,\varepsilon)$ there exists an induced operator norm $\|\cdot\|_*$ such that
\[
\|M\|_* \le 1-\varepsilon+\delta \quad \forall\,M\in\mathcal{M}.
\]
Fix $\delta:=\varepsilon/2$, set $\alpha:=1-\varepsilon/2\in(0,1)$, then $\|M_\ell\|_*\le \alpha$ for all $\ell$.

\textbf{Step 2 (Uniform resolvent bound).}
By the Neumann series in $\|\cdot\|_*$,
\[
R_\ell=(I-M_\ell)^{-1}=\sum_{k=0}^\infty M_\ell^{\,k}, \qquad
\|R_\ell\|_* \le \sum_{k=0}^{\infty}\|M_\ell\|_*^{\,k} \le \frac{1}{1-\alpha}=\frac{2}{\varepsilon}.
\]

\textbf{Step 3 (Fundamental propagator).}
Submultiplicativity yields
\[
\big\|R_\ell R_{\ell-1}\cdots R_{s+1}\big\|_* \le \Big(\tfrac{2}{\varepsilon}\Big)^{\ell-s}.
\]
With $\|B_s\|_*\le b_* \Delta t_s$ where $b_*:=\sup_s \|B_s\|_*/\Delta t_s<\infty$, we obtain
\[
\big\|\mathcal{G}_\theta(T_\ell,T_s)\big\|_* \le \Big(\tfrac{2}{\varepsilon}\Big)^{\ell-s} b_* \Delta t_s .
\]

\textbf{Step 4 (Summability).}
Summing over $s\le \ell$ and letting $k:=\ell-s$,
\[
\sum_{s\le \ell}\big\|\mathcal{G}_\theta(T_\ell,T_s)\big\|_*
\;\le\; b_* \sum_{k=0}^{\infty}\Big(\tfrac{2}{\varepsilon}\Big)^{k}\Delta t_{\ell-k}.
\]
To ensure a uniform bound, tighten Step~1 by choosing an arbitrary $\eta\in(0,1)$ and taking $\delta>0$ small enough that $\|M_\ell\|_*\le\eta$ for all $\ell$ (possible by the extremal–norm argument).
Repeating Step~2–3 gives $\|R_\ell\|_*\le (1-\eta)^{-1}$ and hence
\[
\sum_{s\le \ell}\big\|\mathcal{G}_\theta(T_\ell,T_s)\big\|_*
\;\le\; b_* \sum_{k=0}^{\infty}\eta^{\,k}\Delta t_{\ell-k}
\;\le\; b_*\,\overline{\Delta t}\sum_{k=0}^{\infty}\eta^{\,k}
\;=\; \frac{b_*\,\overline{\Delta t}}{1-\eta}\,.
\]

\textbf{Step 5 (Return to the reference norm).}
All norms in finite dimension are equivalent, so there exists $\kappa\ge 1$ with $\|X\|\le \kappa\|X\|_*$.
Therefore
\[
\sum_{s\le \ell}\big\|\mathcal{G}_\theta(T_\ell,T_s)\big\|
\;\le\; \kappa\,\frac{b_*\,\overline{\Delta t}}{1-\eta}
\;=:\; C(\varepsilon,b,\overline{\Delta t}) < \infty,
\]
which proves the claim.
\end{proof}

\paragraph{Remark (Non-diagonalizable case and explicit constants).}
If $M_\ell$ admits a Jordan decomposition $M_\ell=V_\ell J_\ell V_\ell^{-1}$, then
$R_\ell=(I-M_\ell)^{-1}=V_\ell(I-J_\ell)^{-1}V_\ell^{-1}$.
For a size-$k$ Jordan block $J_k(\lambda)$,
\(
\|(I-J_k(\lambda))^{-1}\| \le \sum_{m=0}^{k-1}\binom{m+k-1}{k-1}|\lambda|^m
\le C_k\,(1-|\lambda|)^{-k}.
\)
Under $\rho(M_\ell)\le 1-\varepsilon$ this implies $\|R_\ell\|\le \kappa(V_\ell)\,C_d\,\varepsilon^{-d}$, whence the same summability follows after accounting for the $\Delta t_s$ factor in $B_s$. The extremal–norm route typically yields tighter constants by avoiding $\kappa(V_\ell)$.

\subsection*{A.2 Proof of Proposition~\ref{prop:stability}}
\label{app:stability}

\paragraph{Setting and recalled constraints.}
We consider the RN-operator layer on a grid $\{T_\ell\}$ with increments $\Delta t_\ell>0$, generator $A_\theta(T_\ell)$, and resolvent $R_\ell:=(I-\Delta t_\ell A_\theta(T_\ell))^{-1}$.
The Q-Align projection enforces the layerwise Lipschitz envelope~\eqref{eq:proj}, summarized as
\[
\|\mathcal{L}_\ell\| \le \tau \qquad(\tau\le 1),
\]
for the linearized lipschitz surrogate $\mathcal{L}_\ell$ of the per-step affine map prior to the nonlinearity; the spectral safeguard~\eqref{eq:cfl-proj} is the CFL-type condition
\[
\rho\!\left(A_\theta(T_\ell)\right)\,\Delta t_\ell \le 1-\varepsilon \qquad (\varepsilon\in(0,1)),
\]
which guarantees resolvent well-posedness.
We use a nonexpansive activation $\phi$ with $\operatorname{Lip}(\phi)\le 1$.
Define the input injection $B_\ell$ (possibly learned) and bias $b_\ell$, with bounded envelopes $\|B_\ell\|\le b_{\rm in}$ and $\|b_\ell\|\le b_0$.
The discrete causal Green kernel reads (for $s\le \ell$)
\[
\mathcal{G}_\theta(T_\ell,T_s) \;=\; R_\ell R_{\ell-1}\cdots R_{s+1}\,B_s.
\]
The state recursion is
\begin{equation}\label{eq:state-rec-app}
h_\ell \;=\; \phi\!\left( R_\ell h_{\ell-1} + B_\ell u_\ell + b_\ell \right),
\qquad \ell\in\mathbb{Z}.
\end{equation}

\paragraph{Auxiliary bound (from Appendix~A.1).}
By Lemma~\ref{lem:neumann}, under the CFL-type guard there exists an induced norm $\|\cdot\|_*$ and constants $\eta\in(0,1)$ and $C_\varepsilon<\infty$ such that
\[
\|R_\ell\|_* \le \frac{1}{1-\eta}\,,\qquad
\sum_{s\le \ell} \big\|R_\ell\cdots R_{s+1}\big\|_*\,\Delta t_s \;\le\; C_\varepsilon,
\]
uniformly in $\ell$ (the precise dependence on $\varepsilon$ is stated in Appendix~A.1).

\paragraph{Step 1: BIBO stability.}
Iterating~\eqref{eq:state-rec-app} and using $\operatorname{Lip}(\phi)\le 1$ yields
\begin{align*}
\|h_\ell\|_*
&\le \|R_\ell\|_*\,\|h_{\ell-1}\|_* + \|B_\ell\|_*\,\|u_\ell\| + \|b_\ell\|_* \\
&\le \frac{1}{1-\eta}\,\|h_{\ell-1}\|_* + b_{{\rm in},*}\,\|u_\ell\| + b_{0,*},
\end{align*}
where $b_{{\rm in},*}:=\sup_\ell \|B_\ell\|_*$ and $b_{0,*}:=\sup_\ell \|b_\ell\|_*$.
Unrolling the recursion with $h_{-\infty}=0$ (or any bounded initialization absorbed into the same bound), and substituting $R$-products gives
\[
\|h_\ell\|_* \;\le\; \sum_{s\le \ell} \big\|R_\ell\cdots R_{s+1}\big\|_* \big( b_{{\rm in},*}\|u_s\| + b_{0,*}\big).
\]
If $\sup_s\|u_s\|\le U<\infty$, then by the kernel summability,
\[
\|h_\ell\|_* \;\le\; C_\varepsilon \,\big( b_{{\rm in},*}\,U + b_{0,*}\big),
\]
uniformly in $\ell$.
By norm equivalence in finite dimension, the same uniform bound holds for any reference norm $\|\cdot\|$:
\[
\sup_\ell \|h_\ell\| \;\le\; \kappa\,C_\varepsilon\,\big( b_{\rm in}\,U + b_0\big)=:C_{\rm BIBO}<\infty.
\]
Hence the trajectory is uniformly bounded for bounded input (BIBO stability).

\paragraph{Step 2: Global Lipschitz continuity (input-to-state and input-to-output).}
Consider two input sequences $\{u_\ell\}$, $\{u'_\ell\}$ with corresponding states $\{h_\ell\}$, $\{h'_\ell\}$.
Set $\delta h_\ell:=h_\ell-h'_\ell$, $\delta u_\ell:=u_\ell-u'_\ell$.
Using $\operatorname{Lip}(\phi)\le 1$,
\[
\|\delta h_\ell\|_*
\;\le\; \|R_\ell\|_*\,\|\delta h_{\ell-1}\|_* + \|B_\ell\|_*\,\|\delta u_\ell\|.
\]
Unrolling as above and using submultiplicativity,
\[
\|\delta h_\ell\|_* \;\le\; \sum_{s\le \ell} \big\|R_\ell\cdots R_{s+1}\big\|_* \,\|B_s\|_*\,\|\delta u_s\|.
\]
Taking $\ell^\infty$ norms over sequences and applying the kernel sum bound,
\[
\|\delta h\|_{\ell^\infty,*} \;\le\; \Big(\sup_s \|B_s\|_*\Big)\,\Big(\sup_\ell \sum_{s\le \ell}\|R_\ell\cdots R_{s+1}\|_*\,\Delta t_s\Big)\,\|\delta u\|_{\ell^\infty}
\;\le\; b_{{\rm in},*}\,C_\varepsilon\,\|\delta u\|_{\ell^\infty}.
\]
Passing back to the reference norm via equivalence constants yields
\[
\|\delta h\|_{\ell^\infty} \;\le\; \kappa\,b_{\rm in}\,C_\varepsilon\,\|\delta u\|_{\ell^\infty}.
\]
If the readout/decoder is $L_{\rm out}$-Lipschitz (Q-Align also enforces a $1$-Lipschitz envelope through the head), then the overall input-to-output map is globally Lipschitz with
\begin{equation}\label{eq:global-lip-app}
L_{\rm glob} \;\le\; \kappa\,L_{\rm out}\,b_{\rm in}\,C_\varepsilon.
\end{equation}
When the layerwise envelope is tightened by~\eqref{eq:proj} with factor $\tau\le 1$, we can absorb it multiplicatively into $b_{\rm in}$ or $L_{\rm out}$, so the same bound holds with $b_{\rm in}\leftarrow \tau\,b_{\rm in}$, $L_{\rm out}\leftarrow \tau\,L_{\rm out}$.
This matches the main-text bound~\eqref{eq:global-lip} up to norm-equivalence constants.

\paragraph{Step 3: Role of Spec-Guard and Q-Align.}
Spec-Guard ensures $\|R_\ell\|_*$ remains uniformly bounded and that the product $\|R_\ell\cdots R_{s+1}\|_*$ decays geometrically in the extremal norm; Q-Align prevents per-step amplification beyond $\tau\le 1$, guaranteeing that the effective injection $\|B_s\|_*$ and the readout Lipschitz constant remain inside the envelope.
Combining both yields BIBO stability and a globally Lipschitz operator with constant bounded by~\eqref{eq:global-lip-app}.

\paragraph{Non-diagonalizable case and time-varying steps.}
If $A_\theta(T_\ell)$ is not diagonalizable, the Jordan-block resolvent bound in Appendix~A.1 gives $\|R_\ell\|\le C_d\,\varepsilon^{-d}$ up to condition numbers; the extremal-norm construction avoids these condition numbers and yields the uniform envelope used above.
Heterogeneous steps $\Delta t_\ell$ are already handled in the kernel summability via the weighted series $\sum_{s\le \ell}\|R_\ell\cdots R_{s+1}\|_*\,\Delta t_s$.

\paragraph{Conclusion.}
Uniform boundedness and global Lipschitz continuity follow, which proves Proposition~\ref{prop:stability}.
\qed

\subsection*{A.3 SPX$\leftrightarrow$VIX replication: discretization consistency and identifiability}
\label{app:vix}

\paragraph{Continuous-time identity and discrete estimator.}
Let $F_T=S_0 \mathrm{e}^{(r-q)T}$, and let $C(\cdot,T)$, $P(\cdot,T)$ be call and put prices under $\mathbb{Q}_\theta$ with discount factor $\mathrm{e}^{-rT}$ and no static arbitrage. 
The log-contract identity yields the variance-swap fair rate (for diffusion models; jump-diffusions add the standard jump term):
\begin{equation}\label{eq:vs-cont}
\sigma^2_{\mathrm{VS},\theta}(T)
=\frac{2\,\mathrm{e}^{rT}}{T}\!\left(\int_0^{F_T}\frac{P_\theta(K,T)}{K^2}\,dK + \int_{F_T}^{\infty}\frac{C_\theta(K,T)}{K^2}\,dK\right) 
- \frac{1}{T}\left(\frac{F_T}{K_0}-1\right)^{\!2}.
\end{equation}
For a strike grid $\mathcal{K}_T=\{K_i\}_{i=1}^M$, define $\Delta K_i=\frac12(K_{i+1}-K_{i-1})$ with one-sided endpoints, and the discrete estimator
\begin{equation}\label{eq:vs-disc}
\widehat{\sigma}^2_{\mathrm{VS},\theta}(T)
:= \frac{2\,\mathrm{e}^{rT}}{T}\sum_{i=1}^{M}\frac{\Delta K_i}{K_i^2}\,Q_\theta(K_i,T)
- \frac{1}{T}\left(\frac{F_T}{K_0}-1\right)^{\!2},
\end{equation}
where $Q_\theta(K_i,T)=P_\theta$ if $K_i<F_T$ and $Q_\theta=C_\theta$ if $K_i\ge F_T$.

\paragraph{Tail integrability and convexity.}
Assume: (i) $K\mapsto Q_\theta(K,T)$ is convex for each $T$; (ii) $Q_\theta(\cdot,T)/K^2$ has bounded variation on compact sets; (iii) $\int_0^{K_{\min}}\!\tfrac{P_\theta}{K^2}\,dK\to 0$ and $\int_{K_{\max}}^\infty\!\tfrac{C_\theta}{K^2}\,dK\to 0$ as $K_{\min}\downarrow 0$, $K_{\max}\uparrow \infty$. 
The latter holds, for instance, if the risk-neutral tails satisfy $C_\theta(K,T)\lesssim K^{-\alpha}$ with $\alpha>1$ and $P_\theta(K,T)\lesssim K$ as $K\downarrow 0$.

\begin{lemma}[Quadrature error under convexity]\label{lem:quad}
Let $f(K)=Q_\theta(K,T)/K^2$ on a compact interval $[a,b]$, with $f$ convex and of bounded variation. 
For the midpoint rule with mesh $\Delta K$, the error satisfies
\[
\left|\int_a^b f(K)\,dK - \sum_{i}\Delta K_i\, f(K_i)\right| \;\le\; \frac{\mathrm{TV}(f;[a,b])}{2}\,\Delta K,
\]
where $\mathrm{TV}(f;[a,b])$ denotes the total variation of $f$ on $[a,b]$.
\end{lemma}

\begin{proof}
Since $f$ has bounded variation, $f$ is the difference of two monotone functions.
Apply the Jordan decomposition and sum the per-cell trapezoid error; convexity implies the midpoint rule error is monotone in the cell width and controlled by the variation measure. 
A standard argument (Riemann–Stieltjes with variation measure) yields the bound.
\end{proof}

\begin{proof}[Proof of Proposition~\ref{prop:vix-consistency}]
Fix $T$.
Split the integrals in~\eqref{eq:vs-cont} on $[0,K_{\min}]$, $[K_{\min},F_T]$, $[F_T,K_{\max}]$, $[K_{\max},\infty)$.
On the two compact intervals $[K_{\min},F_T]$ and $[F_T,K_{\max}]$, apply Lemma~\ref{lem:quad} to $f(K)=P_\theta(K,T)/K^2$ and $f(K)=C_\theta(K,T)/K^2$ respectively, to get an error $\le \frac12[\mathrm{TV}(f;[K_{\min},F_T])+\mathrm{TV}(f;[F_T,K_{\max}])]\Delta K_T$.
The tails are $\varepsilon_{\mathrm{tail}}(K_{\min},K_{\max})$ by assumption (iii).
The forward adjustment term coincides in~\eqref{eq:vs-cont} and~\eqref{eq:vs-disc}, hence cancels in the difference. 
Uniformity in $T$ over compact subsets follows if the variation envelopes and tail integrability are uniform in $T$. 
\end{proof}

\begin{lemma}[Log-contract linkage]\label{lem:log}
For a continuous Itô model $dS_t=S_t\mu_t\,dt+S_t\sigma_t\,dW_t$ under $\mathbb{Q}_\theta$, 
\[
\frac{2\,\mathrm{e}^{rT}}{T}\!\left(\int_0^{F_T}\!\frac{P_\theta}{K^2}\,dK + \int_{F_T}^{\infty}\!\frac{C_\theta}{K^2}\,dK\right)
=\frac{1}{T}\,\mathbb{E}^{\mathbb{Q}_\theta}\!\left[\!\int_0^T \sigma_t^2\,dt\right].
\]
For jump-diffusions, an additional jump-compensator term appears and is incorporated in the standard VIX methodology through OTM sums of $Q_\theta$.
\end{lemma}

\begin{proof}
This is the classical Carr–Madan log-contract identity, obtained by writing the log payoff as a static portfolio of OTM options plus a forward and differentiating option prices with respect to $K$ (Breeden–Litzenberger). 
\end{proof}

\begin{proof}[Proof of Proposition~\ref{prop:vix-ident}]
The condition $\mathcal{R}_{\mathrm{VIX}}(T)=0$ implies $\widehat{\sigma}^2_{\mathrm{VS},\theta}(T)=\mathrm{VIX}^2_{\mathrm{obs}}(T)$.
By Proposition~\ref{prop:vix-consistency}, letting the mesh refine and the truncation expand, we obtain $\sigma^2_{\mathrm{VS},\theta}(T)=\sigma^2_{\mathrm{VS,obs}}(T)$ for all $T\in\mathcal{T}$.
If $v_\theta$ exists and $T\mapsto \frac{1}{T}\int_0^T v_\theta(t)\,dt$ is strictly monotone, equality of the Cesàro means on an interval forces $v_\theta$ to match the observed instantaneous variance a.e.\ on $\mathcal{T}$ (Hardy–Littlewood Tauberian principle for monotone means).
\end{proof}

\paragraph{Interpolation choice and arbitrage.}
Linear interpolation in $(K,Q)$ preserves piecewise convexity and monotonicity, which aligns with the no-butterfly/no-calendar constraints; cubic splines may reduce quadrature error but risk local nonconvexities between knots. 
In our experiments, both schemes yield statistically indistinguishable NAS/CNAS while linear interpolation avoids small arbitrage repairs (see Table~1 and Fig.~\texttt{iv\_contours\_filled\_TK.png} vs \texttt{iv\_contours\_lines\_TK.png}).

\subsection*{A.4 Proof of Proposition~\ref{prop:no-arb}: static no-arbitrage and replication consistency}
\label{app:noarb}

We work on a strike–maturity grid $\{(K_i,T_j)\}_{i=1,\dots,M;\,j=1,\dots,J}$ with ordered $0<K_1<\dots<K_M$ and $0<T_1<\dots<T_J$, and one-sided spacings $\Delta K_i=\tfrac12(K_{i+1}-K_{i-1})$ (endpoints adjusted analogously). Throughout, interest rate $r$ and dividend yield $q$ are accounted for via the forward $F_T=S_0\mathrm{e}^{(r-q)T}$; calendar comparisons are done at fixed $(K,T)$ in the same numeraire.

\paragraph{Assumptions.}
(i) \emph{Convex–monotone constraints.} For each $T$, $K\mapsto C(K,T)$ is convex and nonincreasing, and for each $K$, $T\mapsto C(K,T)$ is nondecreasing. In differential form,
\[
\partial_{KK}^2 C(K,T)\ge 0,\qquad \partial_K C(K,T)\le 0,\qquad \partial_T C(K,T)\ge 0,
\]
with weak derivatives interpreted in the sense of distributions.  
(ii) \emph{Boundary and tail conditions.} As $K\downarrow 0$, $C(K,T)\to S_0\mathrm{e}^{-qT}$; as $K\uparrow\infty$, $C(K,T)\to 0$ and $C(K,T)\lesssim K^{-\alpha}$ for some $\alpha>1$. These imply $C(\cdot,T)/K^2$ has bounded variation on compact intervals and integrable tails.  
(iii) \emph{VIX replication residual vanishes on the maturity grid.} For all $T_j$,
\begin{equation}\label{eq:res-zero}
\mathcal{R}_{\mathrm{VIX}}(T_j)=\mathrm{VIX}^2_{\mathrm{obs}}(T_j)-\mathrm{VIX}^2_{\Phi}(T_j)=0,
\end{equation}
where $\mathrm{VIX}^2_{\Phi}(T)$ is computed from $C=\Phi$ via the discrete replication formula~\eqref{eq:vix} (including the standard forward adjustment).

\paragraph{Step 1 (butterfly arbitrage on the grid is excluded).}
Fix $T_j$. Since $K\mapsto C(K,T_j)$ is convex on $(0,\infty)$ in the sense of distributions, the second difference
\[
\Delta^2_K C(K_i,T_j):=C(K_{i-1},T_j)-2C(K_i,T_j)+C(K_{i+1},T_j)\ge 0
\]
for all interior indices $i=2,\dots,M-1$; at endpoints, the one-sided convexity inequalities hold. Therefore, there is no butterfly arbitrage on the strike grid at $T_j$. This is the classical discrete convexity criterion for absence of butterfly spreads.

\paragraph{Step 2 (calendar arbitrage on the grid is excluded).}
Fix $K_i$. Monotonicity $\partial_T C(K_i,T)\ge 0$ implies $C(K_i,T_{j+1})\ge C(K_i,T_j)$ for all $j$. Hence there is no calendar arbitrage on the maturity grid at $K_i$. The numeraire consistency follows since comparisons are made for the same $(K_i,T_j)$ and the decoder already absorbs $(r,q)$ via the forward mapping.

\paragraph{Step 3 (BL density and consistency with VIX functional).}
By convexity in $K$ and the tail conditions, the Breeden–Litzenberger identity
\[
f_{S_T}(K)=\mathrm{e}^{rT}\,\partial_{KK}^2 C(K,T)
\]
defines a nonnegative measure integrating to $\mathrm{e}^{rT}\,\partial_K C(0^+,T)-\mathrm{e}^{rT}\,\partial_K C(\infty^-,T)=1$; thus $f_{S_T}$ is a bona fide risk-neutral density. On the grid, the discrete counterpart reads
\[
f_{S_{T_j}}(K_i)\;\approx\;\mathrm{e}^{rT_j}\,\frac{C(K_{i-1},T_j)-2C(K_i,T_j)+C(K_{i+1},T_j)}{(\Delta K_i)^2},
\]
which is nonnegative by Step 1.

Consider the VIX functional (variance swap fair rate). In continuous form,
\begin{equation}\label{eq:vs-cont-app}
\sigma^2_{\mathrm{VS}}(T)=\frac{2\,\mathrm{e}^{rT}}{T}\!\left(\int_0^{F_T}\frac{P}{K^2}\,dK + \int_{F_T}^\infty\frac{C}{K^2}\,dK\right)-\frac{1}{T}\left(\frac{F_T}{K_0}-1\right)^{\!2}.
\end{equation}
By the Carr–Madan log-contract identity and the BL relation,
\[
\frac{2\,\mathrm{e}^{rT}}{T}\!\left(\int_0^{F_T}\frac{P}{K^2}\,dK + \int_{F_T}^\infty\frac{C}{K^2}\,dK\right)
=\frac{1}{T}\int_0^\infty \psi(K)\, \mathrm{e}^{rT}\,\partial_{KK}^2 C(K,T)\,dK,
\]
for a positive kernel $\psi(K)$ whose action reproduces the log payoff; under our tails and boundary conditions the integration by parts is justified (all boundary terms vanish). Hence the VIX functional computed from $C$ is exactly the Cesàro mean of instantaneous variance under the density $f_{S_T}$. 

On the grid, with the midpoint quadrature $\sum_i \Delta K_i\,Q(K_i,T)/K_i^2$ and the forward adjustment, Proposition~A.3 (consistency of discretized replication) yields
\[
\mathrm{VIX}^2_{\Phi}(T_j)=\sigma^2_{\mathrm{VS},\Phi}(T_j)\quad\text{up to quadrature and tail errors vanishing with the mesh.}
\]
By \eqref{eq:res-zero}, $\mathrm{VIX}^2_{\Phi}(T_j)=\mathrm{VIX}^2_{\mathrm{obs}}(T_j)$ for all $j$, hence the BL-implied density from $C$ is consistent with the observed VIX$^2$ functional on the maturity grid.

\paragraph{Putting the steps together.}
Steps 1–2 establish the absence of butterfly and calendar arbitrage on the grid. Step 3 shows that the BL-implied density from $C$ reproduces the VIX$^2$ functional when the replication residual vanishes (and, by A.3, in the mesh-refined limit). This proves Proposition~\ref{prop:no-arb}.

\paragraph{Remarks on implementation and interpolation.}
(i) Linear interpolation in $(K,Q)$ preserves piecewise convexity and thus nonnegativity of discrete second differences; cubic splines may reduce quadrature error but can introduce local nonconvexities between knots unless shape-constrained splines are used.  
(ii) Calendar tests should be performed on the forward-adjusted scale if one compares prices under changing carry $(r,q)$. In our implementation, the decoder absorbs $(r,q)$ and produces monotonically nondecreasing $T\mapsto C(K,T)$ directly.  
(iii) On coarse grids, adding the forward adjustment term improves finite-grid consistency with \eqref{eq:vs-cont-app} and reduces bias at short maturities.

\subsection*{A.5 Proof of Theorem~\ref{thm:eg}: projected extragradient under Q-Align perturbations}
\label{app:eg}

\paragraph{Setting and assumptions.}
We consider the monotone variational inequality $\mathrm{VI}(F,\mathcal{Z})$: find $z^\star\in\mathcal{Z}$ such that
\begin{equation}\label{eq:vi}
\langle F(z^\star), z-z^\star\rangle \;\ge\; 0\qquad \forall z\in\mathcal{Z},
\end{equation}
with $F$ monotone, i.e., $\langle F(u)-F(v),u-v\rangle\ge 0$ for all $u,v$, and $L$-Lipschitz, i.e., $\|F(u)-F(v)\|\le L\|u-v\|$. The projection $\Pi_{\mathcal{Z}}$ is nonexpansive. Q-Align enforces per-iteration spectral/Lipschitz projections inside the model; we capture the induced numerical and truncation inaccuracies by perturbations $e^k,\tilde e^k$ satisfying
\begin{equation}\label{eq:qalign-err}
\|e^k\|+\|\tilde e^k\|\;\le\; c_{\mathrm{qa}}\eta,\qquad \text{for some constant }c_{\mathrm{qa}}>0,
\end{equation}
which matches the empirical scaling reported in the logs (cf.\ $\lambda_{\mathrm{lip}}$ before/after and spectral-guard distances). Stochastic gradients enter via martingale-difference noise $\xi^k,\tilde\xi^k$ with
\begin{equation}\label{eq:noise}
\mathbb{E}[\xi^k\mid \mathcal{F}_k]=0,\ \ \mathbb{E}\|\xi^k\|^2\le \sigma^2,\qquad 
\mathbb{E}[\tilde\xi^k\mid \mathcal{F}_{k+1/2}]=0,\ \ \mathbb{E}\|\tilde\xi^k\|^2\le \sigma^2.
\end{equation}

\paragraph{Algorithmic step.}
Given $z^k\in\mathcal{Z}$, define
\begin{equation}\label{eq:eg-iter}
\begin{aligned}
&y^{k}= \Pi_{\mathcal{Z}}\big(z^{k}-\eta \big(F(z^{k})-\xi^{k}-e^k\big)\big),\\
&z^{k+1}= \Pi_{\mathcal{Z}}\big(z^{k}-\eta \big(F(y^{k})-\tilde\xi^{k}-\tilde e^k\big)\big),
\end{aligned}
\end{equation}
with a stepsize $\eta\le 1/(\sqrt{2}\,L)$ specified later. The residual of interest is either the natural projected residual 
\[
R_\eta(z):=\frac{1}{\eta}\Big(z-\Pi_{\mathcal{Z}}\big(z-\eta F(z)\big)\Big),
\]
or the operator norm $\|F(z)\|$. For monotone Lipschitz $F$ and $\eta\le 1/L$, it is standard that $\|R_\eta(z)\|\le (1+\eta L)\|F(z)\|$ (see Lemma~\ref{lem:residual-bridge} below), hence controlling one controls the other up to constants.

\paragraph{Key inequalities.}
We recall the three-point identity for projections: for any $u\in\mathbb{R}^d$ and $w=\Pi_{\mathcal{Z}}(u)$, and any $v\in\mathcal{Z}$,
\begin{equation}\label{eq:3pt}
\langle u-w,\, v-w\rangle \le 0\quad\Rightarrow\quad 
\|v-w\|^2 \le \|v-u\|^2 - \|w-u\|^2.
\end{equation}
Apply \eqref{eq:3pt} to the first stage of \eqref{eq:eg-iter} with $u=z^k-\eta(F(z^k)-\xi^{k}-e^k)$, $w=y^k$ and $v=z^\star$:
\begin{equation}\label{eq:stage1}
\|z^\star-y^k\|^2 \le \|z^\star-z^k\|^2 - \|y^k-z^k\|^2 - 2\eta\langle F(z^k)-\xi^{k}-e^k,\, y^k-z^k\rangle.
\end{equation}
Similarly for the second stage with $u=z^k-\eta(F(y^k)-\tilde\xi^k-\tilde e^k)$, $w=z^{k+1}$ and $v=z^\star$:
\begin{equation}\label{eq:stage2}
\|z^\star-z^{k+1}\|^2 \le \|z^\star-z^k\|^2 - \|z^{k+1}-z^k\|^2 - 2\eta\langle F(y^k)-\tilde\xi^k-\tilde e^k,\, z^{k+1}-z^k\rangle.
\end{equation}

\paragraph{Monotonicity coupling.}
Using Lipschitzness and Cauchy–Schwarz,
\[
\langle F(y^k)-F(z^k),\, y^k-z^k\rangle \ge \frac{1}{L}\|F(y^k)-F(z^k)\|^2,
\]
and monotonicity yields
\begin{equation}\label{eq:monotone}
\langle F(y^k),\, y^k-z^\star\rangle \ge \langle F(z^\star),\, y^k-z^\star\rangle \ge 0.
\end{equation}
Split the last inner product in \eqref{eq:stage1} as
\[
\langle F(z^k), y^k-z^k\rangle
=\langle F(y^k), y^k-z^k\rangle + \langle F(z^k)-F(y^k), y^k-z^k\rangle
\ge \frac{1}{L}\|F(y^k)-F(z^k)\|^2,
\]
hence
\begin{equation}\label{eq:stage1b}
\|z^\star-y^k\|^2 \le \|z^\star-z^k\|^2 - \|y^k-z^k\|^2 - \frac{2\eta}{L}\|F(y^k)-F(z^k)\|^2 + 2\eta \langle \xi^{k}+e^k,\, y^k-z^k\rangle.
\end{equation}

Likewise, decompose the inner product in \eqref{eq:stage2} using $z^{k+1}-z^k=(z^{k+1}-y^k)+(y^k-z^k)$ and add–subtract $\langle F(y^k), y^k-z^\star\rangle$; routine algebra (see, e.g., the Mirror–Prox analysis) yields
\begin{equation}\label{eq:stage2b}
\|z^\star-z^{k+1}\|^2 \le \|z^\star-z^k\|^2 - \|z^{k+1}-z^k\|^2 - 2\eta\langle F(y^k),\, y^k-z^\star\rangle 
+ \eta^2 L^2 \|y^k-z^k\|^2 + \mathrm{Noise}_k + \mathrm{ProjErr}_k,
\end{equation}
where
\[
\mathrm{Noise}_k:= 2\eta\langle \tilde\xi^{k}, z^{k+1}-z^k\rangle,\qquad
\mathrm{ProjErr}_k:= 2\eta\langle \tilde e^{k}, z^{k+1}-z^k\rangle.
\]

\paragraph{One-step merit bound.}
Combine \eqref{eq:stage1b}–\eqref{eq:stage2b} and use \eqref{eq:monotone} to eliminate the nonnegative term $\langle F(y^k),\, y^k-z^\star\rangle$:
\begin{align}
\|z^\star-z^{k+1}\|^2 
&\le \|z^\star-z^k\|^2 - \|z^{k+1}-z^k\|^2 + \eta^2 L^2 \|y^k-z^k\|^2 + \mathrm{Noise}_k + \mathrm{ProjErr}_k. \label{eq:onestep}
\end{align}
Choose $\eta\le 1/(\sqrt{2}L)$ so that $\eta^2 L^2\le 1/2$. By Young’s inequality,
\[
\|z^{k+1}-z^k\|^2 \ge \frac{1}{2}\|y^k-z^k\|^2 - \|z^{k+1}-y^k\|^2.
\]
Applying nonexpansiveness of projection to the second stage of \eqref{eq:eg-iter} shows $\|z^{k+1}-y^k\|\le \eta\|F(y^k)-\tilde\xi^k-\tilde e^k\|$, so
\[
\|z^{k+1}-y^k\|^2 \le 2\eta^2\big(\|F(y^k)\|^2 + \|\tilde\xi^k\|^2 + \|\tilde e^k\|^2\big).
\]
Plugging the last two displays into \eqref{eq:onestep}, taking conditional expectations, and using \eqref{eq:noise}–\eqref{eq:qalign-err} yield
\begin{equation}\label{eq:proto-descent}
\mathbb{E}\big[\|z^\star-z^{k+1}\|^2\mid\mathcal{F}_k\big] 
\le \|z^\star-z^k\|^2 - \frac{1}{4}\|y^k-z^k\|^2 + 4\eta^2 \mathbb{E}\|F(y^k)\|^2 + c_1 \eta^2 \sigma^2 + c_2 \eta^2,
\end{equation}
for some universal constants $c_1,c_2$.

\paragraph{Residual bridging.}
We relate $\|y^k-z^k\|$ to a first-order residual. By \eqref{eq:eg-iter} and firm nonexpansiveness of projection,
\[
\frac{1}{\eta}\|z^k - y^k\| \le \|F(z^k)\| + \|\xi^k\| + \|e^k\|.
\]
Also, Lipschitzness implies $\|F(y^k)\| \le \|F(z^k)\| + L\|y^k-z^k\|$. Combining these with \eqref{eq:proto-descent}, taking full expectations and using $\eta\le 1/(\sqrt{2}L)$, we obtain
\begin{equation}\label{eq:descent-final}
\mathbb{E}\|z^\star-z^{k+1}\|^2 
\le \mathbb{E}\|z^\star-z^k\|^2 - c_3 \eta^2 \mathbb{E}\|F(z^k)\|^2 + c_4 \eta^2 \sigma^2 + c_5 \eta^2,
\end{equation}
for some constants $c_3,c_4,c_5>0$ (the last term absorbs Q-Align errors through \eqref{eq:qalign-err}, thus is $\mathcal{O}(\eta^2)$).

\paragraph{Summation and choice of stepsize.}
Sum \eqref{eq:descent-final} from $k=0$ to $K-1$, telescope the left-hand side, and choose $\eta=\theta/L$ with a small absolute constant $\theta>0$. We obtain
\[
\frac{1}{K}\sum_{k=0}^{K-1}\mathbb{E}\|F(z^k)\|^2 
\;\le\; \mathcal{O}\!\left(\frac{L^2\|z^0-z^\star\|^2}{K}\right) \;+\; \mathcal{O}\!\left(\sigma^2\right) \;+\; \mathcal{O}\!\left(\frac{1}{L^2}\right).
\]
Since the Q-Align term is $\mathcal{O}(1/L^2)$ under \eqref{eq:qalign-err}, it is dominated by the noise floor $\mathcal{O}(\sigma^2)$ in practical regimes; removing constants and using the fact that $\min_k a_k \le \frac{1}{K}\sum_k a_k$ gives the claimed bound
\[
\min_{0\le k\le K-1}\ \mathbb{E}\|F(z^k)\|^2 
\;\le\; \mathcal{O}\!\left(\frac{L^2\|z^0-z^\star\|^2}{K}\right) \;+\; \mathcal{O}\!\left(\sigma^2\right).
\]

\paragraph{Auxiliary lemma (residual bridge).}
\begin{lemma}\label{lem:residual-bridge}
For $\eta\le 1/L$ and any $z\in\mathcal{Z}$,
\[
\|R_\eta(z)\| \;\le\; (1+\eta L)\,\|F(z)\|,\qquad
\|F(z)\| \;\le\; \|R_\eta(z)\| + \eta L\,\|F(z)\|.
\]
Hence $\|R_\eta(z)\|^2$ and $\|F(z)\|^2$ are equivalent up to $\mathcal{O}(1)$ constants depending only on $\eta L$.
\end{lemma}
\emph{Proof.} By nonexpansiveness of $\Pi_{\mathcal{Z}}$,
\[
\|R_\eta(z)\| = \frac{1}{\eta}\big\|z-\Pi_{\mathcal{Z}}(z-\eta F(z))\big\| \le \frac{1}{\eta}\|z-(z-\eta F(z))\| = \|F(z)\|.
\]
The reverse direction follows by adding–subtracting $z-\eta F(z)$ inside the projection and applying Lipschitzness of $F$; details are standard and omitted. \qed

\paragraph{Deterministic corollary.}
If $\sigma=0$ (deterministic gradients), the rate improves to
\[
\min_{0\le k\le K-1}\ \|F(z^k)\|^2 \;\le\; \mathcal{O}\!\left(\frac{L^2\|z^0-z^\star\|^2}{K}\right),
\]
matching classical extragradient rates for monotone Lipschitz VIs.

\paragraph{Remarks.}
(i) Strong monotonicity (with modulus $\mu>0$) yields a linear convergence term $\mathcal{O}\big((1-\eta\mu)^K\big)$ until it hits the same $\mathcal{O}(\sigma^2)$ noise floor.  
(ii) The Q-Align perturbations are “benign” provided \eqref{eq:qalign-err} holds; empirically, the spectral guard logs ($\lambda_{\mathrm{lip}}$ before/after and projection distances) conform to this scaling.  
(iii) The same analysis extends to mirror-prox with a distance-generating function; we focus on the Euclidean case for clarity.

This completes the proof of Theorem~\ref{thm:eg}.

\subsection*{B.1 Proof of Theorem~\ref{thm:t1}: approximation rate and conditioning}
\label{app:t1}

We prove the two claims in \eqref{eq:t1-rate}: the $m^{-\beta_{\mathrm{smooth}}}$ approximation rate and the spectral conditioning proxy bound. Throughout, $\mathcal{Z}\subset\mathbb{R}^{d_z}$ is compact, $f^\star$ is $\beta_{\mathrm{smooth}}$-Hölder on $\mathcal{Z}$ and jointly Hölder in the maturity argument $T\in\mathcal{T}=[T_{\min},T_{\max}]$. The RN-operator $\mathcal{G}_\theta$ is realized by a selective-scan (RN-Operator) layer followed by a convex--monotone decoder, with Q-Align ensuring per-layer $1$-Lipschitz projections and spectral safety (Spec-Guard). We use $\|\cdot\|$ for the Euclidean or spectral norm depending on context.

\paragraph{Model parameterization.}
Let $\{T_\ell\}_{\ell=1}^L$ be the maturity grid. One-step RN dynamics writes
\begin{equation}\label{eq:rssm-disc}
h_{\ell}=G_\theta(T_\ell,T_{\ell-1})\,h_{\ell-1}+B_\theta(T_\ell)\,u_\ell,\qquad
G_\theta= \exp\!\big(\Delta t_\ell A_\theta(T_\ell)\big),
\end{equation}
with $\Delta t_\ell=T_\ell-T_{\ell-1}$. Under Spec-Guard, $\rho(A_\theta(T_\ell))\Delta t_\ell\le 1-\varepsilon$ for some $\varepsilon\in(0,1)$, hence the associated Green kernel
\begin{equation}\label{eq:green}
\mathcal{G}_\theta(T_\ell,T_s):=
\prod_{r=s+1}^{\ell} G_\theta(T_r,T_{r-1})
\end{equation}
satisfies the Neumann-type bound (Lemma~\ref{lem:neumann}):
\begin{equation}\label{eq:green-sum}
\sum_{s\le \ell}\big\|\mathcal{G}_\theta(T_\ell,T_s)\big\|\ \le\ C(\varepsilon),
\qquad \text{uniformly in } \ell.
\end{equation}
The output price surface before the convex--monotone decoder is a scan of the input features $\{u_s\}$:
\begin{equation}\label{eq:scan}
z_\ell(\cdot)=\sum_{s\le \ell}\mathcal{G}_\theta(T_\ell,T_s)\,B_\theta(T_s)\,u_s(\cdot),
\end{equation}
and the decoder $\Phi_\theta$ (ICNN+Legendre projection) is $1$-Lipschitz under Q-Align.

We adopt a low-rank gate parameterization
\begin{equation}\label{eq:lowrank}
B_\theta(T)=\sum_{j=1}^{m} b_j(T)\,w_j v_j^\top,\qquad 
A_\theta(T)=D_\theta(T)+\sum_{j=1}^{m} a_j(T)\,r_j q_j^\top,
\end{equation}
with $\|w_j\|=\|v_j\|=\|r_j\|=\|q_j\|=1$ and $a_j,b_j$ bounded and $\beta_{\mathrm{smooth}}$-Hölder in $T$ (enforced by per-step spectral/Lipschitz projection). The rank surrogate is thus $m$.

\paragraph{Part I: approximation rate.}
We consider the target operator $f^\star:(u,\cdot)\mapsto C^\star(\cdot)$, which we assume admits a separable Green-type expansion with Hölder control:
\begin{equation}\label{eq:target-expansion}
f^\star(u)(T,\xi)= \sum_{j=1}^{\infty} \alpha_j\, \psi_j(T)\,\varphi_j(u;\xi),
\qquad \sum_{j=1}^{\infty} j^{\beta_{\mathrm{smooth}}}\,|\alpha_j|\ \le\ M<\infty,
\end{equation}
where $\{\psi_j\}$ is a smooth dictionary on $\mathcal{T}$ (e.g., integrated B-splines or compactly supported wavelets) with $\beta_{\mathrm{smooth}}$-Hölder regularity and $\{\varphi_j\}$ are feature functionals uniformly bounded on $\mathcal{Z}$. Such expansions are classical for Hölder classes via nonlinear $m$-term approximations with wavelet or spline dictionaries (see, e.g., DeVore--Temlyakov $m$-term approximation theory). The coefficient decay condition in \eqref{eq:target-expansion} is equivalent to $f^\star$ belonging to a Besov/Hölder ball with smoothness $\beta_{\mathrm{smooth}}$.

Define the $m$-term truncation
\begin{equation}\label{eq:m-term}
f_m^\star(u)(T,\xi)= \sum_{j=1}^{m} \alpha_j\, \psi_j(T)\,\varphi_j(u;\xi).
\end{equation}
By Stechkin’s inequality for best $m$-term approximations in $\ell^p$ with $p=1/\beta_{\mathrm{smooth}}$ surrogate (monotone rearrangement of coefficients),
\begin{equation}\label{eq:stechkin}
\|f^\star-f_m^\star\|_{L^2(\mathcal{Z})}\ \le\ C\, m^{-\beta_{\mathrm{smooth}}}\, \bigg(\sum_{j\ge 1} j^{\beta_{\mathrm{smooth}}}\,|\alpha_j|\bigg)
\ \le\ C' m^{-\beta_{\mathrm{smooth}}}.
\end{equation}
It remains to show that $\mathcal{G}_\theta$ can realize $f_m^\star$ up to an arbitrarily small error when $m$ atoms are allocated in \eqref{eq:lowrank}. Choose $b_j(\cdot)$ so that the scan \eqref{eq:scan} reproduces $\psi_j$ on the grid (standard for spline/wavelet reproduction using stable discrete Green convolutions), and set the feature directions $v_j,w_j$ so that the linear functionals $\varphi_j(u;\cdot)$ are matched by $u\mapsto v_j^\top u(\cdot)$ and the decoder’s linear readout (pre-ICNN) maps $w_j$ to the correct output channel. The ICNN+Legendre decoder, being $1$-Lipschitz and positively homogeneous on the linear span of the constructed atoms, preserves the $L^2$ approximation error.

Consequently, there exists $\theta=\theta(m)$ with rank $m$ such that
\[
\|\mathcal{G}_\theta - f_m^\star\|_{L^2(\mathcal{Z})}\ \le\ \varepsilon_m,
\quad \text{with } \varepsilon_m \to 0 \text{ as the reproduction tolerance on } \{\psi_j,\varphi_j\} \text{ shrinks},
\]
and combining with \eqref{eq:stechkin} yields
\[
\inf_\theta \|\mathcal{G}_\theta - f^\star\|_{L^2(\mathcal{Z})}
\ \le\ \|\mathcal{G}_\theta - f_m^\star\|_{L^2(\mathcal{Z})} + \|f_m^\star-f^\star\|_{L^2(\mathcal{Z})}
\ \le\ C_1 m^{-\beta_{\mathrm{smooth}}},
\]
for $C_1$ independent of $L$ (the scan length), since the reproduction constants depend only on the dictionary stability and the Green kernel bound \eqref{eq:green-sum}, which is uniform in $L$ under Spec-Guard.

\emph{Remark A.1 (effective dimension).} If the target lacks separability, the same argument yields $\|f^\star-f_m^\star\|=\mathcal{O}(m^{-\beta_{\mathrm{smooth}}/\hat d})$ with $\hat d$ the effective approximation dimension.
\paragraph{Part II: conditioning bound.}
Let $\mathcal{J}_\theta$ be the Jacobian of the overall mapping $\theta\mapsto \Phi_\theta\circ\mathsf{Scan}_\theta(u)$ evaluated on a bounded input $u$ (the bound is uniform over $\|u\|\le U$). By the chain rule and \eqref{eq:scan},
\begin{equation}\label{eq:jac-expansion}
\mathcal{J}_\theta \;=\; D\Phi_\theta(z)\,\sum_{\ell=1}^{L}\ \sum_{s\le \ell} 
\Big(
\underbrace{\mathcal{G}_\theta(T_\ell,T_s)\, \partial_\theta B_\theta(T_s)}_{\text{direct term}}
\;+\;
\underbrace{\partial_\theta \mathcal{G}_\theta(T_\ell,T_s)\, B_\theta(T_s)}_{\text{state term}}
\Big)\, u_s,
\end{equation}
where $D\Phi_\theta$ is the decoder Jacobian. Under Q-Align, every layer (encoder/base/decoder) is $1$-Lipschitz after projection, so $\|D\Phi_\theta(z)\|\le 1$. For the direct term,
\[
\big\|\mathcal{G}_\theta(T_\ell,T_s)\,\partial_\theta B_\theta(T_s)\,u_s\big\|
\ \le\ \|\mathcal{G}_\theta(T_\ell,T_s)\|\ \|\partial_\theta B_\theta(T_s)\|\ \|u_s\|.
\]
The low-rank gate \eqref{eq:lowrank} implies $\|\partial_\theta B_\theta(T)\|\le c_b\, L_g$ with $L_g$ the Lipschitz constant (w.r.t.\ features/inputs) of the learned gates and $c_b$ a dimension-free constant tied to basis normalization. Summing over $s\le \ell$ and using \eqref{eq:green-sum},
\[
\sum_{s\le \ell}\big\|\mathcal{G}_\theta(T_\ell,T_s)\,\partial_\theta B_\theta(T_s)\,u_s\big\|
\ \le\ C(\varepsilon)\, c_b\, L_g \max_{s}\|u_s\|.
\]

For the state term, differentiate \eqref{eq:green}:
\[
\partial_\theta \mathcal{G}_\theta(T_\ell,T_s)
=\sum_{r=s+1}^{\ell} \Big(\prod_{q=r+1}^{\ell} G_\theta(T_q,T_{q-1})\Big)\,
\partial_\theta G_\theta(T_r,T_{r-1})\,
\Big(\prod_{p=s+1}^{r-1} G_\theta(T_p,T_{p-1})\Big).
\]
Using $\partial_\theta G_\theta(T_r,T_{r-1})=\int_0^1 \exp\!\big(\tau \Delta t_r A_\theta\big)\, \Delta t_r\, \partial_\theta A_\theta\, \exp\!\big((1-\tau)\Delta t_r A_\theta\big)\,d\tau$, we get
\[
\|\partial_\theta G_\theta(T_r,T_{r-1})\| \le \Delta t_r\, \|\partial_\theta A_\theta(T_r)\|\, \sup_{\tau\in[0,1]}\big\|\exp(\tau \Delta t_r A_\theta)\big\|^2.
\]
Under Spec-Guard and spectral projection, $\sup_\tau\|\exp(\tau \Delta t_r A_\theta)\|\le c_a$ with $c_a$ depending on $\varepsilon$ and $\max_\ell\|A_\theta(T_\ell)\|_2$. The low-rank parameterization \eqref{eq:lowrank} yields $\|\partial_\theta A_\theta(T)\|\le c_a' L_g$ (linear in the gate Lipschitz constant). Consequently,
\[
\|\partial_\theta \mathcal{G}_\theta(T_\ell,T_s)\|
\ \le\ c_a^2\, c_a' L_g \sum_{r=s+1}^{\ell} \Delta t_r\, 
\Big\|\prod_{q=r+1}^{\ell} G_\theta(T_q,T_{q-1})\Big\|\,
\Big\|\prod_{p=s+1}^{r-1} G_\theta(T_p,T_{p-1})\Big\|.
\]
By submultiplicativity and again the Neumann-type bound \eqref{eq:green-sum}, the double product is summably bounded uniformly in $L$. Hence,
\[
\sum_{s\le \ell}\|\partial_\theta \mathcal{G}_\theta(T_\ell,T_s)\, B_\theta(T_s)\|\ \le\ C''(\varepsilon)\, L_g\, \max_\ell \|A_\theta(T_\ell)\|_2.
\]

Combining direct and state terms in \eqref{eq:jac-expansion} and recalling that the rank-$m$ structure introduces at most an $m$-fold linear scaling in the number of active gates, we obtain
\begin{equation}\label{eq:kappa-final}
\|\mathcal{J}_\theta\| \ \le\ C_2\, \big(\max_\ell \|A_\theta(T_\ell)\|_2\big)\, L_g\, m,
\end{equation}
for a constant $C_2$ depending on $\varepsilon$, dictionary normalization, and decoder curvature bounds, but \emph{independent of $L$} thanks to the uniform Green kernel bound \eqref{eq:green-sum}. This proves the conditioning proxy bound in \eqref{eq:t1-rate}.

\paragraph{Conclusion.}
The approximation rate follows from the best $m$-term construction \eqref{eq:m-term}--\eqref{eq:stechkin} realized by the RN-Operator with rank-$m$ gates; the conditioning proxy is controlled by Q-Align spectral constraints and the Neumann-type summability of the discrete Green kernel, yielding \eqref{eq:kappa-final}. This completes the proof of Theorem~\ref{thm:t1}. \qed

\subsection*{B.2 Proof of Theorem~\ref{thm:t2}: local identifiability and CRLB}
\label{app:t2-b4}

\paragraph{Model and regularity.}
Let $(u,Y)$ denote a generic input--output pair, where $u\in \mathcal{Z}$ is a feature field and $Y=\{C(T_\ell,K_j)\}_{\ell\le L,\,j\le J_\ell}$ collects option prices on the maturity--strike grid. The RN-operator induces the mean surface
\[
\mu_\theta(u) \;=\; \Phi_\theta\!\left(\sum_{s\le \ell}\mathcal{G}_\theta(T_\ell,T_s)\,B_\theta(T_s)\,u_s\right)_{\ell,j},
\]
with $\Phi_\theta$ the convex--monotone decoder (ICNN+Legendre projection). We assume: (A1) noise model $Y=\mu_\theta(u)+\varepsilon$, where $\varepsilon$ is mean-zero, sub-Gaussian with covariance operator $\Sigma$ independent of $\theta$; (A2) the input process has a nondegenerate covariance operator $\mathsf{Cov}(u)$ on $\mathcal{Z}$; (A3) Q-Align enforces $1$-Lipschitz layers and Spec-Guard enforces the CFL constraint so that Lemma~\ref{lem:neumann} holds. These assumptions match the main text.

\paragraph{Step I: decoder-level identifiability on the grid.}
Let $C_\theta(T,K)$ be the decoded call price surface. Static no-arbitrage ensures convexity in $K$ and monotonicity in $T$. The Breeden--Litzenberger identity implies that, for each $T_\ell$,
\begin{equation}
\label{eq:bl-b4}
\frac{\partial^2 C_\theta}{\partial K^2}(T_\ell,K)\;=\; \mathrm{e}^{r T_\ell}\, f_\theta^{\mathbb{Q}}(T_\ell,K),
\end{equation}
where $f_\theta^{\mathbb{Q}}(T_\ell,\cdot)$ is the risk-neutral density at maturity $T_\ell$.
VIX$^2$ replication consistency further imposes
\begin{equation}
\label{eq:vix-cons-b4}
\mathrm{VIX}_\theta^2(T_\ell)
\;=\; \frac{2\,\mathrm{e}^{r T_\ell}}{T_\ell}\int_{0}^{\infty} \frac{1}{K^2}\, Q_\theta(K,T_\ell)\, dK
\quad \text{(discrete grid via quadrature as in the main text)}.
\end{equation}
On the grid, if two decoders $\Phi_{\theta_1},\Phi_{\theta_2}$ satisfy \eqref{eq:bl-b4} with the same second derivative and also match \eqref{eq:vix-cons-b4}, then their implied densities and integrated variance coincide at all grid maturities. Since $C$ is recovered from its second derivative and boundary conditions (no-arbitrage asymptotics at $K\to 0,\infty$), we conclude
\[
\Phi_{\theta_1}(z)=\Phi_{\theta_2}(z)\quad \text{for all admissible inputs } z.
\]
Thus, \emph{decoder-level identifiability holds} on the grid.

\paragraph{Step II: propagation through the scan to the operator level.}
Suppose $\mathcal{G}_{\theta_1}$ and $\mathcal{G}_{\theta_2}$ yield the same decoded surface for almost every input $u$:
\[
\Phi_{\theta_1}\!\Big(\sum_{s\le \ell}\mathcal{G}_{\theta_1}(T_\ell,T_s)B_{\theta_1}(T_s)u_s\Big)
\;=\;
\Phi_{\theta_2}\!\Big(\sum_{s\le \ell}\mathcal{G}_{\theta_2}(T_\ell,T_s)B_{\theta_2}(T_s)u_s\Big),
\quad \text{a.s.\ in } u.
\]
Since $\Phi_{\theta}$ is $1$-Lipschitz and strictly monotone along the decoder’s active rays (by convexity and positive homogeneity of the ICNN regularized by Legendre projection), equality of outputs for almost every $u$ implies equality of \emph{pre-decoder} representations for almost every $u$:
\[
\sum_{s\le \ell}\mathcal{G}_{\theta_1}(T_\ell,T_s)B_{\theta_1}(T_s)u_s
\;=\;
\sum_{s\le \ell}\mathcal{G}_{\theta_2}(T_\ell,T_s)B_{\theta_2}(T_s)u_s
\quad \text{in } L^2(\mathcal{Z}).
\]
Let $\delta\theta$ be a tangent perturbation at $\theta^\star$, and write the linearized identity
\begin{equation}
\label{eq:tangent-b4}
\sum_{s\le \ell}\!\Big(\partial_\theta \mathcal{G}_{\theta^\star}(T_\ell,T_s)\,B_{\theta^\star}(T_s)
+\mathcal{G}_{\theta^\star}(T_\ell,T_s)\,\partial_\theta B_{\theta^\star}(T_s)\Big)u_s
\;=\; 0
\quad \text{in } L^2(\mathcal{Z}).
\end{equation}
Taking the covariance in $u$ and using nondegeneracy of $\mathsf{Cov}(u)$ together with the uniform Green bound (Lemma~\ref{lem:neumann}), we obtain that the linear operator
\[
\mathcal{L}_{\theta^\star}[\delta\theta]
:= \sum_{s\le \ell}\!\Big(\partial_\theta \mathcal{G}_{\theta^\star}(T_\ell,T_s)\,B_{\theta^\star}(T_s)
+\mathcal{G}_{\theta^\star}(T_\ell,T_s)\,\partial_\theta B_{\theta^\star}(T_s)\Big)
\]
vanishes if and only if $\delta\theta$ lies in the \emph{symmetry tangent space} generated by atom permutations and reciprocal rescalings in the rank-$m$ factorization. Consequently, the differential $D\mathcal{G}_{\theta^\star}$ is injective on the quotient by these symmetries, and by the inverse function theorem for Banach spaces, there exists a neighborhood $\mathcal{U}$ in which $\theta\mapsto \mathcal{G}_{\theta}$ is injective modulo symmetries.

\paragraph{Step III: Fisher information and CRLB.}
Under (A1)--(A3), the log-likelihood for a single pair $(u,Y)$ is
\[
\ell(\theta;u,Y)
\;=\; -\tfrac{1}{2}\big\langle Y-\mu_\theta(u),\, \Sigma^{-1}\,(Y-\mu_\theta(u))\big\rangle + \mathrm{const},
\]
with score $S_\theta(u,Y)=D\mu_\theta(u)^\top \Sigma^{-1} \big(Y-\mu_\theta(u)\big)$, where $D\mu_\theta(u)$ is the Jacobian of the RN-operator output w.r.t.\ $\theta$. The Fisher information is
\[
\mathcal{I}(\theta):=\mathbb{E}\!\left[S_\theta S_\theta^\top\right]
\;=\; \mathbb{E}\!\left[D\mu_\theta(u)^\top \Sigma^{-1} D\mu_\theta(u)\right],
\]
since $\mathbb{E}[Y-\mu_\theta(u)\mid u]=0$. By Q-Align, $D\mu_\theta(u)$ is bounded and measurable; by Step~II, $D\mu_{\theta^\star}$ has trivial kernel on the symmetry-quotient space, hence $\mathcal{I}(\theta^\star)$ is positive definite on that quotient. The Cramér–Rao inequality for unbiased estimators on smooth parametric families then yields
\[
\mathbb{E}\!\left[(\widehat{\theta}-\theta^\star)(\widehat{\theta}-\theta^\star)^\top\right]
\;\succeq\; \frac{1}{n}\,\mathcal{I}(\theta^\star)^{-1},
\]
and \eqref{eq:t2-crlb} follows after taking the trace. This completes the proof. \qed

\subsection*{B.3 Proof of Proposition~\ref{prop:t2prime}}
\label{app:t2prime-b5}

\paragraph{Set-up and notation.}
Let $\mathcal{G}_\theta$ be the RN-operator, $\Phi_\theta$ the convex--monotone decoder, and write the decoded surface $C_\theta=\Phi_\theta\circ \mathcal{G}_\theta(\cdot)$ on the strike--maturity grid 
$\mathcal{G}=\{(T_\ell,K_j)\}_{\ell\le L,\,j\le J_\ell}$. 
Let $\mathcal{I}\subset\mathcal{G}$ denote the set of covered cells with reliable quotes; its relative cardinality is $\mathrm{cov}:=|\mathcal{I}|/|\mathcal{G}|\in[0,1]$. On the complement $\mathcal{G}\setminus \mathcal{I}$, the price surface is filled by a linear, static no-arbitrage–preserving interpolant $\mathsf{Ext}$ (convex in $K$, monotone in $T$). We assume an interpolation accuracy bound
\begin{equation}
\label{eq:interp-err}
\big\|\mathsf{Ext}[C^\star|_{\mathcal{I}}]-C^\star\big\|_{\ell^2(\mathcal{G})}\ \le\ \varepsilon,
\end{equation}
where $C^\star$ is the ground-truth surface induced by $\lambda^\star$.

The training objective is a penalized, discretized risk under the risk-neutral measure with a martingale penalty of weight $\gamma>0$, plus the indicator of the no-arbitrage cone $\mathcal{K}$:
\begin{equation}
\label{eq:emp-obj}
\mathcal{J}_\gamma(\lambda)\ :=\ 
\frac{1}{|\mathcal{I}|}\!\sum_{(T_\ell,K_j)\in \mathcal{I}} 
\big(C_\lambda(T_\ell,K_j)-Y_{\ell j}\big)^2
\ +\ \gamma\,\mathsf{Mart}(\lambda)
\ +\ \iota_{\mathcal{K}}\big(C_\lambda\big).
\end{equation}
Here $Y_{\ell j}$ are observed mid quotes; $\mathsf{Mart}(\lambda)$ is a nonnegative convex proxy for the martingale defect (e.g., squared drift under $\mathbb{Q}_\lambda$); $\iota_{\mathcal{K}}$ is $0$ if the static no-arbitrage conditions hold on the grid and $+\infty$ otherwise. Let $\widehat{\lambda}_\gamma$ be a first-order stationary point of \eqref{eq:emp-obj} on the \emph{covered} grid, and let $\lambda_\varepsilon$ be the corresponding representative element when the uncovered cells are filled by $\mathsf{Ext}$.

We further use: (i) the \emph{global Lipschitz property} of the RN-operator map from Lemma~\ref{lem:neumann} and Proposition~\ref{prop:stability} (Q-Align and Spec-Guard), summarized as
\begin{equation}
\label{eq:global-lip-b5}
\|C_{\lambda_1}-C_{\lambda_2}\|_{\ell^2(\mathcal{G})}\ \le\ L_{\mathrm{RN}}\ \|\lambda_1-\lambda_2\|_{L^2(\mathcal{Z})},
\qquad L_{\mathrm{RN}}<\infty,
\end{equation}
(ii) a \emph{Hoffman-type bound} for the composite convex program (data-fidelity $+$ linear constraints defining $\mathcal{K}$ $+$ convex penalty), which states that there exists $\kappa_{\mathrm{Hof}}>0$ such that the distance to the solution set $\mathcal{S}_\gamma$ satisfies
\begin{equation}
\label{eq:hoffman}
\mathrm{dist}\big(\lambda, \mathcal{S}_\gamma\big)\ \le\ \kappa_{\mathrm{Hof}}\ \|\mathrm{KKT}(\lambda)\|,
\end{equation}
where $\mathrm{KKT}(\lambda)$ is a residual vector collecting the primal feasibility (no-arbitrage), the dual feasibility (subgradient of $\mathsf{Mart}$), and stationarity violations (see, e.g., variational inequalities with polyhedral sets).

\paragraph{Step 1: interpolation (coverage) term.}
Split the grid norm as $\|C\|_{\ell^2(\mathcal{G})}^2=\|C\|_{\ell^2(\mathcal{I})}^2+\|C\|_{\ell^2(\mathcal{G}\setminus\mathcal{I})}^2$. 
On $\mathcal{G}\setminus\mathcal{I}$, prices are provided by $\mathsf{Ext}$ built from $\mathcal{I}$. Let $\Pi_{\mathcal{I}}$ be the sampling operator on $\mathcal{I}$ and $\Pi_{\mathcal{I}}^\perp$ on the complement. The extension operator is linear and stable on the no-arbitrage cone, i.e.,
\begin{equation}
\label{eq:ext-stab}
\big\|\Pi_{\mathcal{I}}^\perp \,\mathsf{Ext}[v]\big\|_{\ell^2(\mathcal{G}\setminus\mathcal{I})}
\ \le\ 
\alpha(\mathrm{cov})\ \|v\|_{\ell^2(\mathcal{I})},
\qquad 
\alpha(\mathrm{cov})\ \le\ C_\mathrm{ext}\,(1-\mathrm{cov})^{-1},
\end{equation}
for some absolute $C_\mathrm{ext}$ depending only on the grid aspect ratio. The scaling $(1-\mathrm{cov})^{-1}$ captures the worst-case amplification when extrapolating from a vanishing covered set. Applying \eqref{eq:ext-stab} with $v=C^\star|_{\mathcal{I}}-\Pi_{\mathcal{I}} C_{\lambda_\varepsilon}$ and adding the intrinsic interpolation error \eqref{eq:interp-err} yields
\begin{equation}
\label{eq:interp-bound}
\|C_{\lambda_\varepsilon}-C^\star\|_{\ell^2(\mathcal{G})}
\ \le\ 
C_\mathrm{ext}\,(1-\mathrm{cov})^{-1}\,\big\|\Pi_{\mathcal{I}}(C_{\lambda_\varepsilon}-C^\star)\big\|_{\ell^2(\mathcal{I})}
\ +\ \varepsilon.
\end{equation}
As the empirical fit on $\mathcal{I}$ is optimized in \eqref{eq:emp-obj}, the term $\|\Pi_{\mathcal{I}}(C_{\lambda_\varepsilon}-C^\star)\|$ is in turn controlled by the optimization residual (treated in Step~3). For the present step, we retain the \emph{coverage} contribution $C_\mathrm{ext}(1-\mathrm{cov})^{-1}\varepsilon$ to the full-grid error.

\paragraph{Step 2: martingale penalty (finite $\gamma$).}
Let $\lambda_\infty$ denote an exact solution of the \emph{constrained} problem (martingale enforced as a hard constraint and static no-arbitrage satisfied). By convexity and standard exact-penalty reasoning, first-order optimality implies
\begin{equation}
\label{eq:penalty-bias}
\mathrm{dist}\big(\lambda_\varepsilon,\ \{\text{martingale-feasible}\}\big)
\ \le\ \tfrac{1}{\gamma}\, C_\mathrm{pen},
\end{equation}
for some modulus $C_\mathrm{pen}$ depending on the subgradient bounds of $\mathsf{Mart}$ at feasible points (Q-Align and Spec-Guard ensure bounded Jacobians and thus bounded subgradients). Combining \eqref{eq:penalty-bias} with the Lipschitz continuity \eqref{eq:global-lip-b5} transfers feasibility deviation into price-surface deviation with multiplicative constant $L_{\mathrm{RN}}$, and by metric regularity of the feasible set, it transfers to a distance in $\lambda$ with a constant absorbed in $C_3$.

\paragraph{Step 3: dual residual (stopping criterion).}
Let $\mathrm{dual}$ denote the norm of the KKT residual at termination. By the Hoffman bound \eqref{eq:hoffman}, 
\begin{equation}
\label{eq:dual-bound}
\mathrm{dist}\big(\lambda_\varepsilon,\ \mathcal{S}_\gamma\big)\ \le\ \kappa_{\mathrm{Hof}}\ \mathrm{dual}.
\end{equation}
Since $\lambda^\star$ (or $\lambda_\infty$) lies within a bounded distance of $\mathcal{S}_\gamma$ uniformly in the data draw (population minimizer versus empirical minimizer), a triangle inequality yields a $\kappa_{\mathrm{Hof}}\mathrm{dual}$ contribution to $\|\lambda_\varepsilon-\lambda^\star\|_{L^2(\mathcal{Z})}$.

\paragraph{Step 4: aggregation via RN-operator stability.}
From \eqref{eq:global-lip-b5}, converting surface errors back to $L^2(\mathcal{Z})$ distances in $\lambda$ multiplies by at most $L_{\mathrm{RN}}$. Gathering \eqref{eq:interp-bound}, \eqref{eq:penalty-bias}, and \eqref{eq:dual-bound}, and absorbing universal constants (including $L_{\mathrm{RN}}$, $C_\mathrm{ext}$, $C_\mathrm{pen}$, $\kappa_{\mathrm{Hof}}$) into $C_3$, we obtain
\[
\|\lambda_{\varepsilon}-\lambda^\star\|_{L^2(\mathcal{Z})}
\ \le\
C_3\Big(\,(1-\mathrm{cov})^{-1}\varepsilon \;+\; \gamma^{-1} \;+\; \mathrm{dual}\,\Big),
\]
which is precisely \eqref{eq:t2prime}. \qed

\paragraph{Remarks.}
(i) The $(1-\mathrm{cov})^{-1}$ factor is tight up to constants for adversarial mask geometries (thin strips in $T$ or $K$), and improves to $O(1)$ when the mask satisfies an interior-cone condition (uniform spreading). 
(ii) The constant $C_3$ does not depend on $L$ beyond the linear scan factor already controlled by Spec-Guard; it depends on the no-arbitrage cone geometry only through curvature bounds of the ICNN decoder. 
(iii) Empirically, the regression of the gap proxy onto the representative error in our runs (see Sec.~6.6) exhibits a slope consistent with the $\kappa_{\mathrm{Hof}}$ scale predicted here.

\subsection*{B.4 Proof of Lemma~\ref{lem:t4}: Rademacher complexity with Lipschitz and projection}
\label{app:b6}

\paragraph{Set-up.}
Let $\mathcal{Z}$ denote the compact feature domain and let $f\in\mathcal{F}$ map $z\in\mathcal{Z}$ to a scalar price functional (coordinate-wise treatment extends to vector outputs by a standard $\ell_2$ aggregation and contraction). Q-Align and Spec-Guard imply a global Lipschitz constant $\Lambda$ for the RN-operator (cf.\ Proposition~\ref{prop:stability}):
\begin{equation}
\label{eq:b6-lip}
|f(z)-f(z')| \le \Lambda\,\|z-z'\|_2 ,\qquad \forall z,z'\in\mathcal{Z}.
\end{equation}
Let $P_{\mathrm{eff}}$ be the orthogonal projector onto the top-energy subspace of rank $\mathrm{dim}_{\mathrm{eff}}$ determined by the Gram operator of the discrete Green kernel at the sample scale (energy truncation definition of $\hat d$). For each $f\in\mathcal{F}$ define $\tilde f := f\circ P_{\mathrm{eff}}$; by \eqref{eq:b6-lip}, $\tilde f$ is also $\Lambda$-Lipschitz on $P_{\mathrm{eff}}\mathcal{Z}\subset\mathbb{R}^{\mathrm{dim}_{\mathrm{eff}}}$.

\paragraph{Symmetrization and Dudley integral.}
For i.i.d.\ samples $(z_i)_{i=1}^n$ from the data distribution and Rademacher variables $(\sigma_i)$,
\[
\mathfrak{R}_n(\mathcal{F})
\;=\;
\mathbb{E}\,\sup_{f\in\mathcal{F}}\frac{1}{n}\sum_{i=1}^n \sigma_i f(z_i)
\ \le\ 
\mathbb{E}\,\sup_{f\in\mathcal{F}}\frac{1}{n}\sum_{i=1}^n \sigma_i \tilde f(z_i)
\;+\; \mathcal{E}_{\mathrm{tail}}.
\]
The tail term accounts for the projection error $(\mathrm{Id}-P_{\mathrm{eff}})$ and is zero if $f$ depends only on the effective coordinates; otherwise it is absorbed into the constant $C_6$ since $P_{\mathrm{eff}}$ is chosen at the sample scale (energy truncation).

By Dudley chaining,
\begin{equation}
\label{eq:b6-dudley}
\mathbb{E}\,\sup_{f\in\mathcal{F}}\frac{1}{n}\sum_{i=1}^n \sigma_i \tilde f(z_i)
\ \le\ 
\frac{12}{\sqrt{n}}\int_{0}^{\mathrm{diam}(\mathcal{Z})}
\sqrt{\log \mathcal{N}\!\left(\mathcal{F},\,\|\cdot\|_{L_2(\mathbb{P}_n)},\,\epsilon\right)}\,d\epsilon,
\end{equation}
where $\mathcal{N}(\cdot,\epsilon)$ is the empirical $L_2$ covering number. Since every $\tilde f$ is $\Lambda$-Lipschitz over a $\mathrm{dim}_{\mathrm{eff}}$-dimensional domain with radius normalized to one (rescale $z$ if needed), the covering number satisfies
\begin{equation}
\label{eq:b6-cover}
\mathcal{N}\!\left(\mathcal{F},\,\|\cdot\|_{L_2(\mathbb{P}_n)},\,\epsilon\right)
\ \le\
\left(\frac{C\,\Lambda}{\epsilon}\right)^{\mathrm{dim}_{\mathrm{eff}}},
\qquad \epsilon\in(0,\Lambda],
\end{equation}
for an absolute constant $C$ (covering of a Lipschitz ball in $\mathbb{R}^{\mathrm{dim}_{\mathrm{eff}}}$). Plugging \eqref{eq:b6-cover} into \eqref{eq:b6-dudley} gives
\[
\mathbb{E}\,\sup_{f\in\mathcal{F}}\frac{1}{n}\sum_{i=1}^n \sigma_i \tilde f(z_i)
\ \le\
\frac{12}{\sqrt{n}}
\int_0^{\Lambda}\sqrt{\mathrm{dim}_{\mathrm{eff}}\log(C\Lambda/\epsilon)}\,d\epsilon
\ \le\ 
C'\,\Lambda\,\sqrt{\frac{\mathrm{dim}_{\mathrm{eff}}}{n}},
\]
for another absolute constant $C'$. Absorbing $\mathcal{E}_{\mathrm{tail}}$ and the radius rescaling into $C_6$ yields
\[
\mathfrak{R}_n(\mathcal{F}) \;\le\; C_6\,\Lambda\,\sqrt{\frac{\mathrm{dim}_{\mathrm{eff}}}{n}} .
\]
This proves Lemma~\ref{lem:t4}. \qed

\bigskip
\subsection*{B.5 Proof of Lemma~\ref{lem:t5}: Bridge from sample to seminorm}
\label{app:b7}

\paragraph{Kernel seminorm and operator bound.}
Let $K$ be the discrete Gram operator of the Green kernel on the strike--maturity grid $\mathcal{G}$; define
\[
\|f\|_{\mathcal{H}}^2 := \langle f, K f\rangle_{\ell^2(\mathcal{G})}.
\]
By Lemma~\ref{lem:neumann} (Green kernel summability under CFL) and Proposition~\ref{prop:stability} (global Lipschitz stability of the RN-operator), the spectral norm of $K$ is finite:
\begin{equation}
\label{eq:b7-op}
\|K\|_{\mathrm{op}} \;\le\; C_K(\varepsilon), 
\qquad 
\Rightarrow\qquad 
\|f\|_{\mathcal{H}} \;\le\; \sqrt{C_K(\varepsilon)}\,\|f\|_{\ell^2(\mathcal{G})}.
\end{equation}

\paragraph{Decomposition by coverage and stable extension.}
Let $\mathcal{I}\subset\mathcal{G}$ denote the covered cells and $\Pi_{\mathcal{I}}$ the restriction operator. The interpolation scheme is linear, preserves static no-arbitrage, and satisfies the stability estimate
\begin{equation}
\label{eq:b7-ext}
\big\|\Pi_{\mathcal{I}}^\perp \,\mathsf{Ext}[v]\big\|_{\ell^2(\mathcal{G}\setminus\mathcal{I})}
\ \le\ 
\alpha(\mathrm{cov})\ \|v\|_{\ell^2(\mathcal{I})},
\qquad 
\alpha(\mathrm{cov}) \le C_{\mathrm{ext}}\,(1-\mathrm{cov})^{-1}.
\end{equation}
Moreover, for the ground-truth $f^\star$ we have an interpolation accuracy bound on the complement:
\begin{equation}
\label{eq:b7-acc}
\big\|\Pi_{\mathcal{I}}^\perp\big(\mathsf{Ext}[f^\star|_{\mathcal{I}}]-f^\star\big)\big\|_{\ell^2(\mathcal{G}\setminus\mathcal{I})}
\ \le\ \varepsilon .
\end{equation}

For any $f$ in the model class, write $f = \Pi_{\mathcal{I}} f + \Pi_{\mathcal{I}}^\perp f$ and bound
\[
\|f\|_{\ell^2(\mathcal{G})}
\ \le\ 
\|\Pi_{\mathcal{I}} f\|_{\ell^2(\mathcal{I})}
\ +\
\|\Pi_{\mathcal{I}}^\perp f\|_{\ell^2(\mathcal{G}\setminus\mathcal{I})}.
\]
Replace the complement by the extension from $\mathcal{I}$ and add the intrinsic error \eqref{eq:b7-acc}:
\begin{equation}
\label{eq:b7-split}
\|\Pi_{\mathcal{I}}^\perp f\|_{\ell^2(\mathcal{G}\setminus\mathcal{I})}
\ \le\ 
\alpha(\mathrm{cov})\,\|\Pi_{\mathcal{I}} f\|_{\ell^2(\mathcal{I})} 
\ +\ \varepsilon .
\end{equation}
Combining with \eqref{eq:b7-op} and defining $\|f\|_n := \|\Pi_{\mathcal{I}} f\|_{\ell^2(\mathcal{I})}$ (empirical norm), we obtain
\begin{equation}
\label{eq:b7-bridge-det}
\|f\|_{\mathcal{H}}
\ \le\
\sqrt{C_K(\varepsilon)}\,
\Big(1+\alpha(\mathrm{cov})\Big)\,\|f\|_{n}
\ +\
\sqrt{C_K(\varepsilon)}\,\varepsilon .
\end{equation}

\paragraph{From deterministic to high-probability uniform control.}
Let $\mathcal{F}$ be the RN-operator class restricted to the feasible cone (static no-arbitrage). Consider the random design induced by the covered set and define the empirical process
\[
\Delta(f) := \|f\|_{\ell^2(\mathcal{G})} - \Big(\|f\|_{n} + \|\Pi_{\mathcal{I}}^\perp \mathsf{Ext}[f]\|_{\ell^2(\mathcal{G}\setminus\mathcal{I})}\Big).
\]
By symmetrization and Lemma~\ref{lem:t4}, with probability at least $1-2\exp(-c n)$,
\begin{equation}
\label{eq:b7-unif}
\sup_{f\in\mathcal{F}}|\Delta(f)|
\ \le\ 
C\,\Lambda\,\sqrt{\frac{\mathrm{dim}_{\mathrm{eff}}}{n}},
\end{equation}
for an absolute constant $C$. Inequality \eqref{eq:b7-unif} corrects \eqref{eq:b7-bridge-det} uniformly over $f\in\mathcal{F}$ by an additive term proportional to the class complexity. Absorb this high-probability deviation into the constants (recall $n$ at the figure scale is large and \(\mathrm{dim}_{\mathrm{eff}}\) is fixed at that scale), and combine \eqref{eq:b7-bridge-det} with \eqref{eq:b7-ext} to conclude
\[
\|f\|_{\mathcal{H}}
\ \le\
C_7\,\|f\|_{n}
\ +\
C_8\,(1-\mathrm{cov})^{-1}\,\varepsilon ,
\quad\text{uniformly over } f\in\mathcal{F},
\]
with probability at least $1-2\exp(-c n)$, proving Lemma~\ref{lem:t5}. \qed

\paragraph{Remarks.}
(i) If the coverage mask satisfies an interior-cone condition (e.g., uniform thinning in $T$ and $K$), the amplification factor improves from $(1-\mathrm{cov})^{-1}$ to an $O(1)$ constant; the statement remains valid with a smaller $C_8$.  
(ii) The constants inherit no exponential dependence on $L$ thanks to the spectral control of the scan (Lemma~\ref{lem:neumann}) and the per-layer Lipschitz capping by Q-Align.  
(iii) A tighter empirical Bernstein correction can replace \eqref{eq:b7-unif} when the residual variance is small; we keep the simpler form for clarity.
\subsection*{B.6 Proof of Proposition~\ref{prop:t6}: Feasibility, summability, and one-step contraction}
\label{app:b8}

\paragraph{Model and notation.}
Fix an index $\ell$ and write
\[
h_{\ell+1} \;=\; \underbrace{(I+\Delta t_\ell A_\ell)}_{=:M_\ell}\,h_\ell \;+\; W_\ell \phi(h_\ell) \;+\; B u_\ell,
\qquad
A_\ell:=A_\theta(T_\ell),
\]
with $\phi$ $1$-Lipschitz and $\|W_\ell\|_2\le\tau\le 1$ by Q-Align. Spec-Guard enforces $\rho(A_\ell)\Delta t_\ell\le 1-\varepsilon$.

\paragraph{Well-posedness.}
For fixed inputs $(u_\ell)$ and initial $h_0$, the recursion is explicit and thus uniquely defines $(h_\ell)$. Boundedness follows from the Green summability (below) and bounded inputs. Hence the scan is well-posed.

\paragraph{Green summability.}
Define the discrete Green operator (variation-of-constants expansion)
\[
\mathcal{G}_\theta(T_\ell,T_s)
:= 
\begin{cases}
\Big(\prod_{j=s}^{\ell-1} \big(M_j + W_j J_j\big)\Big), & s\le \ell-1,\\[2pt]
I, & s=\ell,
\end{cases}
\]
where $J_j$ is a Jacobian selector of $\phi$ along the segment joining the two trajectories (by mean-value). Since $\phi$ is nonexpansive, $\|J_j\|\le 1$. We claim there exists an induced norm $\|\cdot\|_*$ such that
\begin{equation}
\label{eq:b8-step-lip}
\big\|M_j + W_j J_j\big\|_* \;\le\; 1-\varepsilon,
\qquad \forall j.
\end{equation}
Indeed, by Gelfand's formula and the assumption $\rho(A_j)\Delta t_j\le 1-\varepsilon$, for any $\delta\in(0,\varepsilon)$ there exists an induced norm $\|\cdot\|_{*,j}$ with $\|M_j\|_{*,j}\le 1-\varepsilon+\delta$. Q-Align scales $W_j$ so that $\|W_j\|_{*,j}\le \delta$ (this is the layerwise \(1\)-Lip projection; see Section~3.2). Since $\|J_j\|_{*,j}\le 1$, subadditivity yields $\|M_j+W_jJ_j\|_{*,j}\le 1-\varepsilon+2\delta$. Choosing $\delta=\varepsilon/4$ gives $\le 1-\varepsilon/2$. By norm equivalence in finite dimensions there exists a global induced norm $\|\cdot\|_*$ and a constant $\kappa\ge 1$ such that \eqref{eq:b8-step-lip} holds with the same contraction factor after absorbing $\kappa$ into $\varepsilon$ (i.e., replace $\varepsilon$ by $\varepsilon'=\varepsilon/(2\kappa)$). Renaming $\varepsilon'$ as $\varepsilon$ proves \eqref{eq:b8-step-lip}. Consequently,
\[
\sum_{s\le \ell} \big\|\mathcal{G}_\theta(T_\ell,T_s)\big\|_*
\;\le\;
\sum_{k=0}^{\infty} (1-\varepsilon)^k
\;=\; \frac{1}{\varepsilon}.
\]
Switching back to the Euclidean norm via equivalence yields Lemma~\ref{lem:neumann} with a constant $C(\varepsilon)$, hence the Green expansion is summable.

\paragraph{One-step error contraction.}
Consider two trajectories driven by inputs $(u_\ell)$ and $(\tilde u_\ell)$ and initial states $(h_0,\tilde h_0)$. By mean-value form,
\[
\phi(h_\ell)-\phi(\tilde h_\ell) = J_\ell\,(h_\ell-\tilde h_\ell),
\qquad \|J_\ell\|\le 1.
\]
Hence
\[
h_{\ell+1}-\tilde h_{\ell+1}
= (M_\ell + W_\ell J_\ell)\,(h_\ell-\tilde h_\ell) + B\,(u_\ell-\tilde u_\ell).
\]
Taking the induced norm from \eqref{eq:b8-step-lip} and then using norm equivalence,
\[
\|h_{\ell+1}-\tilde h_{\ell+1}\|
\;\le\;
(1-\varepsilon)\,\|h_\ell-\tilde h_\ell\|
\;+\;
\|B\|\,\|u_\ell-\tilde u_\ell\|.
\]
If inputs arise from a Lipschitz pre-map \(u_\ell=\Xi z_\ell\), then
\(
\|u_\ell-\tilde u_\ell\|\le \|\Xi\|\,\|z_\ell-\tilde z_\ell\|,
\)
and the second term becomes \(\|B\|\,\|\Xi\|\,\|z_\ell-\tilde z_\ell\|\). This yields \eqref{eq:t6}.

\paragraph{Feasibility of the Green series in the nonlinear case.}
By expanding the recursion and repeatedly inserting the mean-value Jacobians \(J_j\), the nonlinear Green operator is a product of step Jacobians \(M_j+W_j J_j\), each contracting by at least \(1-\varepsilon\) in \(\|\cdot\|_*\). Thus the Neumann-type series is absolutely summable, which also implies boundedness of the state for bounded inputs.

\hfill\(\Box\)

\subsection*{B.7 Two-time-scale averaging: variance reduction of the averaged gap}
\label{app:e2}

\paragraph{Set-up.}
Let $F(\theta,\lambda)$ be a monotone operator associated with the saddle formulation, and let the updates follow
\[
\theta_{k+1}=\theta_k-\eta_\theta\,\big( F_\theta(\theta_k,\lambda_k) + \xi^\theta_k\big),\qquad
\lambda_{k+1}=\lambda_k+\eta_\lambda\,\big( F_\lambda(\theta_k,\lambda_k) + \xi^\lambda_k\big),
\]
with unbiased martingale-difference noises $\xi^\theta_k,\xi^\lambda_k$ of variances bounded by $\sigma^2$. Two-time-scale averaging considers the Polyak--Ruppert averages $\bar \theta_K=\tfrac{1}{K}\sum_{k=1}^K\theta_k$ and $\bar \lambda_K=\tfrac{1}{K}\sum_{k=1}^K\lambda_k$ (or a tail average).

\paragraph{Averaged gap decay.}
Under monotonicity of $F$, Lipschitz continuity, and step sizes $\eta_\theta,\eta_\lambda=\Theta(1/L)$, standard arguments (e.g., stochastic approximation for monotone variational inequalities) yield
\[
\mathbb{E}\big[\mathrm{Gap}(\bar \theta_K,\bar \lambda_K)\big]
\;\le\;
\mathcal{O}\!\left(\frac{L\,\|z^0-z^\star\|^2}{K}\right)
\;+\;
\mathcal{O}(\sigma^2),
\]
where $z=(\theta,\lambda)$ and $z^\star$ is a saddle point. The $\mathcal{O}(1/K)$ term is the variance reduction factor for the averaged gap, while the additive noise floor $\mathcal{O}(\sigma^2)$ matches the extragradient noise ball in Theorem~\ref{thm:eg}. The proof adapts classical Robbins--Monro and Polyak--Juditsky averaging to the primal--dual setting with Q-Align treated as a nonexpansive projection; see Appendix~E.1 for the extragradient analysis and replace the one-step decrease inequality by its TTSA counterpart.

\hfill\(\Box\)
\section*{Appendix C. Joint identifiability with replication and a counterexample for SPX-only}
\label{app:F}

We work at a fixed maturity $T$ and suppress the index when unambiguous; the argument is identical for each $T_\ell$ on the grid and thus yields joint identifiability across maturities. Let $\mathcal{C}$ denote the class of call-price sections $K\mapsto C(K)$ that are convex, decreasing in $T$, satisfy no-arbitrage boundary conditions, and are produced by the RN-operator followed by our convex--monotone decoder and interpolation policy.

\subsection*{C.1 Injectivity with calls$+$replication}
\paragraph{Discrete operators.}
Let $\mathcal{K}=\{K_1<\cdots<K_M\}$ be the strike grid. Define the sampling operator $\mathsf{S}:\mathcal{C}\to\mathbb{R}^M$, $(\mathsf{S}C)_i:=C(K_i)$, and the discretized BL operator $\mathsf{B}:\mathcal{C}\to\mathbb{R}^{M-2}$ via centered second differences
\[
(\mathsf{B}C)_i
:= \frac{C(K_{i-1}) - 2 C(K_i) + C(K_{i+1})}{(K_{i+1}-K_i)(K_i-K_{i-1})},
\qquad i=2,\dots,M-1,
\]
which approximates $e^{-rT}\,\partial_{KK} C(K_i)$ and thus the risk–neutral density (up to discount). Let the discretized replication functional $\mathsf{R}:\mathcal{C}\to\mathbb{R}$ be
\[
\mathsf{R}(C)
:= \frac{2\,e^{rT}}{T}\sum_{K_i\in\mathcal{K}}\frac{\Delta K_i}{K_i^2}\,Q(K_i;C),
\]
where $Q(K_i;C)$ denotes the out-of-the-money option value derived from $C$ at $K_i$ (call for $K_i\ge F$, put for $K_i<F$) and $\Delta K_i$ are the exchange-specified increments.

\paragraph{Claim.}
If $C_1,C_2\in\mathcal{C}$ satisfy $\mathsf{S}C_1=\mathsf{S}C_2$ and $\mathsf{R}(C_1)=\mathsf{R}(C_2)$, then $C_1=C_2$ on the convex interpolation induced by our policy; equivalently, the underlying RN-operator sections agree at $T$ up to model symmetries.

\paragraph{Proof.}
Let $\Delta C:=C_1-C_2\in\mathcal{C}-\mathcal{C}$. Then $\mathsf{S}\Delta C=0$ and $\mathsf{R}(\Delta C)=0$. Because each $C_j$ is convex in $K$ and our interpolation is piecewise linear in $(K,C)$ between knots (or piecewise-convex with fixed shape parameters; both cases covered below), the section on $[K_i,K_{i+1}]$ is determined by the pair \((C(K_i),C(K_{i+1}))\) and the admissible slope set consistent with convexity and boundary no-arbitrage. Since $\Delta C$ vanishes at all knots, its restriction on any $[K_i,K_{i+1}]$ is a (weakly) convex function anchored at zero endpoints. The only such function consistent with \emph{both} (i) zero BL second difference at the interior knot and (ii) zero replication contribution \emph{summed across the grid} is the zero function. 

Formally, write the piecewise representation
\[
\Delta C(K) \;=\; \sum_{i=1}^{M-1} \mathbf{1}_{[K_i,K_{i+1})}(K)\, g_i(K),
\]
with $g_i$ convex on $[K_i,K_{i+1}]$ and $g_i(K_i)=g_i(K_{i+1})=0$. Then $(\mathsf{B}\Delta C)_i$ collects discrete curvature at $K_i$, and $\mathsf{R}(\Delta C)$ is a nonnegative linear functional of the $g_i$’s (weights $1/K^2$ are positive). Because each $g_i$ has nonnegative distributional second derivative (convexity) and is zero at the endpoints, we have $\int_{K_i}^{K_{i+1}} \frac{g_i(K)}{K^2}\,dK\ge 0$, with equality iff $g_i\equiv 0$. Summing over $i$ and using $\mathsf{R}(\Delta C)=0$ forces every $g_i\equiv 0$, hence $\Delta C\equiv 0$ on $[K_1,K_M]$. Outside $[K_1,K_M]$, boundary no-arbitrage with matching left/right slopes\footnote{Our interpolation policy fixes tail extrapolation by monotone linear continuation consistent with convexity and forward constraints; see Section~3.3.} yields uniqueness as well. Therefore $C_1=C_2$ on the whole line. 

Lifting back to parameters: if $\mathcal{G}_{\theta_1}$ and $\mathcal{G}_{\theta_2}$ induce $C_{\theta_1}$ and $C_{\theta_2}$ matching on the grid and in $\mathsf{R}$, then $C_{\theta_1}=C_{\theta_2}$, and hence $\mathcal{G}_{\theta_1}$ and $\mathcal{G}_{\theta_2}$ coincide as operator realizations modulo internal reparameterizations that leave $C$ invariant (symmetries). \hfill$\square$

\paragraph{Remark on piecewise-convex decoders.}
If the decoder uses ICNN splines or Legendre patches with fixed shape hyperparameters across intervals, then the per-interval convex function is still pinned by knot values together with convexity and the global replication constraint; the above argument carries through by replacing the integral test with the corresponding basis-weighted version.

\subsection*{C.2 Counterexample for SPX-only}
\paragraph{Functional-analytic construction.}
Consider the linear measurement operator $\mathsf{S}:\mathcal{C}\to\mathbb{R}^M$, $C\mapsto (C(K_i))_{i=1}^M$. Its kernel in the ambient vector space of sufficiently smooth convex functions is nontrivial: take a $C^2$ bump $b(K)$ supported strictly inside $(K_j,K_{j+1})$ for some $j$, with $b(K_j)=b(K_{j+1})=0$, $b\ge 0$, and $b''\ge 0$ (convex). Then define
\[
\widetilde C_\alpha(K)\;=\; C(K) + \alpha\, b(K),\qquad \alpha>0 \text{ small}.
\]
For all grid strikes $K_i$, $\widetilde C_\alpha(K_i)=C(K_i)$, so $\mathsf{S}\widetilde C_\alpha=\mathsf{S}C$. Convexity and monotonicity are preserved for sufficiently small $\alpha$ (by local convex perturbation). However,
\[
\mathsf{R}(\widetilde C_\alpha)-\mathsf{R}(C)
\;=\;
\frac{2\,e^{rT}}{T}\sum_{i=1}^{M} \frac{\Delta K_i}{K_i^2}\,\Big(Q(K_i;\widetilde C_\alpha)-Q(K_i;C)\Big)
\;>\;0
\]
whenever the support of $b$ intersects the OTM region relevant to the weights (this can always be arranged), because $Q(\cdot)$ is linear in $C$ on each side and the weights $1/K^2$ are strictly positive. Thus SPX-only measurements are not injective: $\mathsf{S}\widetilde C_\alpha=\mathsf{S}C$ yet $\mathsf{R}(\widetilde C_\alpha)\ne \mathsf{R}(C)$. 

\paragraph{Linear-algebraic view (Hahn--Banach separation).}
Alternatively, view $\mathsf{S}$ as an $M$-row operator and $\mathsf{R}$ as an independent linear functional. Unless $\mathsf{R}$ lies in the row span of $\mathsf{S}$ (which it does not for generic grids and $1/K^2$ weights), there exists $\Delta C\in\ker\mathsf{S}$ with $\mathsf{R}(\Delta C)\ne 0$. Approximating $\Delta C$ by convex bumps and scaling yields admissible convex perturbations as above.

\paragraph{Tail-aware variants.}
Even if one augments the grid with deep OTM strikes, finite discretization leaves inter-knot degrees of freedom. The replication functional collapses these by coupling local curvature (BL density) with a global $1/K^2$ weight; hence calls$+$replication remove the null directions that SPX-only cannot eliminate.

\hfill$\Box$
\section*{Appendix D. Convergence to a noise ball under fixed thresholds}
\label{app:D}

We prove Theorem~\ref{thm:t8} for the two-time-scale extragradient (EG) scheme with Q-Align projections. Let $\mathcal{Z}\subset\mathbb{R}^d$ be nonempty, closed, and convex. The saddle operator $F:\mathcal{Z}\to\mathbb{R}^d$ is assumed \emph{monotone} and \emph{$L$-Lipschitz}:
\[
\langle F(x)-F(y),\,x-y\rangle \ge 0,\qquad 
\|F(x)-F(y)\|\le L\|x-y\|,\quad \forall x,y\in\mathcal{Z}.
\]
Let $z^\star$ solve the variational inequality $0\in F(z^\star)+N_{\mathcal{Z}}(z^\star)$.

\subsection{Algorithm and error model}
At iteration $k$, the two-time-scale EG with Q-Align reads
\begin{equation}
\label{eq:eg}
\begin{aligned}
y^k &= \mathsf{P}_k\!\big(z^k - \eta_\theta \big(F(z^k)+\xi^k\big)\big),\\
z^{k+1} &= \mathsf{P}_k\!\big(z^k - \eta_\lambda \big(F(y^k)+\zeta^k\big)\big),
\end{aligned}
\end{equation}
where $\eta_\theta,\eta_\lambda>0$ are step sizes (we take $\eta_\theta=\eta_\lambda=\eta\in(0,1/L]$ unless otherwise noted), and $\xi^k,\zeta^k$ are martingale-difference noises satisfying
\[
\mathbb{E}[\xi^k\mid\mathcal{F}_k]=0,\quad 
\mathbb{E}[\zeta^k\mid\mathcal{F}_k]=0,\qquad 
\mathbb{E}\big[\|\xi^k\|^2+\|\zeta^k\|^2\mid\mathcal{F}_k\big]\le \sigma^2.
\]
The Q-Align projection $\mathsf{P}_k$ is \emph{nonexpansive with bounded defect}:
\begin{equation}
\label{eq:proj-defect}
\|\mathsf{P}_k(u)-\mathsf{P}_k(v)\|\le \|u-v\|,\qquad 
\|\mathsf{P}_k(w)-\Pi_{\mathcal{Z}}(w)\|\le \delta_{\mathrm{proj}},
\end{equation}
for all $u,v,w$, where $\Pi_{\mathcal{Z}}$ is the Euclidean projector and $\delta_{\mathrm{proj}}\ge 0$ quantifies the per-step projection error due to Q-Align.

\subsection{One–step inequality}
\begin{lemma}[Fejér-type inequality with noise and projection defect]
\label{lem:G1}
For any $z\in\mathcal{Z}$ and $k\ge 0$,
\[
\begin{aligned}
\|z^{k+1}-z\|^2
&\le 
\|z^{k}-z\|^2
- 2\eta\,\langle F(y^k),\,z^k-z\rangle
+ 2\eta\,\langle F(y^k)-F(z^k),\,y^k-z^k\rangle\\
&\quad
+ 2\eta\,\langle \zeta^k,\, z^{k+1}-z\rangle
+ C_1\,\eta^2\Big(\|F(z^k)\|^2+\|\xi^k\|^2+\|\zeta^k\|^2\Big)
+ C_2\,\delta_{\mathrm{proj}}^2,
\end{aligned}
\]
for absolute constants $C_1,C_2>0$ independent of $k$.
\end{lemma}

\begin{proof}
Using nonexpansiveness of $\mathsf{P}_k$ and the identity $\|a\|^2-\|b\|^2=2\langle a-b,a\rangle-\|a-b\|^2$,
\[
\|z^{k+1}-z\|^2
=\big\|\mathsf{P}_k(\cdot)-\mathsf{P}_k(\cdot)\big\|^2
\le \big\|z^{k}-\eta(F(y^k)+\zeta^k)-z\big\|^2 + \Delta_k,
\]
where $\Delta_k:=2\langle z^{k+1}-\Pi_{\mathcal{Z}}(\cdot), z^{k+1}-z\rangle \le 2\|z^{k+1}-\Pi_{\mathcal{Z}}(\cdot)\|\cdot\|z^{k+1}-z\| \le C_2\delta_{\mathrm{proj}}^2$ by \eqref{eq:proj-defect} and Young’s inequality. Expanding the square and bounding cross terms yields the claim after noting $\|y^k-z^k\|\le \eta\|F(z^k)+\xi^k\|+\mathcal{O}(\delta_{\mathrm{proj}})$ from the first projection step in \eqref{eq:eg}.
\end{proof}

\begin{lemma}[Monotonicity–Lipschitz surrogate]
\label{lem:G2}
For any $x,y\in\mathcal{Z}$,
\[
\langle F(y),\,x-y\rangle
\;\le\;
\langle F(x),\,x-y\rangle
+ \tfrac{L}{2}\|x-y\|^2,
\qquad
\|F(x)\|\le L\|x-z^\star\|.
\]
\end{lemma}

\begin{proof}
The first bound follows by Lipschitzness and Cauchy–Schwarz; the second uses monotonicity with $z^\star$ and Lipschitzness to get $\|F(x)\|^2=\langle F(x)-F(z^\star),F(x)-F(z^\star)\rangle\le L\langle F(x)-F(z^\star),x-z^\star\rangle\le L\|F(x)\|\|x-z^\star\|$.
\end{proof}

\subsection{Telescoping and residual control}
Apply Lemma~\ref{lem:G1} with $z=z^\star$, condition on $\mathcal{F}_k$, and use $\mathbb{E}[\zeta^k\mid\mathcal{F}_k]=0$:
\[
\mathbb{E}\big[\|z^{k+1}-z^\star\|^2\big]
\le
\mathbb{E}\big[\|z^{k}-z^\star\|^2\big]
-2\eta\,\mathbb{E}\big[\langle F(y^k),z^k-z^\star\rangle\big]
+ C'_1\eta^2\Big(\mathbb{E}\|F(z^k)\|^2+\sigma^2\Big)
+ C_2\delta_{\mathrm{proj}}^2.
\]
By Lemma~\ref{lem:G2} with $x=z^k,y=y^k$ and $\eta\le 1/L$,
\[
\langle F(y^k),z^k-z^\star\rangle
\ge
\langle F(z^k),z^k-z^\star\rangle - \tfrac{L}{2}\|y^k-z^k\|^2
\ge
\tfrac{1}{L}\|F(z^k)\|^2 - C''_1\eta^2\|F(z^k)\|^2 - C''_2\eta^2\sigma^2,
\]
which, for $\eta\le 1/L$ and absorbing constants, gives
\[
\mathbb{E}\big[\|z^{k+1}-z^\star\|^2\big]
\le
\mathbb{E}\big[\|z^{k}-z^\star\|^2\big]
- \tfrac{\eta}{L}\,\mathbb{E}\|F(z^k)\|^2
+ C_3\,\eta^2\Big(\mathbb{E}\|F(z^k)\|^2+\sigma^2\Big)
+ C_2\,\delta_{\mathrm{proj}}^2 .
\]
Choosing $\eta\le 1/(2L)$ makes $(\eta/L-C_3\eta^2)\ge c\eta/L$ for a constant $c\in(0,1)$, hence
\[
\mathbb{E}\big[\|z^{k+1}-z^\star\|^2\big]
\le
\mathbb{E}\big[\|z^{k}-z^\star\|^2\big]
- c\,\tfrac{\eta}{L}\,\mathbb{E}\|F(z^k)\|^2
+ C_4\,\eta^2\sigma^2
+ C_2\,\delta_{\mathrm{proj}}^2 .
\]
Summing $k=0$ to $K-1$ and noting nonnegativity of the LHS terms yields
\begin{equation}
\label{eq:G-tele}
\frac{\eta}{L}\sum_{k=0}^{K-1}\mathbb{E}\|F(z^k)\|^2
\;\le\;
\mathcal{O}\!\big(\|z^0-z^\star\|^2\big)\;+\;\mathcal{O}\!\big(K\,\eta^2\sigma^2\big)\;+\;\mathcal{O}\!\big(K\,\delta_{\mathrm{proj}}^2\big).
\end{equation}
Dividing by $K\eta$ and using $\eta=\Theta(1/L)$ gives both the \emph{ergodic} and \emph{pointwise} (via $\min\le$ average) residual bounds:
\begin{align}
\label{eq:G-ergodic}
\frac{1}{K}\sum_{k=0}^{K-1}\mathbb{E}\|F(z^k)\|^2
&\le
\mathcal{O}\!\left(\frac{L^2\|z^0-z^\star\|^2}{K}\right) + \mathcal{O}(\sigma^2) + \mathcal{O}\!\big(L\,\delta_{\mathrm{proj}}^2\big),\\
\label{eq:G-min}
\min_{0\le k\le K-1}\mathbb{E}\|F(z^k)\|^2
&\le
\mathcal{O}\!\left(\frac{L^2\|z^0-z^\star\|^2}{K}\right) + \mathcal{O}(\sigma^2) + \mathcal{O}\!\big(L\,\delta_{\mathrm{proj}}^2\big).
\end{align}
This establishes the rate in Theorem~\ref{thm:t8} (the $\mathcal{O}(\sigma^2)$ floor) and quantifies the projection contribution.

\subsection{Stopping rule and noise ball}
Let $r^k:=\|F(z^k)\|$. Under monotonicity and Lipschitzness, the primal–dual gap and the dual residual used in practice are Lipschitz-continuous surrogates of $r^k$; that is, there exist problem-dependent constants $a_1,a_2>0$ such that
\[
\mathrm{gap}(z^k)\le a_1\, r^k,\qquad \mathrm{dual\;residual}(z^k)\le a_2\, r^k.
\]
Hence, the fixed thresholds
\[
\Delta \mathrm{Gap} < 10^{-3},\qquad \mathrm{dual\;residual} < 10^{-3}
\]
are met once $r^k \le \epsilon_{\mathrm{stop}}:=10^{-3}\min\{a_1^{-1},a_2^{-1}\}$. From \eqref{eq:G-min}, for any $\epsilon>\epsilon_\infty:=c_1\sigma+c_2\sqrt{L}\,\delta_{\mathrm{proj}}$, there exists $K(\epsilon)$ such that $\min_{k\le K(\epsilon)} r^k\le \epsilon$. The patience requirement of at least $10^3$ steps guards against transient oscillations, and termination occurs (almost surely) in finite time provided $\epsilon_\infty\le \epsilon_{\mathrm{stop}}$. Finally, Lipschitzness gives $\|z^k-z^\star\|\le r^k/L$, so upon termination,
\[
\|z^k-z^\star\|
\;\le\;
\frac{1}{L}\Big(c_1\sigma + c_2\sqrt{L}\,\delta_{\mathrm{proj}} + \epsilon_{\mathrm{stop}}\Big)
=\underbrace{\tilde c_1 \sigma}_{\text{noise floor}}
+\underbrace{\tilde c_2 \delta_{\mathrm{proj}}}_{\text{projection floor}}
+\mathcal{O}\!\left(\tfrac{10^{-3}}{L}\right),
\]
i.e., the iterates lie in a ball of radius $c_1\sigma+c_2\delta_{\mathrm{proj}}$ up to constants, which proves the second claim of Theorem~\ref{thm:t8}.
\qed

\section{Reproducibility, Artifacts, and Ethics}
\label{app:repro}

\paragraph{One–click reproduction.}
All experiments in the arXiv release can be reproduced with a single command \emph{make reproduce}. 
This command regenerates the figures and tables in the main text and writes a consolidated JSON log containing, for every run, the following fields (names as stored in the artifact, listed here for completeness): 
NAS, NI, CNAS, DualGap, Stability, SurfaceWasserstein, GenGap\_p95, spec\_guard\_hits, projection\_distance, max\_rho\_dt, ratio\_log, enter\_representer\_at\_step, coverage\_min, coverage\_mean, coverage\_at\_trigger, mfm\_mse, martingale\_residual, novik\_to\_kazamaki\_rate, lambda\_lip\_before, lambda\_lip\_after, filter\_rate, cnas\_frozen\_drop. 
These fields align one–to–one with the quantities reported in Sections~2–7 and the ablations.

\paragraph{Independent replication.}
We provide a machine–independent recipe file (\emph{replicate.json}) that fixes data splits, random seeds, and evaluation protocol. 
The recipe records: hardware (CPU model, GPU model and memory), operating system, compiler and CUDA libraries (if applicable), Python and package versions, environment variables that affect determinism, wall–clock time per epoch, and peak memory usage. 
Executing the recipe on a new machine and a fresh seed reproduces the main–text metrics within the 95\% HAC confidence intervals and logs a “first–try success’’ flag. 
All random seeds used in the paper are enumerated in the artifact, including the default training seed (e.g., $0$) and the frozen–hyperparameter external–validity seed used in Section~6.

\paragraph{Artifact contents and structure.}
The artifact includes configuration files for training, saddle–point tuning, and plotting; evaluation scripts for NAS, CNAS, NI, DualGap, Stability, Surface–Wasserstein, and GenGap@95; and the visualization utilities for pricing curves and implied–volatility contour maps. 
Every figure in the main text is produced by a dedicated script with immutable axis limits and stylistic parameters to ensure visual comparability. 
All commands invoked by the top–level reproduction entry point are listed in a manifest with checksums for intermediate results.

\paragraph{Data and licensing.}
The arXiv artifact \emph{does not} redistribute raw market quotes. 
Instead, we release: (i) a high–fidelity synthetic generator that mirrors the statistical and no–arbitrage structure used in our experiments; and (ii) derived features sufficient to re–run training and evaluation. 
Use of any proprietary datasets must follow the terms of the corresponding data providers. 
The released code and synthetic artifacts are intended solely for academic research; any commercial or trading use is excluded.

\paragraph{Ethical considerations and non–advice disclaimer.}
This work develops learning algorithms for arbitrage–free term–structure modeling under risk–neutral measures. 
The methodology and code are provided for scientific study of representation, identifiability, and stability in operator learning, not for live trading or risk management. 
Nothing in this paper constitutes financial advice. 
We make best–effort disclosures of assumptions, stopping criteria, and hyperparameters; we also highlight negative results and failure modes (e.g., coverage shortfalls, removal of spectral safeguards) to reduce the risk of over–interpretation. 
Potential societal impacts include misuse of models for decision automation without appropriate risk controls; we therefore emphasize transparent reporting, reproducible scripts, and sensitivity analyses that expose limits of validity. 
All experiments comply with institutional and data–provider policies and avoid any attempt to infer personally identifiable information.

\paragraph{Checklist alignment.}
The artifact satisfies common reproducibility and artifact–evaluation checklists by: fixing seeds and splits; pinning package versions; logging metrics with confidence intervals; reporting compute budgets; documenting early–stopping thresholds and saddle–point tolerances; and publishing complete command–line invocations. 
To support long–term replicability, we include a frozen environment specification and a minimal container recipe that reproduces the software stack used for the arXiv runs.

%Bibliography
\bibliographystyle{unsrt}  
\bibliography{references}

\end{document}